\documentclass[twoside,11pt]{article}

\usepackage[abbrvbib, preprint]{jmlr2e}

\usepackage{amsmath}
\usepackage{thmtools}
\usepackage{array}
\usepackage{multirow}
\usepackage[caption=false]{subfig}
\usepackage{booktabs}
\usepackage{xcolor}
\usepackage{framed}
\usepackage{setspace}
\usepackage{algorithmic}
\usepackage{algorithm} 
\usepackage{tikz}
\usepackage{adjustbox}

\newcommand{\bx}{\mathbf{x}}
\newcommand{\by}{\mathbf{y}}

\newcommand{\expt}{\mathrm{E}}

\newcommand{\bw}{\mathbf{w}}
\newcommand{\bs}{\mathbf{s}}
\newcommand{\bg}{\mathbf{g}}
\newcommand{\bv}{\mathbf{v}}
\newcommand{\bbeta}{\boldsymbol{\eta}}
\newcommand{\wb}{\overline}

\usepackage[mathlines]{lineno}
\newcommand*\patchAmsMathEnvironmentForLineno[1]{%
    \expandafter\let\csname old#1\expandafter\endcsname\csname #1\endcsname
    \expandafter\let\csname oldend#1\expandafter\endcsname\csname end#1\endcsname
    \renewenvironment{#1}%
    {\linenomath\csname old#1\endcsname}%
    {\csname oldend#1\endcsname\endlinenomath}}%
  \newcommand*\patchBothAmsMathEnvironmentsForLineno[1]{%
    \patchAmsMathEnvironmentForLineno{#1}%
    \patchAmsMathEnvironmentForLineno{#1*}}%
  \patchBothAmsMathEnvironmentsForLineno{equation}%
  \patchBothAmsMathEnvironmentsForLineno{align}%
  \patchBothAmsMathEnvironmentsForLineno{flalign}%
  \patchBothAmsMathEnvironmentsForLineno{alignat}%
  \patchBothAmsMathEnvironmentsForLineno{gather}%
  \patchBothAmsMathEnvironmentsForLineno{multline}%

\RequirePackage[capitalize,nameinlink]{cleveref}[0.19]

\newcommand{\crefassum}[1]{Universal Assumptions}

\newcommand{\crefobjdef}[1]{Problem Assumption}

\allowdisplaybreaks

\crefname{section}{section}{sections}
\crefname{subsection}{subsection}{subsections}
\Crefname{section}{Section}{Sections}
\Crefname{subsection}{Subsection}{Subsections}
\Crefname{figure}{Figure}{Figures}

\crefformat{equation}{\textup{#2(#1)#3}}
\crefrangeformat{equation}{\textup{#3(#1)#4--#5(#2)#6}}
\crefmultiformat{equation}{\textup{#2(#1)#3}}{ and \textup{#2(#1)#3}}
{, \textup{#2(#1)#3}}{, and \textup{#2(#1)#3}}
\crefrangemultiformat{equation}{\textup{#3(#1)#4--#5(#2)#6}}%
{ and \textup{#3(#1)#4--#5(#2)#6}}{, \textup{#3(#1)#4--#5(#2)#6}}{, and \textup{#3(#1)#4--#5(#2)#6}}

\Crefformat{equation}{#2Equation~\textup{(#1)}#3}
\Crefrangeformat{equation}{Equations~\textup{#3(#1)#4--#5(#2)#6}}
\Crefmultiformat{equation}{Equations~\textup{#2(#1)#3}}{ and \textup{#2(#1)#3}}
{, \textup{#2(#1)#3}}{, and \textup{#2(#1)#3}}
\Crefrangemultiformat{equation}{Equations~\textup{#3(#1)#4--#5(#2)#6}}%
{ and \textup{#3(#1)#4--#5(#2)#6}}{, \textup{#3(#1)#4--#5(#2)#6}}{, and \textup{#3(#1)#4--#5(#2)#6}}

\crefdefaultlabelformat{#2\textup{#1}#3}

\newtheorem{proposition}[theorem]{Proposition}

\newcommand{\refobjdef}[1]{Problem Assumption}

\newtheorem{assumption}{Assumption} 

\usepackage{xcolor}

\usepackage{lastpage} 

\graphicspath{{figures/}}

\jmlrheading{24}{2023}{1-\pageref{LastPage}}{3/22}{}{PAPER ID}{Lu Xia and Stefano Massei} 

\ShortHeadings{AdamL: A fast adaptive gradient method incorporating loss function}{Xia, and Massei}
\firstpageno{1}

\begin{document}

\title{AdamL: A fast adaptive gradient method \\ incorporating loss function}

\author{\name Lu Xia \email l.xia1@tue.nl\\
\addr Department of Mathematics and Computer Science\\ 
 Eindhoven University of Technology\\
 Eindhoven, 5600 MB, The Netherlands
 \AND
\name Stefano Massei \email stefano.massei@unipi.it\\
 \addr Department of Mathematics\\
 Università di Pisa\\
 Pisa, 56127, Italy }

\editor{}

\maketitle


\begin{abstract}
Adaptive first-order optimizers are fundamental tools in deep learning, although they may suffer from poor generalization due to the nonuniform gradient scaling. In this work, we propose AdamL, a novel variant of the Adam optimizer, that takes into account the loss function information to attain better generalization results.
We provide sufficient conditions that together with the Polyak-{\L}ojasiewicz inequality, ensure the linear convergence of AdamL. As a byproduct of our analysis, we prove similar convergence properties for the EAdam, and AdaBelief optimizers. Experimental results on benchmark functions show that AdamL typically achieves either the fastest convergence or the lowest objective function values when compared to Adam, EAdam, and AdaBelief. These superior performances are confirmed when considering deep learning tasks such as training convolutional neural networks, training generative adversarial networks using vanilla convolutional neural networks, and long short-term memory networks. Finally, in the case of vanilla convolutional neural networks, AdamL stands out from the other Adam's variants and does not require the manual adjustment of the learning rate during the later stage of the training. 
\end{abstract}


\section{Introduction}\label{sec:intro}
First-order optimization approaches, e.g., stochastic gradient descent (with momentum) methods (SGD) \citep{robbins1951stochastic}, prevail in the training of deep neural networks due to their simplicity. However, the constant learning stepsize along each gradient coordinate generally limits the convergence speed of SGD. This can be addressed by adaptive gradient methods, such as  Adagrad \citep{duchi2011adaptive},  RMSProp \citep{Tieleman2012rmsprop} and Adam \citep{kingma2014adam}. Adam, which is arguably the most used optimizer,  combines the main ingredients from the SGD with momentum \citep{qian1999momentum} and RMSprop. 

Although at early stages of training the adaptive methods usually show a faster decay of the objective function value than SGD, their convergence becomes slower at later stages \citep{keskar2017improving,wilson2017marginal}. Furthermore, the nonuniform scaling of the gradient may cause worse performances on the unseen data, a phenomenon that is often called \emph{poor generalization} \citep{hoffer2017train,keskar2017improving}. 

{\textbf{Prior Work.}} Many techniques have been merged to bridge the generation gap between adaptive and non-adaptive gradient methods. For instance, the procedure SWATS, proposed in \citep{keskar2017improving},  switches from Adam to SGD when a triggering condition is satisfied. Another example is the method Adabound~\citep{luoadaptivebound} that applies a smooth transition from adaptive methods to SGD as the time step increases, by means of dynamic bounds on the learning rates. 

A recent study introduces a variant of the Adam optimizer called AdaBelief \citep{zhuang2020adabelief}, which achieves the initial fast convergence of adaptive methods, good generalization as SGD, and training stability in complex settings such as generative adversarial networks (GANs). Adabelief is obtained from Adam by replacing the exponential moving average of the squares of the gradients, used to estimate the second moment, with the difference between the true gradient and its first moment; the authors call this quantity the ``belief'' term. Finally, AdaBelief adds a small positive constant $\varepsilon$ at each iteration to its second moment estimate. It is claimed that by using the ``belief'' term, AdaBelief considers the curvature of the objective function to improve the convergence \cite[Sec.~2.2]{zhuang2020adabelief}. However, it has been empirically demonstrated on test cases from image classification, language modeling, and object detection,  that modifying the Adam scheme by simply adding the constant $\varepsilon$  to its second moment estimate, yields comparable performances as AdaBelief \citep{yuan2020eadam}. The method obtained with such a modification is called EAdam. 
The similar performances of EAdam and AdaBelief in these case studies require further investigations and are the main motivation of this work.

 These variants of the Adam optimizer such as Adabound, AdaBelief, and EAdam are empirically shown to achieve faster convergence and better generalization performances than Adam. The convergence is normally proven by measuring the algorithm regret $R_{k}$, after a certain number of iteration steps $k$, and by showing that either $\lim_{k\to \infty}\tfrac{R_{k}}{k}=0$ or that $\lim_{k\to \infty} \nabla f(\mathbf{x}^{(k)}) = \mathbf 0$ \citep{kingma2014adam,luoadaptivebound,zhuang2020adabelief}. However, the convergence analyses developed so far do not provide tools for comparing the convergence rates of the different optimizers. From a theoretical perspective, it remains uncertain whether there is an optimizer that is superior to the others, even in specific scenarios. In general, there is a lack of comparisons of the convergence rates of adaptive gradient optimizers.

{\textbf{Contribution of the paper.}} 
Inspired by the body of work on Adam-type optimizers, we propose a novel adaptive optimization method that we call \emph{AdamL}. This method incorporates information from the loss function by linking the magnitude of the update to the distance between the value of the objective function at the current iterate and the optimum. When the loss function is not directly accessible, i.e. the optimal value is unknown, a dynamic approximation strategy can be employed. The idea is to improve the adaptivity of the scheme by taking small steps when the loss function is small and large steps otherwise.  
AdamL has also the advantage that it does not need a user that manually decreases the learning rate at the late stage of the training, e.g., when training convolutional neural networks.  

On the theoretical side, we study and compare the convergence properties of the EAdam, AdaBelief, and AdamL optimizers, under the Polyak-{\L}ojasiewicz (PL) inequality. Within this context, we provide sufficient conditions for the monotonicity and linear convergence of the above Adam's variants. Moreover, we discuss the relation between these conditions and two phases that are often encountered in the execution of Adam's variants. In the first phase, the methods behave similarly to SGD with decreasing learning rate with respect to the iteration steps; in the second phase, they perform like Adam. Finally, we perform extensive numerical tests to show that with a proper scaling strategy of the loss function, the AdamL optimizer yields faster and more stable convergence than the other adaptive gradient optimizers (Adam, AdaBelief, and EAdam).  The considered case studies involve image classification tasks on the \textsf{Cifar10} and \textsf{Cifar100} datasets, language modeling on \textsf{Penn Treebank}, and the training of WGANs on \textsf{Cifar10} and \textsf{Anime Faces}. 

{\bf{Outline.}} A review of the considered adaptive gradient methods is presented in \cref{sec:alg}. The AdamL method is introduced in  \cref{sec:adaml}. In \cref{sec:convergence}, we introduce a framework for analyzing the monotonicity and convergence rate of adaptive gradient methods under the assumption of an objective function satisfying the PL inequality. The main theoretical results are given in Propositions~\ref{prop:eadam_monotonicity}--\ref{prop:ladam_convergencerate}. The performance of AdamL is validated and compared with those of the other optimizers in \cref{sec:experiments}. Conclusions are drawn in \cref{sec:conclusions}.

{\bf{Notation.}} Throughout the paper we use bold lower case letters, e.g. $\bg$, for vectors; superscripted case letters, e.g., $\bg^{(k)}$ for the $k$th iteration step dependency; subscripted case letters, e.g., $g_i$ for the $i$th element of the vector; lower case Greek letters, for scalars. For the ease of readability, all the operations on/between vectors are elementary-wise, e.g., $\tfrac{{\bf m}^{(k)}}{\bg^{(k)}}$, $({\bg}^{(k)})^2$, $\sqrt{\bg^{(k)}}$ and $\vert\, \bg^{(k)} \,\vert $ represent the component-wise division, square, square root and absolute value on $\bg^{(k)}$, respectively. When we write an operation between a  vector and a scalar, we mean the component-wise operation between the former and a vector having all entries equal to the scalar. For instance,  $\frac{\delta}{\bg+\varepsilon}$ indicates the vector with entries $\frac{\delta}{g_i+\varepsilon}$. 

\section{Adaptive gradient optimization methods} \label{sec:alg}
This work is concerned with  the following stochastic optimization problem 
\begin{equation}\label{eq:objfun}
    \min_{\bx\in \mathcal{X}} \{f(\bx):= \expt_{\xi \sim\mathcal{D}}\,[\,F(\bf x,\xi)\,]\},
\end{equation}
where $\mathcal{X}\subset \mathbb{R}^{n}$ is a nonempty compact set, $\xi$ is a random vector whose
probability distribution $P$ is supported on set $\mathcal{D} \subset \mathbb{R}^r$ and $F: \mathcal{X} \times \mathcal{D} \to  \mathbb{R}$. Concerning the objective function, 
we make the following 
assumptions that are common to many other studies on the convergence analysis of stochastic optimization, e.g. see \citep{kingma2014adam,luoadaptivebound,nemirovski2009robust,ghadimi2016mini,ghadimi2013stochastic,nguyen2019tight}.
\begin{assumption}\label{assump:assumption}
{\rm The function $f: \mathbb{R}^{n} \to \mathbb{R}$ satisfies the PL inequality, $F(\cdot,\xi)$ is $L$-smooth for every realization of $\xi$, the stochastic gradient $\nabla_{\bx} F(\bx,\xi)$ is bounded and unbiased for every $\xi \in \mathcal{D}$, and the variance of $\nabla_{\bx} F(\bx,\xi)$ is uniformly bounded. In particular, there exist constants $\sigma,\mu, L,G>0$ such that the following inequalities hold:}
\begin{align}
2\mu \,(f(\mathbf{x})-f^*)&\le\Vert \nabla f(\mathbf{x})\Vert^2,\label{eq:PLineq}\\
\Vert \nabla_{\bx} F(\bx_1,\xi) -\nabla_{\bx} F(\bx_2,\xi)\Vert &\leq L\,\Vert \bx_1-\bx_2\Vert,\label{eq:Lineq}\\
\expt\,[\,\nabla_{\bx} F(\bx,\xi)\,]&=\nabla f(\bx),\label{eq:expeq}\\
\Vert \nabla_{\bx} F(\bx,\xi)\Vert_{\infty}&\leq G,\label{eq:infnorm}\\
    0\leq\expt\,[\,\Vert \nabla_{\bx} F(\bx,\xi)\Vert^2\,] &-\Vert \nabla f(\bx)\Vert^2 \leq \sigma^2, \label{eq:varineq}
\end{align}
{\rm for any $\bx,\bx_1,\bx_2 \in \mathbb R^{n}$}. \end{assumption}
Note that \cref{eq:expeq} always holds if $f(\cdot)$ is finite valued in a neighborhood of a point $x$ \cite[cf.~(13)]{nemirovski2009robust}. On the basis of \cref{eq:Lineq}, \cref{eq:expeq}, and Jensen's inequality, we have that
\begin{align}\label{eq:Lineqf}
  \Vert \nabla f(\bx_1) -\nabla f(\bx_2)\Vert &\leq L\,\Vert \bx_1-\bx_2\Vert. 
\end{align}

\subsection{Algorithms} Let us review the Adam method and two of its variants, namely EAdam and Adabelief. Adam combines the main schemes of two other popular optimizers, i.e., momentum and RMSProp. Similar to momentum, the Adam optimizer updates the parameters using the exponential moving average (EMA) of the gradient $\nabla_{\bx} F(\bx^{(k)},\xi)$, i.e., the first raw moment estimate. The learning rate is scaled with the same rule of RMSProp, i.e., the EMA of the squared gradient, namely the second raw moment estimate. Since the first and second raw moment estimates are biased by their initialization, a bias correction is applied. EAdam and AdaBelief follow the same updating scheme as Adam, with the exception that they employ alternative second moment estimates, as outlined in \cref{tab:eta_k}. 
In \cref{ag:adam}, we present a general template of these adaptive gradient methods; to obtain a specific optimizer it is sufficient to define $\bbeta^{(k)}$ in the while loop with the corresponding second raw moment estimate in \cref{tab:eta_k}: $\bv^{(k)}$ for Adam, $\by^{(k)}$ for EAdam, and $\bs^{(k)}$ for AdaBelief. The hyperparameters $\beta_1, \beta_2 \in (0,1)$ control the exponential decay rates of these moving averages and $\eta>0$ affects the magnitude of the learning rate. For a detailed description of these algorithms see \citep{kingma2014adam,yuan2020eadam,zhuang2020adabelief}.  Differently from the Adam and EAdam optimizers, AdaBelief computes its second moment based on the square of the difference between the first moment and the gradient, expressed as $(\nabla_{\bx} F(\bx^{(k)},\xi)-{\bf{m}}^{(k)})^2$. Note that the second moment of the EAdam ($\by^{(k)}$) and AdaBelief ($\bs^{(k)}$) optimizers involve an additional scalar parameter, denoted by $\varepsilon$. In particular, the presence of the $\varepsilon$ term in EAdam and AdaBelief ensures that all the components of ${\bf y}^{(k)}$ and ${\bf s}^{(k)}$ are positive; instead, the entries of ${\bf v}^{(k)}$ in Adam are only non-negative. 

\begin{algorithm}[H]
\begin{algorithmic}
\STATE \textbf{Input:} $\bx^{(0)}$, hyperparameters $\beta_1, \beta_2 \in (0,1)$ and $\varepsilon>0$\\
\textbf{Initialize} ${\bf m}^{(0)}=\mathbf 0$, $k = 0$\\
\textbf{While} $\bx^{(k)}$ not converged \\
\STATE\hspace{12mm}$\bg^{(k)} = \nabla_{\bx} F(\bx^{(k)},\xi^{(k)})$ \\
\STATE\hspace{12mm}${\bf m}^{(k+1)} = \beta_1 \,{\bf m}^{(k)} + (1 - \beta_1) \,\bg^{(k)}$ \\
\STATE\hspace{12mm}$\bx^{(k+1)}= \bx^{(k)} - \bbeta^{(k+1)} \,\tfrac{{\bf m}^{(k+1)}}{1-\beta_1^{k+1}}$, where $\bbeta^{(k+1)}$ is computed as in \cref{tab:eta_k}
\STATE\hspace{12mm}$k = k + 1 $ \\
\textbf{End While} 
\end{algorithmic}
\caption{Template of adaptive gradient algorithms}\label{ag:adam}
\end{algorithm}

\begin{table}[htp]
\caption{Summary of the adaptive learning rate for the Adam, EAdam, and AdaBelief optimizers. The starting point of all the sequences is the zero vector, i.e., ${\bf v}^{(0)}={\bf y}^{(0)}={\bf s}^{(0)}= {\bf 0}$.}\label{tab:eta_k}
\centering
\begin{tabular}{ll}
\hline
Optimizer & Adaptive learning rate\\\hline \rule{0pt}{2.3ex}
Adam      &    $\bbeta^{(k+1)}=\tfrac{\eta}{\sqrt{\tfrac{\bv^{(k+1)}}{1-\beta_2^{k+1}}}+\varepsilon}$, ${\bf v}^{(k+1)} = \beta_2\, {\bf v}^{(k)} + (1 - \beta_2) \,(\bg^{(k)})^{ 2} $               \\
\rule{0pt}{2.3ex} 
EAdam     &    $\bbeta^{(k+1)}=\tfrac{\eta}{\sqrt{\tfrac{\by^{(k+1)}}{1-\beta_2^{k+1}}}+\varepsilon}$, ${\bf y}^{(k+1)} = \beta_2\, {\bf y}^{(k)}\!+ (1 - \beta_2) \,(\bg^{(k)})^{ 2}+\varepsilon$                                            \\
\rule{0pt}{2.3ex} 
AdaBelief &    $\bbeta^{(k+1)}=\tfrac{\eta}{\sqrt{\tfrac{\bs^{(k+1)}}{1-\beta_2^{k+1}}}+\varepsilon}$, ${\bf s}^{(k+1)} = \beta_2\, {\bf s}^{(k)}\!+ (1 - \beta_2) \,(\bg^{(k)}\!-{\bf m}^{(k+1)})^{ 2}+\varepsilon $ 
\\ \hline
\end{tabular}
\end{table}

\subsection{The role of gradient and curvature in choosing the step size} 
The discussion in \cite[Sec.~2.2]{zhuang2020adabelief} suggests that a smart optimizer should take into account the curvature of the loss function, rather than simply applying large updating stepsizes when the gradients are large. Indeed, the gradient only measures the local steepness of a function while the curvature indicates how fast the gradient is changing. 

\cref{fig:curvature} illustrates the updating magnitude $\Delta^{(k)}:=\vert \eta^{(k)}\,m^{(k)}\vert$ applied in the SGD and in the various adaptive optimizers for two types of objective functions of scalar argument, with different curvature near their optima.  \cref{fig:curve1} is used in \cite[Fig.~1]{zhuang2020adabelief} to give an intuition of the situations where Adabelief outperforms Adam and EAdam. In particular, the region of $x_{4},$ $x_5$, and $x_6$ is characterized by large gradients and a small curvature; here, the ideal stepsize would be large but both Adam and EAdam have to rely on small updates, see the definitions of $\bv^{(k)}$ and $\by^{(k)}$. In view of the small curvature,  Adabelief does not encounter this problem. 

However, even Adabelief might struggle in the situation depicted in \cref{fig:curve2}: both the gradient and the curvature near the optima are small and an ideal stepsize, in the region of $x_{7},$ $x_8$, and $x_9$, would be also small. In this context, the Adam, EAdam, and Adabelief optimizers use a large update of the stepsize. 
Therefore, relying only on considerations of either the gradient or curvature may not guarantee the use of an appropriate updating step size for gradient-based optimization methods.

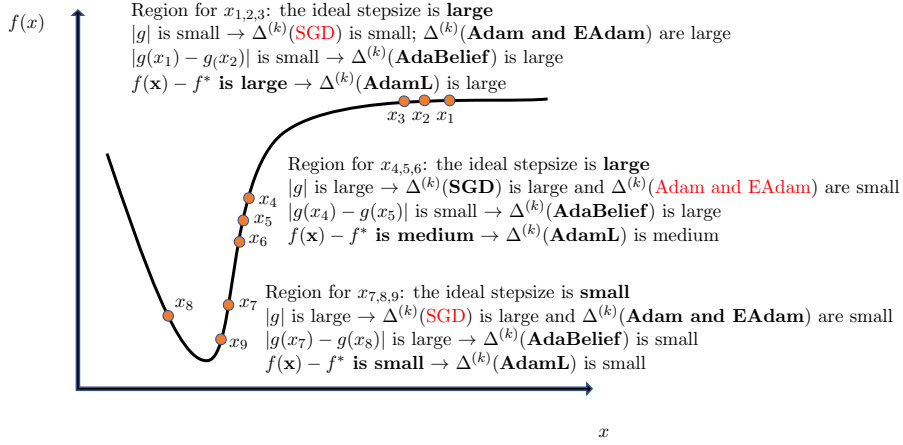
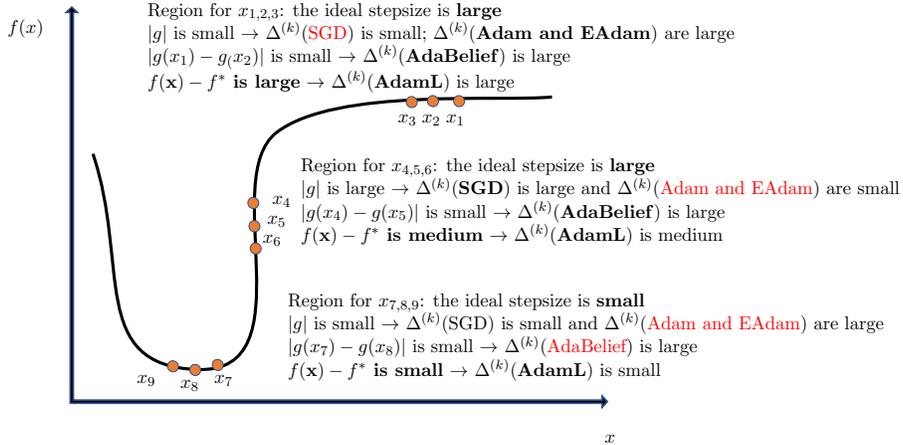
\begin{figure}[htp]
\centering
\subfloat[cf.~{\cite[fig.~1]{zhuang2020adabelief}}]{\label{fig:curve1}\resizebox{0.8\textwidth}{!}{
\begin{tikzpicture}
    \draw (0, 0) node[inner sep=0] {\includegraphics[width=0.8\textwidth]{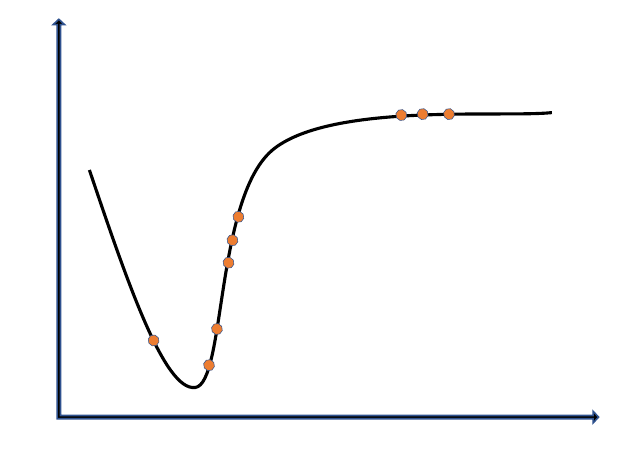}};
    \draw (2.5, 1.8) node {$x_1$};
    \draw (2, 1.8) node {$x_2$};
    \draw (1.5, 1.8) node {$x_3$};
    \draw (-1.1, 0.1) node {$x_4$};
    \draw (-1.2, -0.3) node {$x_5$};
    \draw (-1.3, -0.7) node {$x_6$};
    \draw (-1.5, -2) node {$x_7$};
    \draw (-1.7, -2.8) node {$x_9$};
    \draw (-2.8, -2) node {$x_8$};
    \draw (-6.0, 3.7) node {$f(x)$};
    \draw (5.7, -4.6) node {$x$};
    \draw [align=left] (2.2, 3.2) node {Region for $x_{1,2,3}$: the ideal stepsize is \textbf{large}\\
    $\vert g\vert$ is small $\rightarrow$  $\Delta^{(k)}$(\textcolor{red}{SGD}) is small; $\Delta^{(k)}$(\textbf{Adam and EAdam}) are large\\
    $\vert g(x_1)-g_(x_2)\vert$ is small $\rightarrow$ $\Delta^{(k)}$(\textbf{AdaBelief}) is large\\ \textbf{$f(\bx)-f^*$ is large} $\rightarrow$ $\Delta^{(k)}$(\textbf{AdamL}) is large};
    \draw [align=left] (5.5, 0.1) node {Region for $x_{4,5,6}$: the ideal stepsize is \textbf{large} \\
    $\vert g\vert$ is large $\rightarrow$ $\Delta^{(k)}$(\textbf{SGD}) is large and $\Delta^{(k)}$(\textcolor{red}{Adam and EAdam}) are small\\
    $\vert g(x_4)-g(x_5)\vert$ is small $\rightarrow$ $\Delta^{(k)}$(\textbf{AdaBelief}) is large\\ \textbf{$f(\bx)-f^*$ is medium} $\rightarrow$ $\Delta^{(k)}$(\textbf{AdamL}) is medium};
    \draw [align=left] (5.2, -2.5) node {Region for $x_{7,8,9}$: the ideal stepsize is \textbf{small}\\
    $\vert g\vert$ is large $\rightarrow$  $\Delta^{(k)}$(\textcolor{red}{SGD}) is large and $\Delta^{(k)}$(\textbf{Adam and EAdam}) are small\\
    $\vert g(x_7)-g(x_8)\vert$ is large $\rightarrow$ $\Delta^{(k)}$(\textbf{AdaBelief}) is small\\ 
    \textbf{$f(\bx)-f^*$ is small} $\rightarrow$ $\Delta^{(k)}$(\textbf{AdamL}) is small};
\end{tikzpicture}
}}\quad
\subfloat[]{\label{fig:curve2}
\resizebox{0.8\textwidth}{!}{
\begin{tikzpicture}
    \draw (0, 0) node[inner sep=0] {\includegraphics[width=0.8\textwidth]{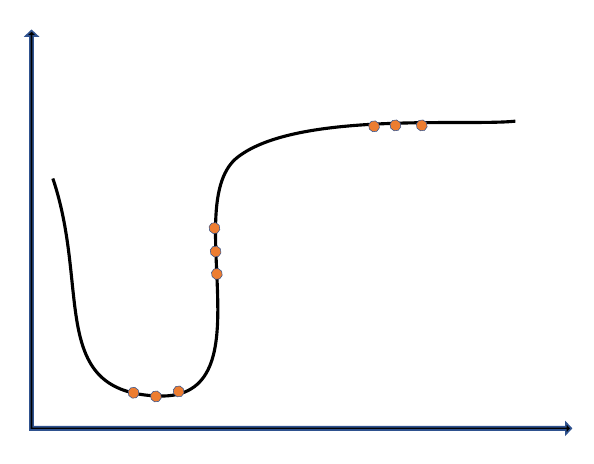}};
    \draw (2.5, 1.8) node {$x_1$};
    \draw (2, 1.8) node {$x_2$};
    \draw (1.5, 1.8) node {$x_3$};
    \draw (-1.1, 0.1) node {$x_4$};
    \draw (-1.2, -0.3) node {$x_5$};
    \draw (-1.3, -0.7) node {$x_6$};
    \draw (-2.3, -3.6) node {$x_7$};
    \draw (-3, -3.7) node {$x_8$};
    \draw (-3.9, -3.6) node {$x_9$};
    \draw (-6.4, 3.7) node {$f(x)$};
    \draw (5.7, -4.8) node {$x$};
    \draw [align=left] (2.2, 3.3) node {Region for $x_{1,2,3}$: the ideal stepsize is \textbf{large}\\
    $\vert g\vert$ is small $\rightarrow$  $\Delta^{(k)}$(\textcolor{red}{SGD}) is small; $\Delta^{(k)}$(\textbf{Adam and EAdam}) are large\\
    $\vert g(x_1)-g_(x_2)\vert$ is small $\rightarrow$ $\Delta^{(k)}$(\textbf{AdaBelief}) is large\\ \textbf{$f(\bx)-f^*$ is large} $\rightarrow$ $\Delta^{(k)}$(\textbf{AdamL}) is large};
    \draw [align=left] (5.5, 0.1) node {Region for $x_{4,5,6}$: the ideal stepsize is \textbf{large}\\ $\vert g\vert$ is large $\rightarrow$ $\Delta^{(k)}$(\textbf{SGD}) is large and $\Delta^{(k)}$(\textcolor{red}{Adam and EAdam}) are small\\
    $\vert g(x_4)-g(x_5)\vert$ is small $\rightarrow$ $\Delta^{(k)}$(\textbf{AdaBelief}) is large \\ \textbf{$f(\bx)-f^*$ is medium} $\rightarrow$ $\Delta^{(k)}$(\textbf{AdamL}) is medium};
    \draw [align=left] (5.2, -2.7) node {Region for $x_{7,8,9}$: the ideal stepsize is \textbf{small}\\ $\vert g\vert$ is small $\rightarrow$  $\Delta^{(k)}$(SGD) is small and $\Delta^{(k)}$(\textcolor{red}{Adam and EAdam}) are large\\
    $\vert g(x_7)-g(x_8)\vert$ is small $\rightarrow$ $\Delta^{(k)}$(\textcolor{red}{AdaBelief}) is large \\   
    \textbf{$f(\bx)-f^*$ is small} $\rightarrow$ $\Delta^{(k)}$(\textbf{AdamL}) is small};
\end{tikzpicture}
}}
\caption{Curvature of different functions and their ideal learning rate for different optimizers; the methods with undesired updating magnitudes are marked in red.}\label{fig:curvature}
\end{figure}

\subsection{A new adaptive gradient method incorporating the loss function}\label{sec:adaml}
In the situation described in \cref{fig:curvature},  the ideal stepsize changes proportionally with respect to the value $f(\mathbf x^{(k)})-f^*$.  This means that a large stepsize is favorable when $f(\mathbf x^{(k)})$ is much larger than $f^*$ and decreasing the stepsize is advantageous when $f(\mathbf x^{(k)})$ gets close to $f^*$. Motivated by this argument, we propose a new variant of Adam that takes into account the loss function, AdamL. The updating rule of AdamL adheres to the same template in \cref{ag:adam}, where the adaptive stepsize and the computation of the second moment estimate are performed as follows:
\begin{subequations}\label{eq:2ndmoment_adaml}
\begin{align}\label{eq:lr_adaml}
 \bbeta^{(k+1)}=\eta\,\sqrt{\tfrac{1-\beta_2^{k+1}}{\bw^{(k+1)}}},
\end{align}
\begin{align}\label{eq:2ndmoment_adamla}
{\bw}^{(k+1)} &= \beta_2\, {\bw}^{(k)} + (1 - \beta_2)\tfrac{(\bg^{(k)})^{ 2}}{\gamma\,(\ell^{(k)})^{\varphi}}+\varepsilon\,\ell^{(k)},
\end{align}
and 
\begin{align}\label{eq:ell}
    \ell^{(k)}:= \ell(f(\bx^{(k)})-f^*),
\end{align}
\end{subequations}
where $\gamma,\varphi$ are positive parameters and $\ell:\mathbb{R}^+\rightarrow \mathbb{R}^+$ is a scalar function that we call \emph{scaling function}. The choice of the scaling function and of the hyperparameters is problem dependent and has to be done to \emph{ensure that the value of $\gamma\,(\ell^{(k)})^{\varphi}$ decreases and increases as $f(\bx^{(k)})$ approaches and moves away from a neighborhood of $f^*$, respectively.} We set the power function with a positive exponent in the denominator of the second moment estimate because it exhibits a desirable property: as the $\varphi$ increases, the function becomes flatter when $\ell^{(k)}\in(0,1)$ and grows more than linearly when $\varphi>1$ and $\ell^{(k)}>1$. Subsection~\ref{sec:hyperparametersinAdamL} contains a detailed discussion of the choice of the scaling function and the hyperparameters for AdamL.

Looking at \cref{eq:2ndmoment_adaml} we see that the second moment estimate of AdamL combines the adaptive term $(1 - \beta_2)\,\tfrac{(\bg^{(k)})^{ 2}}{\gamma\,(\ell^{(k)})^{\varphi}}$ with the non-adaptive term $\varepsilon\,\ell^{(k)}$. 
In particular, AdamL behaves like SGD when $\varepsilon \ell^{(k)}$ is significantly larger than the maximum entry of the first term, and like Adam when $\varepsilon \ell^{(k)}$ is significantly lower than the minimum entry. Therefore, unlike other methods like SWATS \citep{keskar2017improving}, which sets a triggering condition, or Adabound \citep{luoadaptivebound}, which uses dynamic bounds on learning rates, AdamL can automatically switch between a non-adaptive and an adaptive method  (SGD and Adam). 
To facilitate our analysis, we consider the cases where the adaptive term is entry-wise larger than the non-adaptive one and vice versa; see the third-row entry in \cref{tab:2modes} for the rigorous relation. Qualitatively, we can describe these two updating modes as follows:
\begin{itemize}
    \item \textbf{Non-Adaptive Mode}: AdamL behaves like SGD with an increasing stepsize as $\ell^{(k)}$ decreases;
    \item \textbf{Adaptive Mode}: AdamL behaves like Adam with a decreasing stepsize as $\ell^{(k)}$ decreases.
\end{itemize}
Alternative scenarios may arise, for example, when certain gradient components exhibit Non-Adaptive Mode behavior while others adopt the Adaptive Mode, or when the non-adaptive term $\varepsilon\,\ell^{(k)}$ and the adaptive term $(1 - \beta_2)\,\tfrac{(\bg^{(k)})^{ 2}}{\gamma\,(\ell^{(k)})^{\varphi}}$ have similar magnitudes. However, the latter case is highly improbable or undesirable. Indeed, we can use the parameter $\varphi$ to adjust the speed of the transitions from the Non-Adaptive Mode to the Adaptive Mode. Specifically, a larger value of $\varphi$ results in a faster decay of $(\ell^{(k)})^{\varphi}$ when $\ell^{(k)}\in(0,1)$. This indicates that for $\ell^{(k)}\in(0,1)$, the non-adaptive term $\varepsilon\,\ell^{(k)}$ decreases by a factor of $\ell^{(k)}$ and the adaptive term $(1 - \beta_2)\,\tfrac{(\bg^{(k)})^{ 2}}{\gamma\,(\ell^{(k)})^{\varphi}}$ increases by a factor of $(\ell^{(k)})^{\varphi}$. Conversely, when $\ell^{(k)}\ge 1$, the non-adaptive term $\varepsilon\,\ell^{(k)}$ increases by of factor of $\ell^{(k)}$, and the adaptive term $(1 - \beta_2)\,\tfrac{(\bg^{(k)})^{ 2}}{\gamma\,(\ell^{(k)})^{\varphi}}$ decreases by a factor of $(\ell^{(k)})^{\varphi}$. There may be a transient where these two terms become similar, but this is likely to be short already for moderate values of $\varphi$.
 
During the early stage of the training, when $f(\bx^{(k)})$ is far from $f^*$, the Non-Adaptive Mode is beneficial since the updating stepsize increases in the direction where $f(\bx^{(k)})$ approaches $f^*$. Additionally, to prevent \emph{poor generalization} on unseen datasets, it is beneficial to have a small updating stepsize near the optimum. Therefore, it is advantageous that AdamL uses the Adaptive Mode near $f^*$. 

\subsection{Choosing the hyperparameters in AdamL} \label{sec:hyperparametersinAdamL}
\textbf{Impact of hyperparameters in AdamL.}
The hyperparameter $\gamma$ (cf.~\cref{eq:2ndmoment_adamla}) can be used to realize a smooth transition between Non-Adaptive and Adaptive Modes. In general, as $\gamma$ increases, the AdamL optimizer tends to postpone its transition from the Non-Adaptive Mode to the Adaptive Mode. In our implementation, we use as the default setting $\gamma=1$ and we increase the value of $\gamma$  only when an early transition is observed. Additionally, a larger value of $\varphi$ results in a faster decay of $(\ell^{(k)})^{\varphi}$ when $\ell^{(k)}\in(0,1)$, implying a faster switch from the Non-Adaptive Mode to the Adaptive Mode.  On the basis of \cref{eq:ell}, $\ell^{(k)}$ decreases as $f(\bx^{(k)})$ converges to $f^*$.  To activate the Adaptive Mode, the exponent $\varphi$ (cf.~\cref{eq:2ndmoment_adamla}) should be determined based on the minimum attainable value of $\ell^{(k)}$. In particular, it should guarantee that $(1 - \beta_2)\tfrac{(\bg^{(k)})^{ 2}}{\gamma\,(\ell^{(k)})^{\varphi}}$ increases faster than $\varepsilon\,\ell^{(k)}$. 

When a small updating stepsize is beneficial, we recommend scaling $\ell^{(k)}$ to make it range in $(0,1)$, so that  $(\ell^{(k)})^{\varphi}$ decreases as $f(\bx^{(k)})$ approaches $f^*$. In this case, a larger value of $\varphi$ results in an earlier switch from the Non-Adaptive Mode to the Adaptive Mode. By means of numerical experiments in \cref{sec:experiments}, we find that $\varphi=4$ or even larger values are typically good choices for training CNNs, WGANs, and LSTMs.   

\vspace{2mm}
\noindent\textbf{Choice of the scaling function.}  In some deep learning tasks, such as image classification and regression, the minimum cost function value can be very close to zero. For instance, in the context of image classification, it is common to use activation functions like sigmoid or softmax at the output layer and the cross-entropy loss function as the cost function. Therefore, the range of the loss function is often distributed between $(0,1)$. In this case, one can directly use $\ell^{(k)}=f(\bx^{(k)})$. However, in several training tasks, $f^*$ is unknown and one needs at least an estimate of the range of the cost function to define $\ell^{(k)}$.  
A straightforward approach is to first run Adam and store the quantities $f_{\max}$ and $f_{\min}$ representing the highest and smallest value of the cost function observed throughout Adam's execution, respectively; then set $\ell^{(k)}:=\frac{f(\bx^{(k)})-f_{\min}}{f_{\max}-f_{\min}}$. Although this approach requires running Adam, AdamL consistently yields higher training and testing accuracy compared to the first run. 

In the following we present an alternative method to estimate the range of the cost function and determining $\ell^{(k)}$, in the context of various training tasks. In cases where the cost function values are consistently positive, and the minimum cost function value is unknown, we take $f_{\min}=0$. Then, we set $\ell^{(0)} = 1$ for the initial training epoch, and choose $\ell^{(k)} = \frac{f(\bx^{(k)})}{f_{\max}}$ for the remaining training epochs, where $f_{\max}$ indicates the maximum cost function value observed in the first epoch. We remark that in the (unusual) case that full batch size is applied, $f_{\max}$ can be estimated by setting $\ell^{(j)} = 1$ for a few iterations, e.g., $j=0,\dots,k_0$ and $f_{\max}=\max_{j=0,\dots,k_0} f(\bx^{(j)})$. This ensures that $\ell^{(k)}$ falls approximately within the range $(0,1]$. We will assess the effectiveness of this strategy by training 1-, 2-, and 3-layer LSTMs on \textsf{Penn Treebank} in \cref{sec:experiments}. 

 A pseudocode implementing AdamL in scenarios where $f_{\max}$ and $f_{\min}$ are unknown, is reported in \cref{ag:strategyI}. In our numerical tests, we will use this procedure to train LSTMs and WGANs (see also next paragraph).

\noindent\textbf{Applying AdamL to WGANs.} The AdamL optimizer can be adapted to address min-max problems arising in the training of WGANs. In this context, the generator aims to minimize the Wasserstein distance (or critic's output) \cite[Eq.~(2)]{wgan2017}, while the discriminator/critic aims to maximize it. A detailed description of the WGAN procedure is given in \cite[Algorithm.~1]{wgan2017}. 
We emphasize that we apply AdamL only for updating the parameters of the discriminator. This is because, in WGAN, the parameters of the discriminator are more frequently updated compared to that of the generator (cf.~\cite[Algorithm~1]{wgan2017}), and this enables us to get a more accurate approximation of the Wasserstein distance \cite[Sec.~3]{wgan2017}.
For the generator, we set 
$\ell^{(k)}=1$ for all $k$, which is equivalent to the EAdam optimizer. 

To shed some light on the setting for AdamL, we assume that the Wasserstein distance of the discriminator is $f_1(d_{\text{real}},\bx)-f_1(d_{\text{fake}},\bx)$, where $d_{\text{real}}$ and $d_{\text{fake}}$ denote the real and fake samples, respectively; $f_1$ represent the function that generates the discriminator's output.
The discriminator aims to maximize $f_1(d_{\text{real}},\bx)-f_1(d_{\text{fake}},\bx)$, which can be interpreted as maximizing $f_1(d_{\text{real}},\bx)$ and minimizing $f_1(d_{\text{fake}},\bx)$ simultaneously. By defining $f(\bx^{(k)}):=f_1(d_{\text{fake}},\bx^{(k)})$ and $\ell^{(k)}=\frac{f(\bx^{(k)})-f_{\min}}{f_{scale}}$, where $f_{scale}:=10^{\lfloor\log_{10}(f_{\max}-f_{\min})\rfloor}$, we have that $\ell^{(k)}$ is scaled to approximately fall within the range of $(0,1)$. This scaling ensures that AdamL uses a smaller stepsize when $f(\bx^{(k)})$ decreases. Reducing the updating stepsize yields smoother training progress, which often leads to improved convergence towards the desired equilibrium between the generator and discriminator.
Note that $f_{\min}$ and $f_{\max}$ can be easily estimated by setting $\ell^{(0)}=1$ for the first training epoch because, when training a GAN, it is common to observe significant variation in the discriminator's output during the first training epoch. As we will see in the numerical tests of  \cref{sec:experiments}, typically, there is no need to frequently update the maximum or minimum values of the discriminator's output in the rest of the training process. This is because, instead of consistently decreasing the value of the discriminator's (or critic's) output, GANs are designed to gradually improve and reach a more stable equilibrium between the discriminator and generator. 
\begin{algorithm}[htp]
{\small
\begin{algorithmic}
\STATE \textbf{Input:} $\bx^{(0)}$, hyperparameters $\beta_1, \beta_2 \in (0,1)$, $\varepsilon, \eta, \gamma, \varphi >0$, number of training epochs $n_{\mathrm{epoch}}>0$, and number of iterations per epoch $n_{\mathrm{iter}}>0$ (based on mini-batch size)\\
\textbf{Initialize} $\mathbf{ m}^{(0)}, \mathbf{ w}^{(0)}, \mathbf{ s}^{(0)}= \mathbf{ 0}$\\
\textbf{For}  $k=0,\dots, n_{\mathrm{epoch}}$
\STATE\hspace{10mm}\textbf{If} $k=0$
\STATE\hspace{15mm}\textbf{For} $j=1,\dots,n_{\mathrm{iter}}$
\STATE\hspace{25mm}set $\ell^{(j)},\gamma=1$, 
\STATE\hspace{25mm}$\bg^{(j)} = \nabla_{\bx} F(\bx^{(j)},\xi^{(j)})$ 
\STATE\hspace{25mm}$\mathbf{ m}^{(j+1)} = \beta_1 \,\mathbf{ m}^{(j)} + (1 - \beta_1) \,\bg^{(j)}$ 
\STATE\hspace{25mm}${\bw}^{(j+1)} = \beta_2\, {\bw}^{(j)} + (1 - \beta_2)\tfrac{(\bg^{(j)})^{ 2}}{\gamma\,(\ell^{(j)})^{\varphi}}+\varepsilon\,\ell^{(j)}$
\STATE\hspace{25mm}$\bx^{(j+1)}= \bx^{(j)} -\eta\,\frac{\sqrt{1-\beta_2^{j+1}}}{1-\beta_1^{j+1} }\tfrac{\mathbf{m}^{(j+1)}}{\sqrt{{\bw}^{(j+1)}}}$
\STATE\hspace{25mm}\textbf{If} $j=1$ 
\STATE\hspace{30mm}$f_{\min}=F(\bx^{(1)},\xi^{(1)})$ and $f_{\max}=F(\bx^{(1)},\xi^{(1)})$
\STATE\hspace{25mm}\textbf{Else}
\STATE\hspace{30mm}$f_{\min}=\min(f_{\min},F(\bx^{(j)},\xi^{(j)}))$ 
\STATE\hspace{30mm}$f_{\max}=\max(f_{\max},F(\bx^{(j)},\xi^{(j)}))$
\STATE\hspace{25mm}\textbf{End}
\STATE\hspace{15mm}\textbf{End} 
\STATE\hspace{10mm}\textbf{Else}
\STATE\hspace{15mm}\textbf{For} $j=k\,n_{\mathrm{iter}}+1, \dots, (k+1)\, n_{\mathrm{iter}}$
\STATE\hspace{25mm}set $\ell^{(j)} = \frac{f(\bx^{(j)})}{f_{\max}}$ (LSTMs) and $\ell^{(j)}=\frac{f(\bx^{(j)})-f_{\min}}{f_{\max}-f_{\min}}$ (WGANs)
\STATE\hspace{25mm}$\bg^{(j)} = \nabla_{\bx} F(\bx^{(j)},\xi^{(j)})$ 
\STATE\hspace{25mm}$\mathbf{m}^{(j+1)} = \beta_1 \,\mathbf{ m}^{(j)} + (1 - \beta_1) \,\bg^{(j)}$ 
\STATE\hspace{25mm}${\bw}^{(j+1)} = \beta_2\, {\bw}^{(j)} + (1 - \beta_2)\tfrac{(\bg^{(j)})^{ 2}}{\gamma\,(\ell^{(j)})^{\varphi}}+\varepsilon\,\ell^{(j)}$
\STATE\hspace{25mm}$\bx^{(j+1)}= \bx^{(j)} -\eta\,\frac{\sqrt{1-\beta_2^{j+1}}}{1-\beta_1^{j+1} }\tfrac{\mathbf{m}^{(j+1)}}{\sqrt{{\bw}^{(j+1)}}}$
\STATE\hspace{15mm}\textbf{End} 
\STATE\hspace{10mm}\textbf{End}\\
\textbf{End} 
\end{algorithmic}
\caption{AdamL for unknown $f_{\max}$ and $f_{\min}$ for LSTMs and WGANs}\label{ag:strategyI}}
\end{algorithm}

\subsection{Summary of the second moment estimates}\label{sec:summary_2ndmoment}
To facilitate the convergence analysis in \cref{sec:convergence}, we summarize the second moment estimates of the four optimizers, i.e., Adam, EAdam, AdaBelief, and AdamL. The starting vectors $\mathbf{v}^{(0)}$, $\mathbf{y}^{(0)}$, $\mathbf{s}^{(0)}$ and $\mathbf{w}^{(0)}$ are assumed to be initialized as zero vectors. 

\noindent The second raw moment estimate $\mathbf{v}^{(k)}$ of Adam \citep{kingma2014adam} is
\begin{align}\label{eq:2mon_adam}
  \mathbf{v}^{(k+1)}
  =(1-\beta_2)\sum_{j=0}^{k} \beta_2^{k-j} (\bg^{(j)})^{ 2}.
\end{align}
The second raw moment estimate $\mathbf{y}^{(k)}$ of EAdam \citep{yuan2020eadam} is 
\begin{align}\label{eq:2mon_eadam}
  \mathbf{y}^{(k+1)}
  =(1-\beta_2)\sum_{j=0}^{k} \beta_2^{k-j} (\bg^{(j)})^{2}+\varepsilon\sum_{j=0}^{k} \beta_2^{j}.
\end{align}
The estimation of second raw moment $\mathbf{s}^{(k)}$ of AdaBelief \citep{zhuang2020adabelief} is 
\begin{align}\label{eq:2mon_AdaBelief}
  \mathbf{s}^{(k+1)}
  &=(1-\beta_2)\sum_{j=0}^{k} \beta_2^{k-j}(\bg^{(j)}-\mathbf{m}^{(j+1)})^{2}+\varepsilon\sum_{j=0}^{k} \beta_2^{j}.
\end{align} 
The second raw moment estimate $\mathbf{w}^{(k)}$ of AdamL is defined as
\begin{align}\label{eq:2mon_ladam}
 \mathbf{w}^{(k+1)}
 &=(1-\beta_2)\sum_{j=0}^{k} \beta_2^{k-j} \tfrac{(\bg^{(j)})^{ 2}}{\gamma\,(\ell^{(j)})^{\varphi}}  +\varepsilon\,\sum_{j=0}^{k} \beta_2^{k-j}\ell^{(j)}.
\end{align}

\subsection{Adaptive and non-adaptive behavior of the optimizers}
Looking at \cref{eq:2mon_adam,eq:2mon_eadam,eq:2mon_AdaBelief,eq:2mon_ladam} we note that similarly to AdamL, the second moment estimates in EAdam and AdaBelief integrate non-adaptive and adaptive terms.  As in \cref{sec:adaml}, although there may be additional scenarios, we focus on the following two major modes:
\begin{itemize}
    \item \textbf{Non-Adaptive Mode}: the method performs like SGD with decreasing updating stepsize as the increasing number of iteration steps;
    \item \textbf{Adaptive Mode}: the method behaves like an adaptive gradient method with the adaptive updating stepsize;
\end{itemize}
that are identified by the conditions in the first and second-row entry of \cref{tab:2modes}.
 Note that for the EAdam and AdaBelief optimizers, the updating stepsize is the same in the Non-Adaptive Mode. However, they may have different triggering conditions for the transition between the Non-Adaptive mode and the Adaptive mode. In the Adaptive Mode, EAdam and AdaBelief take large updating stepsizes in the direction of small gradient and curvature, respectively. When both the curvature and gradients are sufficiently small to the extent that only the accumulated $\varepsilon$ is significant, they tend to switch to the Non-Adaptive Mode, where the performance of EAdam and AdaBelief is likely to be similar, given that their second moment estimates are identical. This can explain the comparable results of EAdam and AdaBelief in image classifications and language modeling in \citep{yuan2020eadam}. However, since AdaBelief considers the curvature of the loss function, it may perform better than EAdam in the regions of the objective function characterized by large gradients and small curvature. The proposed AdamL takes into account the magnitude of the loss function and adjusts the second raw moment estimate accordingly. In general, it proceeds an increasing stepsize in the Non-Adaptive Mode when the loss function value decreases and uses decreasing stepsize in the Adaptive Mode when the loss function value decreases.

\begin{table}[ht!]
{\footnotesize
\caption{{\footnotesize Condition for activating Adaptive and Non-Adaptive updating modes for each optimizer}}\label{tab:2modes}
\vspace{3mm}
\begin{center}
\begin{tabular}{lll}
 \cline{1-3} \rule{0pt}{2.3ex}%
Optimizer & Updating mode  & Sufficient condition \\ \cline{1-3}\rule{0pt}{2.3ex}%
\multirow{2}{*}{EAdam} &  Adaptive &
$(1-\beta_2)\sum_{j=0}^{k} \beta_2^{k-j} (\bg^{(j)})^{2}>\varepsilon\sum_{j=0}^{k} \beta_2^{j}$ \\ \rule{0pt}{2.3ex}%
&Non-Adaptive& $\varepsilon\sum_{j=0}^{k} \beta_2^{j}> (1-\beta_2)\sum_{j=0}^{k} \beta_2^{k-j} (\bg^{(j)})^{2}$\\ \rule{0pt}{4.3ex}%
\multirow{2}{*}{AdaBelief} & Adaptive&
$(1-\beta_2)\sum_{j=0}^{k} \beta_2^{k-j}(\bg^{(j)}-\mathbf{m}^{(j+1)})^{2}>\varepsilon\sum_{j=0}^{k} \beta_2^{j}$ \\ \rule{0pt}{2.3ex}%
&Non-Adaptive & $ \varepsilon\sum_{j=0}^{k} \beta_2^{j}>(1-\beta_2)\sum_{j=0}^{k} \beta_2^{k-j}(\bg^{(j)}-\mathbf{m}^{(j+1)})^{2}$ \\  \rule{0pt}{4.3ex}
\multirow{2}{*}{AdamL} &
Adaptive &$(1-\beta_2)\sum_{j=0}^{k} \beta_2^{k-j} \tfrac{(\bg^{(j)})^{ 2}}{\gamma\,(\ell^{(j)})^{\varphi}}>\varepsilon\,\sum_{j=0}^{k} \beta_2^{k-j}\ell^{(j)}$ \\ \rule{0pt}{2.3ex}%
&Non-Adaptive & $\varepsilon\,\sum_{j=0}^{k} \beta_2^{k-j}\ell^{(j)}>(1-\beta_2)\sum_{j=0}^{k} \beta_2^{k-j} \tfrac{(\bg^{(j)})^{ 2}}{\gamma\,(\ell^{(j)})^{\varphi}}$
\\\cline{1-3}
\end{tabular}
\end{center}
}
\end{table}

In the subsequent section, we conduct an analysis on how different estimations of the second raw moment influence the convergence rate under the PL condition. 

\section{Convergence analysis}\label{sec:convergence}
Throughout this section, we indicate with $k\in\mathbb N$ the number of iteration steps that have been run by the algorithm under consideration. Note that Adam, EAdam, AdaBelief, and AdamL share the same first moment estimates, as they are all based on Algorithm~\ref{ag:adam}. Therefore, the different convergence rates are only attributed to the various second raw moment estimates. In view of this and similarly to other studies \citep{cao2020convergence,reddi2018convergence,zhuang2020adabelief}, where the condition $\lim_{k\to\infty}\beta_1^{(k)} =0$ is employed to streamline the analysis, we make the simplifying assumption that $\beta_1=0$.
The convergence analysis incorporating such simplification is frequently denoted as the convergence of an ``Adam-type'' optimizer~\citep{cao2020convergence,zhuang2020adabelief}. 

In \cref{sec:summary_2ndmoment}, we observed that we can identify two modes for the updating procedure of the EAdam, AdaBelief, and AdamL; see, e.g., \cref{tab:2modes}.
It is crucial to note the possibility that certain gradient components may exhibit behavior resembling the Non-Adaptive Mode, while others adhere to the Adaptive Mode. 
However, in our analysis, we only draw connections between our theoretical results and the conditions shown in \cref{tab:2modes}. Note that for the AdaBelief optimizer, we have only the Non Adaptive mode since $\beta_1=0$ implies $\mathbf{m}^{(k+1)}=\bg^{(k)}$ and this makes its first term identically zero.

For each optimizer,  we provide sufficient conditions to ensure the monotonicity and the linear convergence of the scheme up to $\mathcal O(\sigma)$ and $\mathcal O(\sigma^2)$ terms, depending on the larger value between $\sigma$ and $\sigma^2$ (cf.~\cref{eq:simga_i}). 
We remark that looking at the coefficients in front of the $\mathcal O(\sigma)$ and $\mathcal O(\sigma^2)$ terms yields insights into the role of the parameters in determining how rapidly the method converges and how close to the optimum can get. 
A technical aspect, that will be useful for quantifying the convergence rates of the various methods, is to have non-zero lower bounds for the entries of the second moment estimates (that are all made of non-negative entries). We see that the sequences $\mathbf{y}^{(k)}, \mathbf{s}^{(k)},$ and $\mathbf{w}^{(k)}$ are ensured to be always positive and that the sought lower bounds can be retrieved by looking at the second addends of \cref{eq:2mon_eadam,eq:2mon_AdaBelief,eq:2mon_ladam}. 
To get the positivity property also for $\mathbf{v}^{(k)}$ (Adam) we have to assume that the chosen starting point $\mathbf x^{(0)}$ is such that $\mathbf g^{(0)}$ has all non-zero entries. Note that this is not too restrictive as, in most cases, the set of starting points violating this condition has zero Lebesgue measure in $\mathcal X$. In view of \cref{eq:2mon_adam}, we see that $\mathbf v^{(k)}\ge (1-\beta)\beta_2^{k-1}\mathbf g^{(0)}$ that  allows us to consider 
\begin{equation}\label{eq:vmin}
   v_{\min}(\mathbf x_0):=(1-\beta_2)\,\beta_2^{k-1}\,\min_{i=1,\dots,n} (g_i^{(0)})^2>0,
\end{equation}
as lower bound for the minimum non-zero entry among the second raw moment estimates computed over the first $k$ iteration steps of Adam, starting from $\mathbf x^{(0)}$. When the starting point is clear from the context we just write $v_{\min}$ to lighten the notation.  

Finally, other useful tools for our analysis are the positive random variables \begin{equation}\label{eq:simga_i}
    \sigma_i^2:=\expt\,[(g_i^{(k)})^2\,]-\nabla f(\bx^{(k)})_i^2,
\end{equation} so that  $\sum_{i=1}^n \sigma_i^2=\sum_{i=1}^n \big(\expt\,[(g_i^{(k)})^2\,]-(\nabla f(\bx^{(k)})_i)^2\big) =\expt\,[\, \Vert \bg^{(k)}\Vert^2\,] -\Vert \nabla f(\bx^{(k)})\Vert^2 \leq \sigma^2$ (cf.~\cref{eq:varineq}). On the basis of Jensen's inequality, concavity, and \cref{eq:simga_i}, we get 
\begin{equation}\label{eq:boundvertg}
    \expt\,[\,\vert g_i^{(k)}\vert\,]\le \sqrt{\expt\,[\,(g_i^{(k)})^2\,]} =\sqrt{\sigma_i^2+\nabla f(\bx^{(k)})_i^2}\le \vert \nabla f(\bx^{(k)})_i \vert +\sigma_i.
\end{equation} 
In the next sections, we address the convergence analysis of EAdam-type, AdaBelief-type, and AdamL-type optimizers.

\subsection{Convergence results for EAdam}

Let us begin by studying the convergence of the EAdam-type optimizer under \cref{assump:assumption}. 
 The proposition below outlines two sufficient conditions that guarantee the monotonicity of EAdam and are applicable in both the Non-Adaptive and Adaptive updating modes (see \cref{tab:2modes}). 

\begin{proposition}\label{prop:eadam_monotonicity}
Let the objective function satisfy \cref{assump:assumption} and the starting point $\mathbf x^{(0)}\in\mathbb R^{n}$ be  such that $g_i^{(0)}\neq 0~\forall i$. Assume that $\beta_1=0$ and $k\in\mathbb N$ iteration steps of the EAdam optimizer (cf.~\cref{ag:adam} with \cref{eq:2mon_eadam}) have been executed with input parameters $\beta_2\in (0,1)$, $\eta,\varepsilon\in\mathbb R^+$ such that 
one of the following two conditions is satisfied
\begin{itemize}
\item[$(i)$] 
$\sum_{j=0}^{k}\beta_2^j\ge \max\big\{\tfrac{4\,G}{\sqrt{\varepsilon}},\tfrac{2}{\sqrt{1-\beta_2}}\big\}$, and $\eta\leq \tfrac{\sqrt{\varepsilon\,\sum_{j=0}^{k}\beta_2^j}}{2\,L}$;
\item[$(ii)$]  
$v_{\min}> 4\,\varepsilon$, $\sum_{j=0}^{k}\beta_2^j\ge \max\big\{\tfrac{4\,G}{\sqrt{v_{\min}}}, \tfrac{4\,\varepsilon}{(1-\beta_2)\,v_{\min}}\big\}$, and $\eta\leq \tfrac{\sqrt{v_{\min}}}{2\,L}$;
\end{itemize}
then, the random vector $\mathbf{x}^{(k)}$ satisfies $$\expt\,[\,f(\bx^{(k)})\,]\leq 
     \expt\,[\,f(\bx^{(k-1)})\,]+\,C_{E,1}\,\sigma+\,C_{E,2}\,\sigma^2,$$
     where, in case $(i)$
    \begin{equation}\label{eq:ce12_case1}
     C_{E,1}=\sqrt{\frac{n}{\beta_2\,\varepsilon}}\,\cdot\frac{\eta \,G}{\sum_{j=0}^{k-1}\beta_2^j},\qquad C_{E,2}=\frac{\eta}{\varepsilon\sum_{j=0}^{k-1}\beta_2^j}\left(\tfrac12\,L\,\eta+G\sqrt{\frac{1-\beta_2}{\beta_2}}\right),
     \end{equation}
     and in case $(ii)$
     \begin{equation}\label{eq:ce12_case2}
     C_{E,1}=\sqrt{\frac{n\,\varepsilon}{\beta_2}}\cdot\frac{\eta\,G}{v_{\min}},\qquad C_{E,2}=\frac{\eta}{v_{\min}}\left(\tfrac12\,L\,\eta+G\sqrt{\frac{1-\beta_2}{\beta_2}}\right).
     \end{equation}
\end{proposition}
\begin{proof}
We follow the main line of the proof for \cite[Thm.~1]{zaheer2018adaptive}. On the basis of the property that $\expt\,[\,\expt\,[\,f(\bx^{(k+1)})- f(\bx^{(k)})\ \big\vert \ \bx^{(k)},\, \xi^{(k)}\,] \ \big\vert \ \bx^{(k)}\,]=\expt\,[\,f(\bx^{(k+1)})\ \big\vert \ \bx^{(k)}\,]- f(\bx^{(k)})$, the Lipschitz continuous gradient property and \cref{ag:adam} (cf.~\cref{eq:2mon_adam}), we have that 
   \begin{align}
\expt\,&[\,f(\bx^{(k+1)})\ \big\vert \ \bx^{(k)}\,]- f(\bx^{(k)})\nonumber\\&\le-\eta\,\bigg\Vert\tfrac{\nabla f(\bx^{(k)})}{\big(\sqrt{\tfrac{1-\beta_2^k}{1-\beta_2^{k+1}}\,\beta_2\,\,\wb{\by}^{(k)}}+\varepsilon\big)^{1/2}}\bigg\Vert^2+\tfrac{L\eta^2}{2}\,\expt\,\Big[\,\Big\Vert\tfrac{\bg^{(k)}}{\sqrt{\wb{\by}^{(k+1)}}+\varepsilon}\Big\Vert^{2}\ \big\vert \ \bx^{(k)}\,\Big]\nonumber\\
    &\qquad+\eta\,\vert\nabla f(\bx^{(k)})\vert^T\,\Big\vert\expt\,\Big[\,\tfrac{\bg^{(k)}}{\sqrt{\tfrac{1-\beta_2^k}{1-\beta_2^{k+1}}\,\beta_2\,\wb{\by}^{(k)}}+\varepsilon}-\tfrac{\bg^{(k)}}{\sqrt{\wb{\by}^{(k+1)}}+\varepsilon}\ \big\vert \ \bx^{(k)}\,\Big]\Big\vert.\label{eq:general_fk-f*}
\end{align}
Additionally, for $k \ge1$, we have
\begin{align}
    &\tfrac{\bg^{(k)}}{\sqrt{\tfrac{1-\beta_2^k}{1-\beta_2^{k+1}}\,\beta_2\,\wb{\by}^{(k)}}+\varepsilon}-\tfrac{\bg^{(k)}}{\sqrt{\wb{\by}^{(k+1)}}+\varepsilon}\nonumber\\
    &\le \tfrac{\vert \bg^{(k)}\vert}{\big(\sqrt{\wb{\by}^{(k+1)}}+\varepsilon\big) \left(\sqrt{\tfrac{\beta_2}{1-\beta_2^{k+1}}\,\by^{(k)}}+\varepsilon\right)\sqrt{1-\beta_2^{k+1}}}\cdot\tfrac{(1-\beta_2)\,(\bg^{(k)})^{ 2}+\varepsilon}{\sqrt{\by^{(k+1)}}+\sqrt{\beta_2\,\by^{(k)}}}\nonumber\\
    &= \tfrac{\vert \bg^{(k)}\vert}{\big(\sqrt{\wb{\by}^{(k+1)}}+\varepsilon\big) \left(\sqrt{\tfrac{\beta_2}{1-\beta_2^{k+1}}\,\by^{(k)}}+\varepsilon\right)\sqrt{1-\beta_2^{k+1}}}\cdot   \tfrac{(1-\beta_2)\,(\bg^{(k)})^{ 2}+\varepsilon}{(\beta_2\,\by^{(k)}+(1-\beta_2)\,(\bg^{(k)})^{ 2}+\varepsilon)^{1/2}+\sqrt{\beta_2\,\by^{(k)}}}\nonumber\\
    &\leq \tfrac{\sqrt{1-\beta_2}\,(\bg^{(k)})^{ 2}+\sqrt{\varepsilon}\,\vert \bg^{(k)}\vert}{\big(\sqrt{\wb{\by}^{(k+1)}}+\varepsilon\big) \left(\sqrt{\tfrac{\beta_2}{1-\beta_2^{k+1}}\,\by^{(k)}}+\varepsilon\right)\sqrt{1-\beta_2^{k+1}}}.
     \label{eq:ineq_eadam}
    \end{align}
Substituting it into \cref{eq:general_fk-f*}, we achieve the following upper bound:
\begin{align}
    \expt\,&[\,f(\bx^{(k+1)})\ \big\vert \ \bx^{(k)}\,]-f(\bx^{(k)})\nonumber\\&\le  -\eta\,\bigg\Vert\tfrac{\nabla f(\bx^{(k)})}{\big(\sqrt{\tfrac{\beta_2\,\by^{(k)}}{1-\beta_2^{k+1}}}+\varepsilon\big)^{1/2}}\bigg\Vert^2 +\tfrac{L\eta^2}{2}\,\expt\,\Big[\,\Big\Vert\tfrac{\bg^{(k)}}{\sqrt{\wb{\by}^{(k+1)}}+\varepsilon}\Big\Vert^{2}\ \big\vert \ \bx^{(k)}\,\Big]\nonumber\\
    &\qquad+\eta\,\big\vert\nabla f(\bx^{(k)}) \big\vert^T\,\expt\,\bigg[\,\tfrac{\sqrt{1-\beta_2}\,(\bg^{(k)})^{ 2}+\sqrt{\varepsilon}\,\vert \bg^{(k)}\vert}{\big(\sqrt{\wb{\by}^{(k+1)}}+\varepsilon\big) \left(\sqrt{\tfrac{\beta_2}{1-\beta_2^{k+1}}\,\by^{(k)}}+\varepsilon\right)\sqrt{1-\beta_2^{k+1}}}\,\bigg].
    \label{eq:fk+1_eadam}
    \end{align}
Now, we divide the proof according to the assumptions $(i)$ and $(ii)$.   

\noindent \fbox{$(i)$} Note that \cref{eq:2mon_eadam} implies $y_i^{(k)}\geq \varepsilon\,\sum_{j=0}^{k-1}\beta_2^j$ for all $k$, which yields
\begin{align}
    \expt\,&[\,f(\bx^{(k+1)}) \ \big\vert \ \bx^{(k)}\,]-f(\bx^{(k)})\nonumber\\
    &\le -\eta\,\Bigg\Vert\tfrac{\nabla f(\bx^{(k)})}{\big(\sqrt{\tfrac{\beta_2\,\by^{(k)}}{1-\beta_2^{k+1}}}+\varepsilon\big)^{1/2}}\Bigg\Vert^2 +\tfrac{L\eta^2}{2\,\sqrt{\varepsilon\,\sum_{j=0}^{k}\beta_2^j}}\,\expt\,\Big[\,\Big\Vert\tfrac{\bg^{(k)}}{(\sqrt{\wb{\by}^{(k+1)}}+\varepsilon)^{1/2}}\Big\Vert^{2}\ \big\vert \ \bx^{(k)}\,\Big]\nonumber\\
    & \qquad+\tfrac{\eta\,\sqrt{1-\beta_2}\,G}{\sqrt{\varepsilon\,\sum_{j=0}^{k}\beta_2^j}\,\sqrt{1-\beta_2^{k+1}}}\,\expt\,\bigg[\,\tfrac{(\bg^{(k)})^{ 2}}{\sqrt{\tfrac{\beta_2}{1-\beta_2^{k+1}}\,\by^{(k)}}+\varepsilon} \ \bigg\vert \ \bx^{(k)}\,\bigg]\nonumber\\
    & \qquad+\tfrac{\eta\sqrt{\varepsilon}}{\sqrt{\varepsilon\,\sum_{j=0}^{k}\beta_2^j}\,\sqrt{1-\beta_2^{k+1}}}\,\big\vert\nabla f(\bx^{(k)}) \big\vert^T\,\expt\,\bigg[\,\tfrac{\vert \bg^{(k)}\vert}{ \sqrt{\tfrac{\beta_2}{1-\beta_2^{k+1}}\,\by^{(k)}}+\varepsilon}\,\bigg]\nonumber.
\end{align}
In view of the properties  $\wb{\by}^{(k+1)}=\tfrac{\by^{(k+1)}}{1-\beta_2^{k+1}}\ge \tfrac{\beta_2\,\by^{(k)}}{1-\beta_2^{k+1}}$, \cref{eq:simga_i}, and \cref{eq:boundvertg}, we have that
\begin{align}
  \expt\,[&\,f(\bx^{(k+1)}) \ \big\vert \ \bx^{(k)}\,]-f(\bx^{(k)}) \nonumber\\  &-\eta\,\Bigg\Vert\tfrac{\nabla f(\bx^{(k)})}{\big(\sqrt{\tfrac{\beta_2\,\by^{(k)}}{1-\beta_2^{k+1}}}+\varepsilon\big)^{1/2}}\Bigg\Vert^2 +\tfrac{L\eta^2}{2\,\sqrt{\varepsilon\,\sum_{j=0}^{k}\beta_2^j}}\Big(\,\Bigg\Vert\tfrac{\nabla f(\bx^{(k)})}{\big(\sqrt{\tfrac{\beta_2\,\by^{(k)}}{1-\beta_2^{k+1}}}+\varepsilon\big)^{1/2}}\Bigg\Vert^2+\sum_{i=1}^{d}\tfrac{\sigma_i^2}{\sqrt{\,y^{(k+1)}_i}+\varepsilon}\Big)\nonumber\\
    &\qquad+\tfrac{\eta\,\sqrt{1-\beta_2}\,G}{\sqrt{\varepsilon\,\sum_{j=0}^{k}\beta_2^j}\,\sqrt{1-\beta_2^{k+1}}}\,\,\Big(\,\Bigg\Vert\tfrac{\nabla f(\bx^{(k)})}{(\sqrt{\tfrac{\beta_2\,\by^{(k)}}{1-\beta_2^{k+1}}}+\varepsilon)^{1/2}}\Bigg\Vert^2+\sum_{i=1}^{d}\tfrac{\sigma_i^2}{\sqrt{\tfrac{\beta_2}{1-\beta_2^{k+1}}\,y_i^{(k)}}+\varepsilon}\Big)\nonumber\\
    &\qquad+\tfrac{\eta\,\sqrt{\varepsilon}}{\sqrt{\varepsilon\,\sum_{j=0}^{k}\beta_2^j}\,\sqrt{1-\beta_2^{k+1}}}\,\Big(\,\Bigg\Vert\tfrac{\nabla f(\bx^{(k)})}{\big(\sqrt{\tfrac{\beta_2\,\by^{(k)}}{1-\beta_2^{k+1}}}+\varepsilon\big)^{1/2}}\Bigg\Vert^2+\sum_{i=1}^{d}\tfrac{G\,\sigma_i}{ \sqrt{\tfrac{\beta_2}{1-\beta_2^{k+1}}\,y_i^{(k)}}+\varepsilon}\Big)\nonumber\nonumber.
\end{align}
Furthermore, \cref{eq:simga_i} yields $\sum_{i=1}^n \sigma_i\leq \sqrt{n}\,\sigma$. Substituting it into the equation above and considering the properties $\sqrt{\beta_2\,y^{(k)}_i}\geq\sqrt{\beta_2\,\varepsilon\,\sum_{j=0}^{k-1}\beta_2^j}$ and $1-\beta_2^{k+1}=(1-\beta_2)\sum_{j=0}^{k}\beta_2^j$, it follows that
{\small
\begin{align}
    \expt\,[\,f(\bx^{(k+1)})\ \big\vert \ \bx^{(k)}\,]-
    f(\bx^{(k)})
    &\leq -\eta\,\Bigg(1-\tfrac{2\,\Big(\tfrac{G}{\,\sqrt{\sum_{j=0}^{k}\beta_2^j}}+\sqrt{\tfrac{\varepsilon}{1-\beta_2^{k+1}}}\Big)+L\,\eta}{2\,\sqrt{\varepsilon\,\sum_{j=0}^{k}\beta_2^j}}\Bigg)\Bigg\Vert\tfrac{\nabla f(\bx^{(k)})}{\big(\sqrt{\tfrac{\beta_2\,\by^{(k)}}{1-\beta_2^{k+1}}}+\varepsilon\big)^{1/2}}\Bigg\Vert^2\nonumber\\
    &\qquad+\Big(\tfrac{L\eta^2}{2\,\varepsilon\,\sum_{j=0}^{k-1}\beta_2^j}+\tfrac{\eta\,G\,\sqrt{1-\beta_2}}{\sqrt{\beta_2}\,\varepsilon\,\sum_{j=0}^{k-1}\beta_2^j} \Big)\,\sigma^2+\tfrac{\eta\,G\,\sqrt{n}}{\sqrt{\beta_2\,\varepsilon\,}\sum_{j=0}^{k-1}\beta_2^j\,}\,\sigma.
    \label{eq:cond1_fxk+1_eadam}
\end{align}}\noindent
The inequalities $\sum_{j=0}^{k}\beta_2^j\ge\tfrac{4\,G}{\sqrt{\varepsilon}}$ and $\eta\le \tfrac{\sqrt{\varepsilon\,\sum_{j=0}^{k}\beta_2^j}}{2\,L}$, imply that
\begin{align}
    \expt\,[\,f(\bx^{(k+1)})\ \big\vert \ \bx^{(k)}\,]-
    f(\bx^{(k)})
    &\leq -\eta\,\Big(\tfrac{1}{2}-\tfrac{1}{\sqrt{\sum_{j=0}^{k}\beta_2^j}\,\sqrt{1-\beta_2^{k+1}}}\Big)\Bigg\Vert\tfrac{\nabla f(\bx^{(k)})}{\big(\sqrt{\tfrac{\beta_2\,\by^{(k)}}{1-\beta_2^{k+1}}}+\varepsilon\big)^{1/2}}\Bigg\Vert^2\nonumber\\ 
    &\qquad+\underbrace{\tfrac{\eta\,G\,\sqrt{n}}{\sqrt{\beta_2\,\varepsilon\,}\sum_{j=0}^{k-1}\beta_2^j\,}}_{C_{E,1}}\,\sigma+\underbrace{\tfrac{\eta}{\varepsilon\,\sum_{j=0}^{k-1}\beta_2^j}\Big(\tfrac{L\,\eta}{2}+\tfrac{G\,\sqrt{1-\beta_2}}{\sqrt{\beta_2}} \Big)}_{C_{E,2}}\,\sigma^2.\label{eq:case1_eadam}
\end{align}
Finally, $\sum_{j=0}^k\beta_2^j\ge\tfrac{2}{\sqrt{1-\beta_2}}$ and $1-\beta_2^{k+1}=(1-\beta_2)\sum_{j=0}^{k}\beta_2^j$ imply that $\tfrac{1}{2}\ge\tfrac{1}{\sqrt{\sum_{j=0}^{k}\beta_2^j}\,\sqrt{1-\beta_2^{k+1}}}$, which gives the claim.

\noindent \fbox{$(ii)$} In view of $y_i^{(k)} \ge v_{\min}$ and \cref{eq:fk+1_eadam}, we have that  $\expt\,[\,f(\bx^{(k+1)}) \ \big\vert \ \bx^{(k)}\,]-f(\bx^{(k)})$ is bounded from above by
\begin{align}
& -\eta\,\Bigg\Vert\tfrac{\nabla f(\bx^{(k)})}{\big(\sqrt{\tfrac{\beta_2\,\by^{(k)}}{1-\beta_2^{k+1}}}+\varepsilon\big)^{1/2}}\Bigg\Vert^2 +\tfrac{L\eta^2}{2\,(\sqrt{v_{\min}}+\varepsilon)}\,\expt\,\Big[\,\Big\Vert\tfrac{\bg^{(k)}}{(\sqrt{\wb{\by}^{(k+1)}}+\varepsilon)^{1/2}}\Big\Vert^{2}\ \big\vert \ \bx^{(k)}\,\Big]\nonumber\\
    &\qquad+\tfrac{\eta\,\sqrt{1-\beta_2}\,G}{(\sqrt{v_{\min}}+\varepsilon)\,\sqrt{1-\beta_2^{k+1}}}\,\expt\,\bigg[\,\tfrac{(\bg^{(k)})^{ 2}}{\sqrt{\tfrac{\beta_2}{1-\beta_2^{k+1}}\,\by^{(k)}}+\varepsilon} \ \bigg\vert \ \bx^{(k)}\,\bigg]\nonumber\\
    &\qquad+\tfrac{\eta}{(\sqrt{v_{\min}}+\varepsilon)\,\sqrt{1-\beta_2^{k+1}}}\,\big\vert\nabla f(\bx^{(k)}) \big\vert^T\,\expt\,\bigg[\,\tfrac{\sqrt{\varepsilon}\,\vert \bg^{(k)}\vert}{ \sqrt{\tfrac{\beta_2}{1-\beta_2^{k+1}}\,\by^{(k)}}+\varepsilon}\,\bigg]\nonumber.
\end{align}
Following an analogous argument to the one used to get \cref{eq:cond1_fxk+1_eadam}, we obtain
    \begin{align}
    \expt\,[\,&f(\bx^{(k+1)}) \ \big\vert \ \bx^{(k)}\,]-f(\bx^{(k)})\nonumber\\
    &\le-\eta\,\Big(1-\tfrac{\sqrt{1-\beta_2}\,G+\sqrt{\varepsilon}}{(\sqrt{v_{\min}}+\varepsilon)\,\sqrt{1-\beta_2^{k+1}}}-\tfrac{L\eta}{2\,(\sqrt{v_{\min}}+\varepsilon)}\Big)\,\Big\Vert\tfrac{\nabla f(\bx^{(k)})}{\big(\sqrt{\tfrac{\beta_2\,\by^{(k)}}{1-\beta_2^{k+1}}}+\varepsilon\big)^{1/2}}\Big\Vert^2 \nonumber\\
    &\qquad+\Big(\tfrac{L\eta^2}{2\,(\sqrt{v_{\min}}+\varepsilon)^2}+\tfrac{\eta\,\sqrt{1-\beta_2}\,G}{\sqrt{\beta_2}\,(\sqrt{v_{\min}}+\varepsilon)^2\,}\Big)\,\sigma^2+\tfrac{\eta\,\sqrt{\varepsilon}\,G\,\sqrt{n}}{\sqrt{\beta_2}\,(\sqrt{v_{\min}}+\varepsilon)^2\,}\,\sigma\nonumber.
    \end{align}
The inequalities $\sum_{j=0}^{k} \beta_2^j\ge\tfrac{4\,G}{\sqrt{v_{\min}}}$ and $\eta\le \tfrac{\sqrt{v_{\min}}}{2\,L}$ imply that 
\begin{align}
    \expt\,[\,f(\bx^{(k+1)}) \ \big\vert \ \bx^{(k)}\,]-f(\bx^{(k)})&\le-\eta\,\Big(\tfrac{1}{2}-\tfrac{\sqrt{\varepsilon}}{\sqrt{(1-\beta_2^{k+1})\,v_{\min}}}\Big)\,\Big\Vert\tfrac{\nabla f(\bx^{(k)})}{\big(\sqrt{\tfrac{\beta_2\,\by^{(k)}}{1-\beta_2^{k+1}}}+\varepsilon\big)^{1/2}}\Big\Vert^2 \nonumber\\
    &\qquad+\underbrace{\tfrac{\eta\,\sqrt{\varepsilon}\,G\,\sqrt{n}}{\sqrt{\beta_2}\,v_{\min}}}_{C_{E,1}}\,\sigma+\underbrace{\tfrac{\eta}{v_{\min}}\Big(\tfrac{L\,\eta}{2}+\tfrac{\sqrt{1-\beta_2}\,G}{\sqrt{\beta_2}}\Big)}_{C_{E,2}}\,\sigma^2.\label{eq:case2_eadam}
    \end{align}
Therefore, when the properties $v_{\min}> 4\,\varepsilon$ and $\sum_{j=0}^k\beta_2^j\ge\tfrac{4\,\varepsilon}{(1-\beta_2)\,v_{\min}}$ hold we get the claim.
\end{proof}
\cref{prop:eadam_monotonicity} shows two sufficient conditions that ensure  monotonicity  until an  $\mathcal O(\sigma)$- or $\mathcal O(\sigma^2)$-neighborhood. The explicit expressions of $C_{E,1}$ and $C_{E,2}$ suggest that as $\eta$ decreases, the $\sigma$-dependent terms becomes negligible. 
When the EAdam optimizer employs the Non-Adaptive Mode, it is likely that condition $(ii)$ of \cref{prop:eadam_monotonicity} is not verified; in this scenario we can still ensure monotonicity when the relations $\sum_{j=0}^{k}\beta_2^j\ge \max\{\tfrac{4\,G}{\sqrt{\varepsilon}},\tfrac{2}{\sqrt{1-\beta_2}}\}$, and $\eta\leq \tfrac{\sqrt{\varepsilon\,\sum_{j=0}^{k}\beta_2^j}}{2\,L}$ hold true. In the concrete case $\beta_2=0.999$ (the recommended value in Adam), \cref{prop:eadam_monotonicity} says that to achieve monotonicity, $k$ should be chosen at least larger than $65$, so that $\sum_{j=0}^k\beta_2^j\ge\tfrac{2}{\sqrt{1-\beta_2}}$ holds. However, the condition $\sum_{j=0}^{k}\beta_2^j\ge\tfrac{4\,G}{\sqrt{\varepsilon}}$ may yield a much larger lower bound of $k$, particularly when a small $\varepsilon$ and/or a large $G$ are employed. 
On the other hand, even the necessity for a $k$ larger than a few thousands to achieve monotonicity, is not detrimental in many machine learning tasks, where mini-batch sizes are applied in the training. As an example, when training neural networks with a batch size of $64$ on a training dataset containing 60,000 samples, the EAdam-type optimizer updates the iterates approximately $938$ times per training epoch.
Additionally, a larger value of $k$ yields a larger upper bound of $\eta$, i.e., $\tfrac{\sqrt{\varepsilon\,\sum_{j=0}^{k}\beta_2^j}}{2\,L}$, which indicates that a larger $\eta$ is allowed.
When the EAdam optimizer uses the Adaptive Mode, i.e., $(1-\beta_2)\sum_{j=0}^{k} \beta_2^{k-j} (\bg^{(j)})^{2} >\varepsilon\sum_{j=0}^{k} \beta_2^{j}$, we have a higher chance to match the condition $v_{\min}\ge 4\,\varepsilon$ and the other relations in case $(ii)$ of \cref{prop:eadam_monotonicity}. If $v_{\min}$ is not close to zero, in view of the usual choices for $\varepsilon$, using case $(ii)$ has the potential to ensure monotonicity with significantly smaller values for $k$ than using case $(i)$. 

Next, we analyze the convergence rate of the EAdam-type optimizers under the PL condition. We show that by decreasing the value of $\eta$, a more accurate approximation of the optimal value can be obtained. Moreover, we demonstrate that EAdam exhibits a linear convergence rate, applicable for both updating strategies presented in \cref{tab:2modes}. 

\begin{proposition}\label{prop:eadam_convergencerate}
Under the same assumptions of \cref{prop:eadam_monotonicity}, if one of the following two conditions is satisfied
\begin{itemize}
\item[$(i)$] 
$\sum_{j=0}^{k}\beta_2^j\ge \tfrac{4\,G}{\sqrt{\varepsilon}}$, $\eta\leq \min\big\{\tfrac{\sqrt{\varepsilon\,\sum_{j=0}^{k}\beta_2^j}}{2\,L},\tfrac{\sqrt{\beta_2}\,G}{\mu}\big\}$
and $k_0\in\mathbb N$ is such that $k_0>\log_{\beta_2}(1-2\,\sqrt{1-\beta_2})$;
\item[$(ii)$] 
$v_{\min}> 4\,\varepsilon$, $\sum_{j=0}^{k}\beta_2^j\ge \tfrac{4\,G}{\sqrt{v_{\min}}}$, $\eta\leq \min\big\{\tfrac{\sqrt{v_{\min}}}{2\,L},\tfrac{\sqrt{\beta_2}\,G}{\mu}\big\}$, 
and $k_0\in\mathbb N$ is such that $k_0>\log_{\beta_2}(1-\tfrac{4\,\varepsilon}{v_{\min}})$;
\end{itemize}
then for any $k>k_0$, the random vector $\mathbf{x}^{(k)}$ satisfies {\small \begin{align}\label{eq:linearconvergence_eadam}
\expt\,[\,f(\bx^{(k)})-f^*\,]&\leq (1-2\,\mu\,C_{E})^{k-k_0}\,\expt\,[\,f(\bx^{(k_0)})-f^*\,]
 +\frac{C_{E,1}}{2\,\mu\,C_{E}}\,\sigma+\frac{C_{E,2}}{2\,\mu\,C_{E}}\,\sigma^2,
\end{align}}\noindent
 where in case $(i)$ we have that $C_{E,1}$ and $C_{E,2}$ satisfy \cref{eq:ce12_case1} by replacing $k$ with $k_0$ and
\begin{align*}
C_{E}&=\Big(\tfrac{1}{2}-\tfrac{1}{\sqrt{(1-\beta_2^{k_0+1})\sum_{j=0}^{k_0}\beta_2^j}}\Big)\,\tfrac{\eta}{\sqrt{\beta_2}\,\big(G+\sqrt{\tfrac{\varepsilon}{1-\beta_2}}\big)},
\end{align*}
and in case $(ii)$ we have that $C_{E,1}$ and $C_{E,2}$ satisfy \cref{eq:ce12_case2} and
\begin{align*}
C_{E}&=\Big(\tfrac{1}{2}-\sqrt{\tfrac{\varepsilon}{v_{\min}(1-\beta_2^{k_0+1})}}\Big)\,\tfrac{\eta}{\sqrt{\beta_2}\,\big(G+\sqrt{\tfrac{\varepsilon}{1-\beta_2}}\big)},
\end{align*}
and $0<2\,\mu\,C_{E}<1$ for both cases. 
\end{proposition}

\begin{proof}
First note that the conditions $(i)$ and $(ii)$ imply the corresponding claims of \cref{prop:eadam_monotonicity}. On the basis of \cref{eq:2mon_eadam}, we have that 
{\small
\begin{align*}
        \tfrac{1}{1-\beta_2^{k+1}}\,\by^{(k)}\leq \tfrac{1}{1-\beta_2^{k}}\,\by^{(k)}=\tfrac{(1-\beta_2)}{1-\beta_2^k}\sum_{j=0}^{k-1} \beta_2^{k-1-j} \big((\bg^{(j)})^{2} +\tfrac{\varepsilon}{1-\beta_2}\big)\le\tfrac{(1-\beta_2)}{1-\beta_2^k}\sum_{j=0}^{k-1} \beta_2^{j} \big(G^2 +\tfrac{\varepsilon}{1-\beta_2}\big).
    \end{align*} }
    The latter and the property $(1-\beta_2)\sum_{j=0}^{k-1} \beta_2^{j}=1-\beta_2^k$ give the chain of inequalities $\sqrt{\tfrac{\beta_2}{1-\beta_2^{k+1}}\,y_i^{(k)}}\leq \sqrt{\beta_2\,(G^2+\tfrac{\varepsilon}{1-\beta_2})}\le \sqrt{\beta_2}\,(G+\sqrt{\tfrac{\varepsilon}{1-\beta_2}})$. 
    
\noindent\fbox{$(i)$} By applying the previous inequality to \cref{eq:case1_eadam} and taking the expectation on both sides, we have that 
\begin{align}
    \expt\,[&\,f(\bx^{(k+1)})\,]-
    \expt\,[\,f(\bx^{(k)})\,]\nonumber\\
    &\leq -\underbrace{\eta\,\Big(\tfrac{1}{2}-\tfrac{1}{\sqrt{\sum_{j=0}^{k}\beta_2^j}\,\sqrt{1-\beta_2^{k+1}}}\Big)\,\tfrac{1}{\sqrt{\beta_2}\,(G+\sqrt{\tfrac{\varepsilon}{1-\beta_2}})}}_{C_{E}(k)}\expt\,[\,\Vert\nabla f(\bx^{(k)})\Vert^2\,]\nonumber\\   &\qquad+\underbrace{\tfrac{\eta\,G\,\sqrt{n}}{\sqrt{\beta_2\,\varepsilon\,}\sum_{j=0}^{k-1}\beta_2^j\,}}_{C_{E,1}(k)}\,\sigma+\underbrace{\tfrac{\eta}{\varepsilon\,\sum_{j=0}^{k-1}\beta_2^j}\Big(\tfrac{L\,\eta}{2}+\tfrac{G\,\sqrt{1-\beta_2}}{\sqrt{\beta_2}} \Big)}_{C_{E,2}(k)}\sigma^2. \nonumber
\end{align}
Property \cref{eq:PLineq} yields $\expt\,[\,\Vert \nabla f(\bx^{(k)})\Vert^2\,]\geq 2\mu \,\expt\,[\,f(\bx^{(k)})-f^*\,]$. Additionally, the value of $C_{E}(k)$ increases as $k$ increases, while $C_{E,1}(k)$ and $C_{E,2}(k)$ decrease as $k$ increases. For a $k_0<k$, it follows that
{\small
\begin{align}
    \expt\,[\,f(\bx^{(k+1)})-f^*\,]\leq (1-2\,\mu\,C_{E}(k_0))\,\expt\,[\,f(\bx^{(k)})-f^*\,]+C_{E,1}(k_0)\,\sigma+ C_{E,2}(k_0)\,\sigma^2.\nonumber
\end{align}}\noindent 
The property $\log_{\beta_2}(1-2\,\sqrt{1-\beta_2})\le k_0$ implies $\sqrt{(1-\beta_2^{k_0+1})\,\sum_{j=0}^{k_0}\beta_2^j}-2\ge0$. Since $\mu\le\tfrac{\sqrt{\beta_2}\,G}{\eta}$, we have that 
\begin{align}
    \mu\,\eta&\le \tfrac{\sqrt{\beta_2}\,G\,\sqrt{(1-\beta_2^{k_0+1})\,\sum_{j=0}^{k_0}\beta_2^j}}{\sqrt{(1-\beta_2^{k_0+1})\,\sum_{j=0}^{k_0}\beta_2^j}-2} = \tfrac{\sqrt{\beta_2}\,G\,\sqrt{(1-\beta_2)\,(\sum_{j=0}^{k_0}\beta_2^j)^2}}{\sqrt{(1-\beta_2)\,(\sum_{j=0}^{k_0}\beta_2^j)^2}-2} \le \tfrac{\sqrt{\beta_2}\,\sum_{j=0}^{k}\beta_2^j(\sqrt{1-\beta_2}\,G+\sqrt{\varepsilon})}{\sqrt{1-\beta_2}\,\sum_{j=0}^{k}\beta_2^j-2},
\end{align}
The first equality is derived from the property $1-\beta_2^{k_0+1}=(1-\beta_2)\sum_{j=0}^{k_0}\beta_2^j$. In view of $2\,\mu \,C_E(k_0)=2\,\mu\,\eta\,\Big(\tfrac{1}{2}-\tfrac{1}{\sqrt{\sum_{j=0}^{k}\beta_2^j}\,\sqrt{1-\beta_2^{k+1}}}\Big)\,\tfrac{1}{\sqrt{\beta_2}\,(G+\sqrt{\tfrac{\varepsilon}{1-\beta_2}})}$, we have that 
\begin{align}
2\,\mu \,C_E(k_0) &\le 2\,\Big(\tfrac{1}{2}-\tfrac{1}{\sqrt{\sum_{j=0}^{k}\beta_2^j}\,\sqrt{1-\beta_2^{k+1}}}\Big)\,\tfrac{1}{\sqrt{\beta_2}\,(G+\sqrt{\tfrac{\varepsilon}{1-\beta_2}})}\,\tfrac{\sqrt{\beta_2}\,\sum_{j=0}^{k}\beta_2^j(\sqrt{1-\beta_2}\,G+\sqrt{\varepsilon})}{\sqrt{1-\beta_2}\,\sum_{j=0}^{k}\beta_2^j-2}\nonumber\\
&= 2\,\Big(\tfrac{1}{2}-\tfrac{1}{\sqrt{\sum_{j=0}^{k}\beta_2^j}\,\sqrt{1-\beta_2^{k+1}}}\Big)\,\tfrac{\sqrt{1-\beta_2}}{(\sqrt{1-\beta_2}\,G+\sqrt{\varepsilon})}\,\tfrac{\sum_{j=0}^{k}\beta_2^j(\sqrt{1-\beta_2}\,G+\sqrt{\varepsilon})}{\sqrt{1-\beta_2}\,\sum_{j=0}^{k}\beta_2^j-2}\nonumber\\
&= 2\,\Big(\tfrac{1}{2}-\tfrac{1}{\sqrt{\sum_{j=0}^{k}\beta_2^j}\,\sqrt{1-\beta_2^{k+1}}}\Big)\,\tfrac{\sqrt{1-\beta_2}\sum_{j=0}^{k}\beta_2^j}{\sqrt{1-\beta_2}\,\sum_{j=0}^{k}\beta_2^j-2}\nonumber\\
&= \tfrac{\sum_{j=0}^{k}\beta_2^j\,\sqrt{1-\beta_2}-2}{\sum_{j=0}^{k}\beta_2^j\,\sqrt{1-\beta_2}}\cdot\tfrac{\sqrt{1-\beta_2}\sum_{j=0}^{k}\beta_2^j}{\sqrt{1-\beta_2}\,\sum_{j=0}^{k}\beta_2^j-2}=1.\nonumber
\end{align}
Again, the second equality from the last is derived from the property $1-\beta_2^{k_0+1}=(1-\beta_2)\sum_{j=0}^{k_0}\beta_2^j$,
 which in turn indicates that $1-2\,\mu\,C_{E}(k_0)\ge0$. Expanding the recursion from $k_0$ to $k$, we have 
\begin{align}\label{eq:condi1_eadam_pl}
 \expt\,[\,f(\bx^{(k)})-f^*\,]&\leq (1-2\,\mu\,C_{E}(k_0))^{k-k_0}\,\expt\,[\,f(\bx^{(k_0)})-f^*\,]\nonumber\\
 &\qquad+(C_{E,1}(k_0)\,\sigma + C_{E,2}(k_0)\,\sigma^2)\sum_{j=0}^{k-k_0} (1-2\,\mu\,C_{E}(k_0))^j\,\nonumber\\ 
 &\leq (1-2\,\mu\,C_{E}(k_0))^{k-k_0}\,\expt\,[\,f(\bx^{(k_0)})-f^*\,]\nonumber\\ &\qquad+\tfrac{C_{E,1}(k_0)}{2\,\mu\,C_{E}(k_0)}\,\sigma+\tfrac{C_{E,2}(k_0)}{2\,\mu\,C_{E}(k_0)}\,\sigma^2.
\end{align}

\noindent\fbox{(ii)} In view of the property $\sqrt{\tfrac{\beta_2}{1-\beta_2^{k+1}}\,y_i^{(k)}}\le \sqrt{\beta_2}\,(G+\sqrt{\tfrac{\varepsilon}{1-\beta_2}})$ and \cref{eq:case2_eadam}, we get the upper bound of $\expt\,[\,f(\bx^{(k+1)}) \ \big\vert \ \bx^{(k)}\,]-f(\bx^{(k)})$:
\begin{align}
   &-\underbrace{\eta\,(\tfrac{1}{2}-\tfrac{\sqrt{\varepsilon}}{\sqrt{v_{\min}}\,\sqrt{1-\beta_2^{k+1}}})\tfrac{1}{\sqrt{\beta_2}\,\big(G+\sqrt{\tfrac{\varepsilon}{1-\beta_2}}\big)}}_{C_E(k)}\,\Vert\nabla f(\bx^{(k)})\Vert^2 \nonumber\\
   &\qquad+\underbrace{\tfrac{\eta\,\sqrt{\varepsilon}\,G\,\sqrt{n}}{\sqrt{\beta_2}\,v_{\min}}}_{C_{E,1}}\,\sigma+\underbrace{\tfrac{\eta}{v_{\min}}\Big(\tfrac{L\,\eta}{2}+\tfrac{\sqrt{1-\beta_2}\,G}{\sqrt{\beta_2}}\Big)}_{C_{E,2}}\,\sigma^2\nonumber.
    \end{align}
Additionally, the inequalities $v_{\min}\ge 4\,\varepsilon$ and $k_0\ge \log_{\beta_2}(1-\tfrac{4\,\varepsilon}{v_{\min}})-1$ imply that $\sqrt{1-\beta_2^{k_0+1}}-\tfrac{2\sqrt{\varepsilon}}{\sqrt{v_{\min}}+\varepsilon}\ge0$. Finally,  $\mu\le \tfrac{\sqrt{\beta_2}\,(G+\tfrac{\varepsilon}{1-\beta_2})}{\eta}$ implies that $$\mu\,\eta\le \tfrac{\sqrt{\beta_2}\,(G+\tfrac{\varepsilon}{1-\beta_2})\,\sqrt{1-\beta_2^{k_0+1}}}{\sqrt{1-\beta_2^{k_0+1}}-\tfrac{2\sqrt{\varepsilon}}{\sqrt{v_{\min}}+\varepsilon}},$$ which further indicates $1-2\mu\,C_{E}(k_0)\ge0$. Proceeding analogously as for getting \cref{eq:condi1_eadam_pl},  we get the claim.
\end{proof}
\cref{prop:eadam_convergencerate} estimates the convergence rate of EAdam in two cases, identified by the values of $\varepsilon$ and $\bg^{(k)}$, similar to the ones of \cref{prop:eadam_monotonicity}.
 When the term $\varepsilon\,\sum_{j=0}^{k-1}\beta_2^j$ enforces a stricter lower bound on the second moment estimate, the EAdam-type optimizer adjusts its current iterations to follow the convergence rate observed in the first scenario in \cref{prop:eadam_convergencerate}. For both cases, decreasing the value of $\eta$ results in smaller coefficients in front of the $\sigma$-dependent terms and a smaller value of $C_{E}$; this suggests a higher accuracy on the final approximation of the optimal value at the price of a slower convergence rate. When $\varepsilon$ dominates the second raw moment estimate \cref{eq:2mon_eadam}, a larger value of $\varepsilon$ may result in smaller constants in front of both $\sigma $ and $\sigma^2$. In the second case, a larger value of $\varepsilon$ may result in a larger constant in front of $\sigma$. Therefore, monotonically increasing or decreasing the value of $\varepsilon$ will not cause a monotonic convergence behavior of the EAdam-type optimizer. Finally, note that in the first case of \cref{prop:eadam_convergencerate}, when increasing the number of iteration steps, the constants in front of $\sigma$ and $\sigma^2$ converge to $0$. 

\subsection{Convergence results for AdaBelief}
 As already observed,  by assuming $\beta_1=0$, we can only observe one convergence scenario of the AdaBelief-type optimizer that is similar to the first case in \cref{prop:eadam_monotonicity,prop:eadam_convergencerate}. Let us commence with the analysis of monotonicity.
\begin{proposition}\label{prop:AdaBelief_monotonicity}
Let the objective function satisfy \cref{assump:assumption} and the starting point $\mathbf x^{(0)}\in\mathbb R^{n}$ be chosen such that $g_i^{(0)}\neq0~\forall i$. Assume that $\beta_1=0$ and $k\in\mathbb N$ iteration steps of the AdaBelief optimizer (\cref{ag:adam} with \cref{eq:2mon_AdaBelief}) have been executed with input parameters $\beta_2\in(0,1)$, $\eta,\varepsilon\in\mathbb R^+$, such that $\eta\leq \tfrac{\sqrt{\varepsilon\,\sum_{j=0}^{k}\beta_2^j}}{2\,L}$. If $\sum_{j=0}^k\beta_2^j\ge\tfrac{4}{3\,\sqrt{1-\beta_2}}$, then $$\expt\,[\,f(\bx^{(k)})\,]\leq 
     \expt\,[\,f(\bx^{(k-1)})\,]+\,C_{B,1}\,\sigma+C_{B,2}\,\sigma^2,$$
     with
\begin{align}\label{eq:c_b12}
       C_{B,1}=\tfrac{\sqrt{n}\,\eta\,G}{\sqrt{\beta_2}\,\sum_{j=0}^{k-1}\beta_2^j},\qquad C_{B,2}=\tfrac{L\,\eta^2}{2\,\varepsilon\,\sum_{j=0}^{k-1}\beta_2^j}.   
\end{align}
\end{proposition}
\begin{proof}
By making the assumption $\mathbf{m}^{(k)}=\bg^{(k)}$, we have that ${\mathbf{s}}^{(k+1)}=\varepsilon\,\sum_{j=0}^{k} \beta_2^{j}= \beta_2 \mathbf{s}^{(k)}+\varepsilon$ and $\wb{\mathbf{s}}^{(k+1)}=\varepsilon\,\tfrac{\sum_{j=0}^{k} \beta_2^{k-j}}{1-\beta_2^{k+1}}$. Following the same steps that we used to get \cref{eq:ineq_eadam} yields 
\begin{align}
    &\tfrac{\bg^{(k)}}{\sqrt{\tfrac{1-\beta_2^k}{1-\beta_2^{k+1}}\,\beta_2\,\wb{\mathbf{s}}^{(k)}}+\varepsilon}-\tfrac{\bg^{(k)}}{\sqrt{\wb{\mathbf{s}}^{(k+1)}}+\varepsilon}\nonumber\\
    &\le \tfrac{\vert \bg^{(k)}\vert}{(\sqrt{\wb{\mathbf{s}}^{(k+1)}}+\varepsilon) \Big(\sqrt{\tfrac{\beta_2}{1-\beta_2^{k+1}}\,\mathbf{s}^{(k)}}+\varepsilon\Big)\sqrt{1-\beta_2^{k+1}}}\cdot\tfrac{\varepsilon}{\sqrt{\beta_2\,\mathbf{s}^{(k)}+\varepsilon}+\sqrt{\beta_2\,\mathbf{s}^{(k)}}},\nonumber\\
    &\leq \tfrac{\sqrt{\varepsilon}\,\vert \bg^{(k)}\vert}{(\sqrt{\wb{\mathbf{s}}^{(k+1)}}+\varepsilon) \Big(\sqrt{\tfrac{\beta_2}{1-\beta_2^{k+1}}\,\mathbf{s}^{(k)}}+\varepsilon\Big)\sqrt{1-\beta_2^{k+1}}}\nonumber.
\end{align}
Substituting the latter relation into \cref{eq:general_fk-f*}, we obtain
\begin{align}
\expt\,[&\,f(\bx^{(k+1)})\ \big\vert \ \bx^{(k)}\,]-f(\bx^{(k)})\nonumber\\
    &\le-\eta\,\bigg\Vert\tfrac{\nabla f(\bx^{(k)})}{\big(\sqrt{\tfrac{\beta_2\,\mathbf{s}^{(k)}}{1-\beta_2^{k+1}}}+\varepsilon\big)^{1/2}}\bigg\Vert^2 +\tfrac{L\eta^2}{2}\,\expt\,\Big[\,\Big\Vert\tfrac{\bg^{(k)}}{\sqrt{\wb{\mathbf{s}}^{(k+1)}}+\varepsilon}\Big\Vert^{2}\ \big\vert \ \bx^{(k)}\,\Big]\nonumber\\
    &\qquad+\eta\,\big\vert\nabla f(\bx^{(k)}) \big\vert^T\,\expt\,\bigg[\,\tfrac{\sqrt{\varepsilon}\,\vert \bg^{(k)}\vert}{(\sqrt{\wb{\mathbf{s}}^{(k+1)}}+\varepsilon) \Big(\sqrt{\tfrac{\beta_2}{1-\beta_2^{k+1}}\,\mathbf{s}^{(k)}}+\varepsilon\Big)\sqrt{1-\beta_2^{k+1}}}\,\bigg]\nonumber.
    \end{align}
The property $\wb{s}_i^{(k+1)}\ge s_i^{(k+1)} =\varepsilon\,\sum_{j=0}^{k}\beta_2^j$ yields 
\begin{align}
    \expt\,&[\,f(\bx^{(k+1)}) \ \big\vert \ \bx^{(k)}\,]-f(\bx^{(k)})\nonumber\\
    &\le -\eta\,\bigg\Vert\tfrac{\nabla f(\bx^{(k)})}{\big(\sqrt{\tfrac{\beta_2\,\mathbf{s}^{(k)}}{1-\beta_2^{k+1}}}+\varepsilon\big)^{1/2}}\bigg\Vert^2 +\tfrac{L\eta^2}{2\,\sqrt{\varepsilon\,\sum_{j=0}^{k}\beta_2^j}}\,\expt\,\Big[\,\Big\Vert\tfrac{\bg^{(k)}}{(\sqrt{\wb{\mathbf{s}}^{(k+1)}}+\varepsilon)^{1/2}}\Big\Vert^{2}\ \big\vert \ \bx^{(k)}\,\Big]\nonumber\\
    &\qquad+\tfrac{\eta}{\sqrt{\varepsilon\,\sum_{j=0}^{k}\beta_2^j}\,\sqrt{1-\beta_2^{k+1}}}\,\big\vert\nabla f(\bx^{(k)}) \big\vert^T\,\expt\,\bigg[\,\tfrac{\sqrt{\varepsilon}\,\vert \bg^{(k)}\vert}{ \sqrt{\tfrac{\beta_2}{1-\beta_2^{k+1}}\,\mathbf{s}^{(k)}}+\varepsilon}\,\bigg]\nonumber.
\end{align}
Based on \cref{eq:boundvertg}, we have that $\expt\,[\,f(\bx^{(k+1)}) \ \big\vert \ \bx^{(k)}\,]-f(\bx^{(k)})$ is bounded from above by
\begin{align} 
    &-\eta\,\bigg\Vert\tfrac{\nabla f(\bx^{(k)})}{\big(\sqrt{\tfrac{\beta_2\,\mathbf{s}^{(k)}}{1-\beta_2^{k+1}}}+\varepsilon\big)^{1/2}}\bigg\Vert^2 +\tfrac{L\eta^2}{2\,\sqrt{\varepsilon\,\sum_{j=0}^{k}\beta_2^j}}\,\expt\,\Big[\,\Big\Vert\tfrac{\bg^{(k)}}{(\sqrt{\wb{\mathbf{s}}^{(k+1)}}+\varepsilon)^{1/2}}\Big\Vert^{2}\ \big\vert \ \bx^{(k)}\,\Big]\nonumber\\
    &\qquad+\tfrac{\eta}{\sqrt{\sum_{j=0}^{k}\beta_2^j}\,\sqrt{1-\beta_2^{k+1}}}\,\bigg\Vert\tfrac{\nabla f(\bx^{(k)})}{\big(\sqrt{\tfrac{\beta_2\,\mathbf{s}^{(k)}}{1-\beta_2^{k+1}}}+\varepsilon\big)^{1/2}}\bigg\Vert^2 +\tfrac{\eta\,G}{\sqrt{\sum_{j=0}^{k}\beta_2^j}\,\sqrt{1-\beta_2^{k+1}}}\,\sum_{i=1}^{d}\tfrac{\sigma_i}{ \sqrt{\tfrac{\beta_2}{1-\beta_2^{k+1}}\,s_i^{(k)}}+\varepsilon}\nonumber\\
    &=-\eta\,\big( \tfrac{3}{4}-\tfrac{1}{\sqrt{1-\beta_2}\sum_{j=0}^{k}\beta_2^j}\big)\bigg\Vert\tfrac{\nabla f(\bx^{(k)})}{\big(\sqrt{\tfrac{\beta_2\,\mathbf{s}^{(k)}}{1-\beta_2^{k+1}}}+\varepsilon\big)^{1/2}}\bigg\Vert^2 +\tfrac{\eta\,G\,\sqrt{n}}{\sqrt{\beta_2}\,\sum_{j=0}^{k-1}\beta_2^j}\,\sigma + \tfrac{L\eta^2}{2\,\varepsilon\,\sum_{j=0}^{k-1}\beta_2^j} \,\sigma^2 \nonumber.
\end{align}
The last inequality is derived from the properties $\eta\leq\tfrac{\sqrt{\varepsilon\,\sum_{j=0}^{k}\beta_2^j}}{2\,L}$ and $\sum_{i=1}^n \sigma_i\leq \sqrt{n}\,\sigma$. Finally, since $\sum_{j=0}^{k}\beta_2^j\ge \tfrac{4}{3\sqrt{1-\beta_2}}$ we get the claim.
\end{proof}
Comparing \cref{prop:eadam_monotonicity} and \cref{prop:AdaBelief_monotonicity}, we observe that \cref{prop:AdaBelief_monotonicity} necessitates the smaller lower bound $\sum_{j=0}^k\beta_2^j\ge\tfrac{4}{3\,\sqrt{1-\beta_2}}$ compared to $\sum_{j=0}^k\beta_2^j\ge\tfrac{2}{\sqrt{1-\beta_2}}$ in \cref{prop:eadam_monotonicity}. For instance, when $\beta_2=0.999$, $k$ should be chosen such that $k>43$ in instead of $k>65$. This implies that the AdaBelief optimizer has the potential to achieve monotonicity slightly earlier than the EAdam optimizer.  Proceeding analogously to \cref{prop:eadam_convergencerate}, we obtain a linear convergence rate of the AdaBelief optimizer. In the following proposition, we show that a larger bound on $\eta$ is allowed compared to \cref{prop:eadam_convergencerate}. 
\begin{proposition}\label{prop:AdaBelief_PL}
     Under the same assumptions of \cref{prop:AdaBelief_monotonicity},  if $\eta$ is chosen such that 
    $\eta\le \tfrac{4\,\sqrt{\beta_2}\,G}{3\,\mu}$ and $k_0\in\mathbb N$ is such that $k_0>\log_{\beta_2}(1-\frac{4\,\sqrt{1-\beta_2}}{3})$, then for any $k>k_0$, the random vector $\bx^{(k)}$ satisfies 
     \begin{align}\label{eq:linearconvergence_Cab}
 \expt\,[\,f(\bx^{(k)})-f^*\,]\leq
     (1-2\,\mu\,C_{B})^{k-k_0}\,\expt\,[\,f(\bx^{(k_0)})-f^*\,]+\tfrac{C_{B,1}}{2\,\mu\,C_{B}}\,\,\sigma+\tfrac{C_{B,2}}{2\,\mu\,C_{B}}\,\sigma^2,
\end{align}
where $0<2\,\mu\,C_{B}<1$, $C_{B,1}$ and $C_{B,2}$ satisfy \cref{eq:c_b12} by replacing $k$ with $k_0$ and 
\begin{align*}
C_B&=\eta\,\Big( \tfrac{3}{4}\!-\tfrac{1}{\sqrt{1-\beta_2}\sum_{j=0}^{k_0}\beta_2^j}\Big)\tfrac{1}{\sqrt{\beta_2}\,\big(2\,G+\sqrt{\tfrac{\varepsilon}{1-\beta_2}}\big)}.
    \end{align*}
\end{proposition}
\begin{proof}
On the basis of the inequality $\sqrt{\tfrac{\beta_2}{1-\beta_2^{k+1}}\,s_i^{(k)}}\leq \sqrt{\beta_2}\,\big(2\,G+\sqrt{\tfrac{\varepsilon}{1-\beta_2}}\big)$, we get that $\expt\,[\,f(\bx^{(k+1)})\, \ \big\vert \ \bx^{(k)}]- f(\bx^{(k)})$ is bounded from above by
\begin{align}
-\underbrace{\eta\,\Big( \tfrac{3}{4}\!-\tfrac{1}{\sqrt{1-\beta_2}\sum_{j=0}^{k}\beta_2^j}\Big)\tfrac{1}{\sqrt{\beta_2}\,\Big(2\,G+\sqrt{\tfrac{\varepsilon}{1-\beta_2}}\Big)}}_{C_{B}(k)}\Vert\nabla f(\bx^{(k)})\Vert^2\!+\!\underbrace{\tfrac{\eta\,G\,\sqrt{n}}{\sqrt{\beta_2}\,\sum_{j=0}^{k-1}\beta_2^j}}_{C_{B,1}(k)}\,\sigma+\! \underbrace{\tfrac{L\eta^2}{2\,\varepsilon\,\sum_{j=0}^{k-1}\beta_2^j}}_{C_{B,2}(k)} \,\sigma^2 \!\nonumber.
\end{align}
The property $k_0\ge \log_{\beta_2}\big(1-\frac{4\,\sqrt{1-\beta_2}}{3}\big)$ implies that ${C_{B}(k)}\ge0$ for $k\ge k_0$. Proceeding analogously to the derivation of \cref{eq:condi1_eadam_pl} and taking the expectation from both sides, we get that $\expt\,[\,f(\bx^{(k+1)})\,]-f^*$ is bounded from above by
\begin{align}
    &(1-2\,\mu\,C_{B}(k_0))\,(\expt\,[\,f(\bx^{(k)})\,]-f^*)+C_{B,1}(k_0)\, \sigma+C_{B,2}(k_0)\,\sigma^2.\nonumber
\end{align}
The inequalities $k_0\ge \log_{\beta_2}(1-\frac{4\,\sqrt{1-\beta_2}}{3})$ and $\mu\le \tfrac{4\,\sqrt{\beta_2}\,G}{3\,\eta}$ imply  $2\,\mu\,C_{B}(k_0)\le 1$. Therefore, expanding the recursion from the $k_0$th to the $k$th iteration, we obtain \cref{eq:linearconvergence_Cab}.
\end{proof}
\cref{prop:AdaBelief_PL} shows that decreasing the value of $\eta$ has the potential to increase the accuracy of the optimal value by paying the price of a slower convergence rate. Additionally, comparing \cref{prop:eadam_convergencerate} and \cref{prop:AdaBelief_PL}, we have that $C_E(k)\le \frac{1}{2}\tfrac{1}{\sqrt{\beta_2}\,(G+\sqrt{\tfrac{\varepsilon}{1-\beta_2}})}$ and $C_{B}(k)\le \frac{3}{4}\tfrac{1}{\sqrt{\beta_2}\,(2\,G+\sqrt{\tfrac{\varepsilon}{1-\beta_2}})}$. Therefore, when $G\ge\sqrt{\tfrac{\varepsilon}{1-\beta_2}}$, we expect AdaBelief to exhibit slower convergence than EAdam, when the same parameter configuration is employed. 

\subsection{Convergence results for AdamL}
Next, we study the convergence behavior of the AdamL optimizer. To facilitate the convergence analysis, let us denote by
$$\ell_{\min}:=\min_{j=1,\dots,k} \ell^{(j)}\quad \text{and}\quad \ell_{\max}:=\max_{j=1,\dots,k} \ell^{(j)}$$ 
the minimum and maximum of the loss function, encountered during the first $k$ iteration steps, respectively.
\begin{proposition}\label{prop:ladam_monotonicity}Let the objective function satisfy \cref{assump:assumption} and the starting point $\mathbf x^{(0)}\in\mathbb R^{n}$ be chosen such that $g_i^{(0)}\neq0~\forall i$. Assume that $k\in\mathbb N$ iteration steps of the AdamL-type optimizer (\cref{ag:adam} with \cref{eq:2mon_ladam}) have been executed with input parameters $\beta_2\in(0,1)$, $\eta, \varepsilon, \gamma\in\mathbb R^+$, such that one of the following two conditions is satisfied
\begin{itemize}
\item[$(i)$] $\ell_{\min}\ge1$,
 $\sum_{j=0}^k\beta_2^j\ge\max\{\tfrac{4\,G\,}{\sqrt{\gamma\,\varepsilon}}, \tfrac{2}{\sqrt{1-\beta_2}}\}$, and $\eta\leq \frac{\sqrt{\varepsilon\,\sum_{j=0}^{k}\beta_2^j}}{2\,L}$;
\item[$(ii)$]  $\ell_{\max}<1$, 
$v_{\min}> 4\,\varepsilon\,\gamma\,\ell^{(k)}$,  $\sum_{j=0}^k\beta_2^j\ge\max\{\tfrac{4\,G}{\sqrt{v_{\min}}}, \tfrac{4\,\varepsilon\,\gamma\,\ell^{(k)}}{(1-\beta_2)\,v_{\min}}\}$, and $\eta\le \tfrac{\sqrt{v_{\min}}}{2\,L\,\sqrt{\gamma}}$;
\end{itemize}
then, the random vector $\bx^{(k)}$ satisfies
$$\expt\,[\,f(\bx^{(k)})\,]\leq \expt\,[\,f(\bx^{(k-1)})\,]+\,C_{L,1} \,\sigma+C_{L,2}\,\sigma^2,$$
where in case $(i)$
\begin{align}\label{eq:cl12_case1}
    C_{L,1}=\sqrt{\tfrac{n}{\beta_2\,\varepsilon}}\cdot\tfrac{\eta\,G}{\sum_{j=0}^{k-1}\beta_2^{j}},\qquad C_{L,2}=\tfrac{\eta}{\varepsilon\,\sum_{j=0}^{k-1}\beta_2^{j}}\Big(\tfrac{L\,\eta}{2}+\tfrac{G}{\sqrt{\gamma}}\,\sqrt{\tfrac{1-\beta_2}{\beta_2}}\Big),
\end{align}
and in case $(ii)$
\begin{align}\label{eq:cl12_case2}
C_{L,1}=\sqrt{\tfrac{n\,\varepsilon\,\gamma\,\ell^{(k)}}{\beta_2}}\cdot\tfrac{\eta\,G}{v_{\min}},\qquad C_{L,2}=\tfrac{\eta}{v_{\min}}\Big(\tfrac{L\,\eta\,\gamma}{2}+G\,\sqrt{\tfrac{1-\beta_2}{\beta_2}}\Big).
\end{align}
\end{proposition}
\begin{proof}
By replacing $y_i^{(k)}$ in \cref{eq:ineq_eadam} with $w_i^{(k)}$, we get\begin{align}
    &\tfrac{\bg^{(k)}}{\sqrt{\tfrac{1-\beta_2^k}{1-\beta_2^{k+1}}\,\beta_2\,\wb{\bw}^{(k)}}+\varepsilon}-\tfrac{\bg^{(k)}}{\sqrt{\wb{\bw}^{(k+1)}}+\varepsilon}\nonumber\\
    &\quad\le \tfrac{\vert \bg^{(k)}\vert}{(\sqrt{\wb{\bw}^{(k+1)}}+\varepsilon) \Big(\sqrt{\tfrac{\beta_2}{1-\beta_2^{k+1}}\,\bw{(k)}}+\varepsilon\Big)\sqrt{1-\beta_2^{k+1}}}\cdot\tfrac{\tfrac{1-\beta_2}{\gamma\,(\ell^{(k)})^{\varphi}}\,(\bg^{(k)})^{2}+\varepsilon\,\ell^{(k)}}{\sqrt{\bw^{(k+1)}}+\sqrt{\beta_2\,\bw^{(k)}}},\nonumber\\
    &\quad\le \tfrac{\vert \bg^{(k)}\vert}{(\sqrt{\wb{\bw}^{(k+1)}}+\varepsilon) \Big(\sqrt{\tfrac{\beta_2}{1-\beta_2^{k+1}}\,\bw^{(k)}}+\varepsilon\Big)\sqrt{1-\beta_2^{k+1}}} \Big(\sqrt{\tfrac{1-\beta_2}{\gamma\,(\ell^{(k)})^{\varphi}}}\,\vert \bg^{(k)}\vert+\sqrt{\varepsilon\,\ell^{(k)}}\Big)\nonumber\\
    &\quad \leq \tfrac{\sqrt{\tfrac{1-\beta_2}{\gamma\,(\ell^{(k)})^{\varphi}}}\,( \bg^{(k)})^{2}+\sqrt{\varepsilon\,\ell^{(k)}}\,\vert \bg^{(k)}\vert}{\big(\sqrt{\wb{\bw}^{(k+1)}}+\varepsilon\big) \left(\sqrt{\tfrac{\beta_2}{1-\beta_2^{k+1}}\,\bw^{(k)}}+\varepsilon\right)\sqrt{1-\beta_2^{k+1}}}.
    \end{align}    
    Substituting it into \cref{eq:general_fk-f*}, we achieve 
\begin{align}
   \expt\,&[\,f(\bx^{(k+1)})\ \big\vert \ \bx^{(k)}\,]-f(\bx^{(k)}) \nonumber\\
   &\le -\eta\,\bigg\Vert\tfrac{\nabla f(\bx^{(k)})}{\big(\sqrt{\tfrac{\beta_2\,\bw^{(k)}}{1-\beta_2^{k+1}}}+\varepsilon\big)^{1/2}}\bigg\Vert^2 +\tfrac{L\eta^2}{2}\,\expt\,\Big[\,\Big\Vert\tfrac{\bg^{(k)}}{\sqrt{\wb{\bw}^{(k+1)}}+\varepsilon}\Big\Vert^{2}\ \big\vert \ \bx^{(k)}\,\Big]\nonumber\\
    & \qquad+\eta\,\big\vert\nabla f(\bx^{(k)}) \big\vert^T\,\expt\,\bigg[\,\tfrac{\sqrt{\tfrac{1-\beta_2}{\gamma\,(\ell^{(k)})^{\varphi}}}\,( \bg^{(k)})^{2}+\sqrt{\varepsilon\,\ell^{(k)}}\,\vert \bg^{(k)}\vert}{\big(\sqrt{\wb{\bw}^{(k+1)}}+\varepsilon\big) \big(\sqrt{\tfrac{\beta_2\,\bw^{(k)}}{1-\beta_2^{k+1}}}+\varepsilon\big)\sqrt{1-\beta_2^{k+1}}}\,\bigg]\nonumber\\
    &\le -\eta\,\bigg\Vert\tfrac{\nabla f(\bx^{(k)})}{\big(\sqrt{\tfrac{\beta_2\,\bw^{(k)}}{1-\beta_2^{k+1}}}+\varepsilon\big)^{1/2}}\bigg\Vert^2 +\tfrac{L\eta^2}{2}\,\expt\,\bigg[\,\bigg\Vert\tfrac{\bg^{(k)}}{\big(\sqrt{\tfrac{\beta_2\,\bw^{(k)}}{1-\beta_2^{k+1}}}+\varepsilon\big)^{1/2}(\sqrt{\wb{\bw}^{(k+1)}}+\varepsilon)^{1/2}}\bigg\Vert^{2}\ \bigg\vert \ \bx^{(k)}\,\bigg]\nonumber\\
    &\qquad +\eta\,\big\vert\nabla f(\bx^{(k)}) \big\vert^T\,\expt\,\bigg[\,\tfrac{\sqrt{\tfrac{1-\beta_2}{\gamma\,(\ell^{(k)})^{\varphi}}}\,( \bg^{(k)})^{2}+\sqrt{\varepsilon\,\ell^{(k)}}\,\vert \bg^{(k)}\vert}{\big(\sqrt{\wb{\bw}^{(k+1)}}+\varepsilon\big) \Big(\sqrt{\tfrac{\beta_2\,\bw^{(k)}}{1-\beta_2^{k+1}}}+\varepsilon\Big)\sqrt{1-\beta_2^{k+1}}}\,\bigg].\label{eq:ladamfk+1-fk}
    \end{align}
\noindent \fbox{$(i)$} The property $\ell_{\min}\ge1$, gives $w_i^{k+1}\ge \varepsilon\sum_{j=0}^{k}\beta_2^{k-j}\,\ell^{(j)}\ge\varepsilon\sum_{j=0}^{k}\beta_2^{j}$, which in turn implies that $\expt\,[\,f(\bx^{(k+1)})\ \big\vert \ \bx^{(k)}\,]- f(\bx^{(k)})$ is bounded from above by
\begin{align}
 &-\eta\,\bigg\Vert\tfrac{\nabla f(\bx^{(k)})}{\big(\sqrt{\tfrac{\beta_2\,\bw^{(k)}}{1-\beta_2^{k+1}}}+\varepsilon\big)^{\tfrac{1}{2}}}\bigg\Vert^2 + \tfrac{L\eta^2}{2\,\sqrt{\varepsilon\sum_{j=0}^{k}\beta_2^{j}}}\,\Big(\,\bigg\Vert\tfrac{\nabla f(\bx^{(k)})}{\big(\sqrt{\tfrac{\beta_2\,\bw^{(k)}}{1-\beta_2^{k+1}}}+\varepsilon\big)^{1/2}}\bigg\Vert^2+\sum_{i=1}^{d}\tfrac{\sigma_i^2}{\sqrt{\,w^{(k+1)}_i}+\varepsilon}\Big)\nonumber\\
&\qquad +\tfrac{\eta\,\sqrt{1-\beta_2}\,G}{\sqrt{\varepsilon\,\gamma\,\sum_{j=0}^{k}\beta_2^{j}} \,\sqrt{1-\beta_2^{k+1}}}\Big(\,\bigg\Vert\tfrac{\nabla f(\bx^{(k)})}{\big(\sqrt{\tfrac{\beta_2\,\bw^{(k)}}{1-\beta_2^{k+1}}}+\varepsilon\big)^{1/2}}\bigg\Vert^2+\sum_{i=1}^{d}\tfrac{\sigma_i^2}{\sqrt{\tfrac{\beta_2}{1-\beta_2^{k+1}}\,w_i^{(k)}}+\varepsilon}\Big)\nonumber\\
&\qquad +\tfrac{\eta\,\sqrt{\ell^{(k)}}}{\sqrt{\sum_{j=0}^{k}\beta_2^{k-j}\,\ell^{(j)}}\,\sqrt{1-\beta_2^{k+1}}}\,\Big(\,\bigg\Vert\tfrac{\nabla f(\bx^{(k)})}{\big(\sqrt{\tfrac{\beta_2\,\bw^{(k)}}{1-\beta_2^{k+1}}}+\varepsilon\big)^{1/2}}\bigg\Vert^2+\sum_{i=1}^{d}\tfrac{G\,\sigma_i}{ \sqrt{\tfrac{\beta_2}{1-\beta_2^{k+1}}\,w_i^{(k)}}+\varepsilon}\Big)\nonumber\\
&\le  -\eta \,\Big(1-\tfrac{L\eta}{2\,\sqrt{\varepsilon\sum_{j=0}^{k}\beta_2^{j}}} - \tfrac{\sqrt{1-\beta_2}\,G}{\sqrt{\varepsilon\,\gamma\,\sum_{j=0}^{k}\beta_2^{j}} \,\sqrt{1-\beta_2^{k+1}}}\nonumber\\&\qquad-\tfrac{1}{\sqrt{\sum_{j=0}^{k}\,\beta_2^{j}}\,\sqrt{1-\beta_2^{k+1}}}\Big)\,\bigg\Vert\tfrac{\nabla f(\bx^{(k)})}{\big(\sqrt{\tfrac{\beta_2\,\bw^{(k)}}{1-\beta_2^{k+1}}}+\varepsilon\big)^{1/2}}\bigg\Vert^2\nonumber\\
&\qquad+\Big(\tfrac{L\eta^2}{2\,\varepsilon\,\sum_{j=0}^{k}\beta_2^{j}}+\tfrac{\eta\,\sqrt{1-\beta_2}\,G}{\sqrt{\beta_2\,\gamma}\,\varepsilon\,\sum_{j=0}^{k-1}\beta_2^{j}}\Big)\,\sigma^2 + \tfrac{\eta\,G\,\sqrt{n}}{\sqrt{\beta_2\,\varepsilon}\sum_{j=0}^{k-1}\beta_2^{j}}\,\sigma.\nonumber
\end{align}
On the basis of $\eta\le \tfrac{\sqrt{\varepsilon\,\sum_{j=0}^{k}\beta_2^j}}{2L}$ and $\sum_{j=0}^{k}\beta_2^j\ge\tfrac{4\,G}{\sqrt{\varepsilon\,\gamma}}$, we get 
\begin{align}
 \expt\,[\,f(\bx^{(k+1)})\ \big\vert \ \bx^{(k)}\,]- f(\bx^{(k)})
&\le -\eta \,\big(\tfrac{1}{2}-\tfrac{1}{\sqrt{\sum_{j=0}^{k}\,\beta_2^{j}}\,\sqrt{1-\beta_2^{k+1}}}\big)\,\bigg\Vert\tfrac{\nabla f(\bx^{(k)})}{\big(\sqrt{\tfrac{\beta_2\,\bw^{(k)}}{1-\beta_2^{k+1}}}+\varepsilon\big)^{1/2}}\bigg\Vert^2\nonumber\\
&\phantom{\le}\ +\tfrac{\eta\,G\,\sqrt{n}}{\sqrt{\beta_2\,\varepsilon}\sum_{j=0}^{k-1}\beta_2^{j}}\,\sigma+\tfrac{\eta}{\varepsilon\,\sum_{j=0}^{k-1}\beta_2^{j}}\Big(\tfrac{L\,\eta}{2}+\tfrac{\sqrt{1-\beta_2}\,G}{\sqrt{\beta_2\,\gamma}}\Big)\,\sigma^2.\label{eq:lminle1_adaml}
\end{align}
Finally, the properties $\sum_{j=0}^k\beta_2^j\ge\tfrac{2}{\sqrt{1-\beta_2}}$ and $1-\beta_2^{k+1}=(1-\beta_2)\sum_{j=0}^{k}\beta_2^j$ imply that $\tfrac{1}{2}\ge \tfrac{1}{\sqrt{\sum_{j=0}^{k}\,\beta_2^{j}}\,\sqrt{1-\beta_2^{k+1}}}$, which gives the claim.

\noindent \fbox{$(ii)$} The property $\ell_{\max}< 1$ implies that $\wb{w}_i^{(k+1)}\ge\tfrac{(1-\beta_2)\,\sum_{j=0}^{k}(g_i^{(j)})^2\,\beta_2^{k-j}}{\gamma\,(\ell^{(k)})^{\varphi}}\ge \tfrac{v_{\min}}{\gamma\,(\ell^{(k)})^{\varphi}}\ge \tfrac{v_{\min}}{\gamma}$. Substituting it into \cref{eq:ladamfk+1-fk}, we get 
$\expt\,[\,f(\bx^{(k+1)})\ \big\vert \ \bx^{(k)}\,]- f(\bx^{(k)})$ is bounded from above by
\begin{align}
    &-\eta\,\bigg\Vert\tfrac{\nabla f(\bx^{(k)})}{\big(\sqrt{\tfrac{\beta_2\,\bw^{(k)}}{1-\beta_2^{k+1}}}+\varepsilon\big)^{1/2}}\bigg\Vert^2 +\tfrac{L\eta^2}{2}\,\expt\,\bigg[\,\bigg\Vert\tfrac{\bg^{(k)}}{\big(\sqrt{\wb{\bw}^{(k+1)}}+\varepsilon\big)^{1/2}\big(\sqrt{\tfrac{v_{\min}}{\gamma}}+\varepsilon\big)^{1/2}}\bigg\Vert^{2}\ \bigg\vert \ \bx^{(k)}\,\bigg]\nonumber\\
    & \qquad+\eta\,\big\vert\nabla f(\bx^{(k)}) \big\vert^T\,\expt\,\bigg[\,\tfrac{\sqrt{1-\beta_2}\,( \bg^{(k)})^{2}}{\sqrt{v_{\min}} \big(\sqrt{\tfrac{\beta_2\,\bw^{(k)}}{1-\beta_2^{k+1}}}+\varepsilon\big)\sqrt{1-\beta_2^{k+1}}}\ \bigg\vert \ \bx^{(k)}\,\bigg]\nonumber\\
    & \qquad+\eta\,\big\vert\nabla f(\bx^{(k)}) \big\vert^T\,\expt\,\bigg[\,\tfrac{\sqrt{\varepsilon\,\ell^{(k)}}\,\vert \bg^{(k)}\vert}{\sqrt{\tfrac{v_{\min}}{\gamma}} \big(\sqrt{\tfrac{\beta_2\,\bw^{(k)}}{1-\beta_2^{k+1}}}+\varepsilon\big)\sqrt{1-\beta_2^{k+1}}}\ \bigg\vert \ \bx^{(k)}\,\bigg]\nonumber\\
    &\le -\eta\,\bigg\Vert\tfrac{\nabla f(\bx^{(k)})}{\big(\sqrt{\tfrac{\beta_2\,\bw^{(k)}}{1-\beta_2^{k+1}}}+\varepsilon\big)^{1/2}}\bigg\Vert^2 +\tfrac{L\eta^2}{2\,\sqrt{\tfrac{v_{\min}}{\gamma}}}
    \bigg(\,\bigg\Vert\tfrac{\nabla f(\bx^{(k)})}{\big(\sqrt{\tfrac{\beta_2\,\bw^{(k)}}{1-\beta_2^{k+1}}}+\varepsilon\big)^{1/2}}\bigg\Vert^2+\sum_{i=1}^{d}\tfrac{\sigma_i^2}{\sqrt{w^{(k+1)}_i}+\varepsilon}\bigg)\nonumber\\
    & \qquad+\tfrac{\eta\,\sqrt{1-\beta_2}\,G}{\sqrt{v_{\min}\,(1-\beta_2^{k+1})}}\,\bigg(\,\bigg\Vert\tfrac{\nabla f(\bx^{(k)})}{(\sqrt{\tfrac{\beta_2\,\bw^{(k)}}{1-\beta_2^{k+1}}}+\varepsilon)^{ \tfrac12}}\bigg\Vert^2+\sum_{i=1}^{d}\tfrac{\sigma_i^2}{\sqrt{\tfrac{\beta_2}{1-\beta_2^{k+1}}\,w_i^{(k)}}+\varepsilon}\bigg)\nonumber\\
    &\qquad +\tfrac{\eta\,\sqrt{\varepsilon\,\gamma\,\ell^{(k)}}}{\sqrt{v_{\min}\,(1-\beta_2^{k+1})}}\,\bigg(\,\bigg\Vert\tfrac{\nabla f(\bx^{(k)})}{\big(\sqrt{\tfrac{\beta_2\,\bw^{(k)}}{1-\beta_2^{k+1}}}+\varepsilon\big)^{1/2}}\bigg\Vert^2+\sum_{i=1}^{d}\tfrac{G\,\sigma_i}{ \sqrt{\tfrac{\beta_2}{1-\beta_2^{k+1}}\,w_i^{(k)}}+\varepsilon}\bigg)\nonumber
    \end{align}
    \vspace{-2mm}
    \begin{align}
    &=-\eta\,\Big(1-\tfrac{L\,\eta\,\sqrt{\gamma}}{2\,\sqrt{v_{\min}}}-\tfrac{\sqrt{1-\beta_2}\,G}{\sqrt{v_{\min}\,(1-\beta_2^{k+1})}}-\tfrac{\sqrt{\varepsilon\,\gamma\,\ell^{(k)}}}{\sqrt{v_{\min}\,(1-\beta_2^{k+1})}}\Big)\bigg\Vert\tfrac{\nabla f(\bx^{(k)})}{\big(\sqrt{\tfrac{\beta_2\,\bw^{(k)}}{1-\beta_2^{k+1}}}+\varepsilon\big)^{1/2}}\bigg\Vert^2 \nonumber\\
    &\qquad+\tfrac{\eta\,G\,\sqrt{n\,\varepsilon\,\gamma\,\ell^{(k)}}}{\sqrt{\beta_2}\,v_{\min}}\,\sigma+\big(\tfrac{L\,\eta^2\,\gamma}{2\,v_{\min}}+\tfrac{\eta\,\sqrt{1-\beta_2}\,G}{\sqrt{\beta_2}\,v_{\min}}\big)\,\sigma^2\nonumber.
    \end{align}
The properties $\eta\le \tfrac{\sqrt{v_{\min}}}{2\,L\,\sqrt{\gamma}}$ and $\sum_{j=0}^{k} \beta_2^j\ge\tfrac{4\,G}{\sqrt{v_{\min}}}$ yield 
\begin{align}
\expt\,[\,f(\bx^{(k+1)})\ \big\vert \ \bx^{(k)}\,]- f(\bx^{(k)})
    &\le -\eta\,\Big(\tfrac{1}{2}-\tfrac{\sqrt{\varepsilon\,\gamma\,\ell^{(k)}}}{\sqrt{v_{\min}\,(1-\beta_2^{k+1})}}\Big)\bigg\Vert\tfrac{\nabla f(\bx^{(k)})}{\big(\sqrt{\tfrac{\beta_2\,\bw^{(k)}}{1-\beta_2^{k+1}}}+\varepsilon\big)^{1/2}}\bigg\Vert^2 \nonumber\\
    &\phantom{\le}\ +\tfrac{\eta\,G\,\sqrt{n\,\varepsilon\,\gamma\,\ell^{(k)}}}{\sqrt{\beta_2}\,v_{\min}}\,\sigma+\big(\tfrac{L\,\eta^2\,\gamma}{2\,v_{\min}}+\tfrac{\eta\,\sqrt{1-\beta_2}\,G}{\sqrt{\beta_2}\,v_{\min}}\big)\,\sigma^2.\label{eq:lminle2_adaml}
    \end{align}
The claim follows by noting that $v_{\min}\ge 4\,\varepsilon\,\gamma\,\ell^{(k)}$ and $\sum_{j=0}^k\beta_2^j\ge\tfrac{4\,\varepsilon\,\gamma\,\ell^{(k)}}{(1-\beta_2)\,v_{\min}}$ ensure the non-positivity of the first addend.
\end{proof}

\cref{prop:ladam_monotonicity} suggests that the monotonicity of the AdamL optimizer can be guaranteed under similar conditions to those in \cref{prop:eadam_monotonicity}. Comparing the conditions of case (i) $\sum_{j=0}^k\beta_2^j\ge\tfrac{4\,G}{\sqrt{\varepsilon}}$ and $\sum_{j=0}^k\beta_2^j\ge\tfrac{4\,G}{\sqrt{\gamma\,\varepsilon}}$ in \cref{prop:eadam_monotonicity,prop:ladam_monotonicity}, respectively, a larger value of $\gamma$ in AdamL may force a smaller lower bound of $k$ than the one corresponding to the EAdam optimizer. This suggests that in the AdamL optimizer, fewer number of iterations $k$ are required to ensure the monotonicity of the optimizer compared to EAdam. Moreover, we remark that when $\ell^{(j)}$ approaches $0$,  the condition of case $(ii)$ becomes less strict; therefore, it is likely to be satisfied in the final stages of the executions. 

We conclude the theoretical analysis by stating a result on the linear convergence rate of AdamL for both cases.
\begin{proposition}\label{prop:ladam_convergencerate}
Under the same assumptions of  \cref{prop:ladam_monotonicity}, 
if one of the following two conditions is satisfied
\begin{itemize}
\item[$(i)$] $\ell_{\min}\ge1$,
$\sum_{j=0}^k\beta_2^j\ge \tfrac{4\,G\,}{\sqrt{\gamma\,\varepsilon}}$, $\eta\leq \min \big\{ \tfrac{\sqrt{\beta_2}\,G}{\sqrt{\gamma}\,\mu}, \frac{\sqrt{\varepsilon\,\sum_{j=0}^{k}\beta_2^j}}{2\,L}\big\}$
and $k_0\in\mathbb N$ is such that $k_0>\log_{\beta_2}(1-2\,\sqrt{1-\beta_2})$;
\item[$(ii)$] $\ell_{\max}< 1$,
$v_{\min}> 4\,\varepsilon\,\gamma\,\ell^{(k)}$, $\sum_{j=0}^k\beta_2^j\ge\tfrac{4\,G}{\sqrt{v_{\min}}}$, $\eta\le \min\big\{\tfrac{\sqrt{v_{\min}}}{2\,L\,\sqrt{\gamma}}, \tfrac{\sqrt{\beta_2}\,\sqrt{\tfrac{G}{\gamma\,(\ell_{\min})^{\varphi}}}}{\mu}\big\}$, and $k_0\in\mathbb N$ is such that $k_0>\log_{\beta_2}(1-\tfrac{4\,\varepsilon\,\gamma\,\ell^{(k)}}{v_{\min}})$;
\end{itemize}
then for any $k>k_0$, the random vector $\mathbf{x}^{(k)}$  verifies
{\small \begin{align}\label{eq:linearconvergence_ladam}
\expt\,[\,f(\bx^{(k)})-f^*\,]&\leq (1-2\,\mu\,C_{L})^{k-k_0}\,\expt\,[\,f(\bx^{(k_0)})-f^*\,]+\frac{C_{L,1}}{2\,\mu\,C_{L}}\,\sigma+\frac{C_{L,2}}{2\,\mu\,C_{L}}\,\sigma^2,
\end{align}}\noindent
 where in case $(i)$ we have that $C_{L,1}$ and $C_{L,2}$ satisfy \cref{eq:cl12_case1} by replacing $k$ with $k_0$,
 \begin{align*}
 C_{L}=\Big(\tfrac{1}{2}-\tfrac{1}{\sqrt{(1-\beta_2^{k_0+1})\sum_{j=0}^{k_0}\,\beta_2^{j}}}\Big)\cdot\tfrac{\eta}{\sqrt{\beta_2}\,\big(G\,\sqrt{\tfrac{1}{\gamma}} +\sqrt{\tfrac{\varepsilon\,\ell_{\max}}{1-\beta_2}}\big)},
 \end{align*}
and in case $(ii)$ we have $C_{L,1}$ and $C_{L,2}$ satisfy \cref{eq:cl12_case2},
\begin{align*}
    C_L=\Big(\tfrac{1}{2}-\sqrt{\tfrac{\varepsilon\,\gamma\,\ell^{(k)}}{v_{\min}\,(1-\beta_2^{k_0+1})}}\Big)\cdot\tfrac{\eta}{\sqrt{\beta_2}\,\big(G\,\sqrt{\tfrac{1}{\gamma\,(\ell_{\min})^{\varphi}}}+\sqrt{\tfrac{\varepsilon}{1-\beta_2}}\big)},
\end{align*}
and $0<2\,\mu\,C_{L}(k_0)<1$ for both cases.    
\end{proposition}
\begin{proof}
     On the basis of \cref{eq:2mon_ladam}, we have that \begin{align*}
        \tfrac{1}{1-\beta_2^{k+1}}\,\bw^{(k)}&\leq \tfrac{1}{1-\beta_2^{k}}\,\bw^{(k)}=\tfrac{1-\beta_2}{1-\beta_2^{k+1}}\sum_{j=0}^{k} \beta_2^{k-j} \tfrac{(\bg^{(j)})^{ 2}}{\gamma\,(\ell^{(j)})^{\varphi}}  +\tfrac{1}{1-\beta_2^{k+1}}\varepsilon\,\sum_{j=0}^{k} \beta_2^{k-j}\ell^{(j)}\nonumber\\
        &\le \tfrac{(1-\beta_2)\,G^{2}}{\gamma\,(1-\beta_2^{k+1})}\sum_{j=0}^{k}  \tfrac{\beta_2^{k-j}}{(\ell^{(j)})^{\varphi}}  +\tfrac{1}{1-\beta_2^{k+1}}\varepsilon\,\sum_{j=0}^{k} \beta_2^{k-j}\ell^{(j)}\nonumber.
    \end{align*} 
    
\noindent \fbox{$(i)$} If $\ell_{\min}\ge1$, then
\begin{align*}
        \sqrt{\tfrac{\beta_2}{1-\beta_2^{k+1}}\,\bw^{(k)}}
        &\le \sqrt{\tfrac{\beta_2\,(1-\beta_2)\,G^{2}}{\gamma\,(1-\beta_2^{k+1})}\sum_{j=0}^{k} \beta_2^{j}  +\tfrac{\beta_2}{1-\beta_2^{k+1}}\varepsilon\,\sum_{j=0}^{k} \beta_2^{k-j}\ell^{(j)}}\nonumber\\
        &\le \sqrt{\tfrac{\beta_2}{\gamma}}\,G +\sqrt{\tfrac{\beta_2}{1-\beta_2^{k+1}}\varepsilon\,\sum_{j=0}^{k} \beta_2^{k-j}\ell^{(j)}}\nonumber\\
        &\le \sqrt{\beta_2}\,\big(G\,\sqrt{\tfrac{1}{\gamma}} +\sqrt{\tfrac{\varepsilon\,\ell_{\max}}{1-\beta_2}}\big)\nonumber.
    \end{align*} 
Substituting the latter into \cref{eq:lminle1_adaml} and taking the expectation from both sides, we obtain
\begin{align}
 \expt\,[&\,f(\bx^{(k+1)})\,]- \expt\,[\,f(\bx^{(k)})\,]\nonumber\\
&\le  -\underbrace{\eta \,\big(\tfrac{1}{2}-\tfrac{1}{\sqrt{\sum_{j=0}^{k}\,\beta_2^{j}}\,\sqrt{1-\beta_2^{k+1}}}\big)\tfrac{1}{\sqrt{\beta_2}\,\big(G\,\sqrt{\tfrac{1}{\gamma}} +\sqrt{\tfrac{\varepsilon\,\ell_{\max}}{1-\beta_2}}\big)}}_{C_{L}(k)}\,\expt\,[\,\Vert\nabla f(\bx^{(k)})\Vert^2\,]\nonumber\\
&\qquad+ \underbrace{\tfrac{\eta\,G\,\sqrt{n}}{\sqrt{\beta_2\,\varepsilon}\sum_{j=0}^{k-1}\beta_2^{j}}}_{C_{L,1}(k)}\,\sigma+\underbrace{\tfrac{\eta}{\varepsilon\,\sum_{j=0}^{k-1}\beta_2^{j}}\Big(\tfrac{L\,\eta}{2}+\tfrac{\sqrt{1-\beta_2}\,G}{\sqrt{\beta_2\,\gamma}}\Big)}_{C_{L,2}(k)}\,\sigma^2 .\nonumber
\end{align}
On the basis of  $\expt\,[\,\Vert \nabla f(\bx^{(k)})\Vert^2\,]\geq 2\mu \,\expt\,[\,f(\bx^{(k)})-f^*\,]$ and $\log_{\beta_2}(1-2\,\sqrt{1-\beta_2})\le k_0$, for $k_0<k$, we have that $C_{L}(k_0)\le C_{L}(k)$, $C_{L,1}(k_0)\le C_{L,1}(k)$, and $C_{L,2}(k_0)\le C_{L,2}(k)$, which in turn yield
\begin{align}
    \expt\,[\,f(\bx^{(k+1)})\!-f^*\,]\leq (1\!-2\,\mu\,C_{L}(k_0))\,\expt\,[\,f(\bx^{(k)})-f^*\,]+ C_{L,1}(k_0)\,\sigma+C_{L,2}(k_0)\,\sigma^2.\nonumber
\end{align}
The assumption $\log_{\beta_2}(1-2\,\sqrt{1-\beta_2})\le k_0$ implies $\sqrt{(1-\beta_2^{k_0+1})\,\sum_{j=0}^{k_0}\beta_2^j}\ge 2$; together with the property $\mu\le\tfrac{\sqrt{\beta_2}\,G}{\sqrt{\gamma}\,\eta}$, we have that $$\mu\,\eta\le \sqrt{\beta_2}\,(\frac{G}{\sqrt{\gamma}}+\sqrt{\frac{\varepsilon\,\ell_{\max}}{1-\beta_2}})\le \tfrac{\sqrt{\beta_2}\,(\frac{G}{\sqrt{\gamma}}+\sqrt{\frac{\varepsilon\,\ell_{\max}}{1-\beta_2}})\,\sqrt{(1-\beta_2^{k_0+1})\,\sum_{j=0}^{k_0}\beta_2^j}}{\sqrt{(1-\beta_2^{k_0+1})\,\sum_{j=0}^{k_0}\beta_2^j}-2},$$ which in turn indicates that $1-2\,\mu\,C_{L}(k_0)\ge0$. 
With the same argument used to get \cref{eq:condi1_eadam_pl}, we achieve \cref{eq:linearconvergence_ladam}.

\noindent \fbox{$(ii)$} If $\ell_{\max}<1$, we have that
\begin{align*}
        \sqrt{\tfrac{\beta_2}{1-\beta_2^{k+1}}\,\bw^{(k)}}
        &\le \sqrt{\tfrac{\beta_2}{\gamma\,(\ell_{\min})^{\varphi}}}\,G +\sqrt{\tfrac{\beta_2}{1-\beta_2^{k+1}}\,\varepsilon\,\sum_{j=0}^{k} \beta_2^{j}}\nonumber\\
        &\le \sqrt{\beta_2}\,\big(\sqrt{\tfrac{1}{\gamma\,(\ell_{\min})^{\varphi}}}\,G +\sqrt{\tfrac{\varepsilon}{1-\beta_2}}\big)\nonumber.
    \end{align*} 
Substituting the latter into \cref{eq:lminle2_adaml} and taking the expectation from both sides, we obtain
\begin{align}
\expt\,[\,&f(\bx^{(k+1)})\,]- \expt\,[\,f(\bx^{(k)})\,]\nonumber\\
    &\le -\underbrace{\eta\,\Big(\tfrac{1}{2}-\tfrac{\sqrt{\varepsilon\,\gamma\,\ell^{(k)}}}{\sqrt{v_{\min}\,(1-\beta_2^{k+1})}}\Big)\tfrac{1}{\sqrt{\beta_2}\,\big(\sqrt{\tfrac{1}{\gamma\,(\ell_{\min})^{\varphi}}}\,G +\sqrt{\tfrac{\varepsilon}{1-\beta_2}}\big)}}_{C_{L}(k)}\,\expt\,[\,\Vert\nabla f(\bx^{(k)})\Vert^2\,] \nonumber\\
    &\qquad+\underbrace{\tfrac{\eta\,G\,\sqrt{n\,\varepsilon\,\gamma\,\ell^{(k)}}}{\sqrt{\beta_2}\,v_{\min}}}_{C_{L,1}(k)}\,\sigma+\underbrace{\tfrac{\eta}{v_{\min}}\big(\tfrac{L\,\eta\,\gamma}{2}+\tfrac{\sqrt{1-\beta_2}\,G}{\sqrt{\beta_2}}\big)}_{C_{L,2}(k)}\,\sigma^2.\nonumber
    \end{align}
In view of $v_{\min}\ge 4\,\varepsilon\,\gamma\,\ell^{(k)}$ and $k_0>\log_{\beta_2}(1-\tfrac{4\,\varepsilon\,\gamma\,\ell^{(k)}}{v_{\min}})$, we have $\frac{2\,\sqrt{\varepsilon\,\gamma\,\ell^{(k)}}}{\sqrt{v_{\min}\,(1-\beta_2^{k_0+1})}}\le 1$. Combining the latter with  $\mu\le \tfrac{\sqrt{\beta_2}\,\sqrt{\tfrac{1}{\gamma\,(\ell_{\min})^{\varphi}}}\,G}{\eta}$ yields $\mu\,\eta\le \beta_2\,\big(\sqrt{\tfrac{1}{\gamma\,(\ell_{\min})^{\varphi}}}\,G +\sqrt{\tfrac{\varepsilon}{1-\beta_2}}\big)\le \tfrac{\sqrt{\beta_2}\,\big(\sqrt{\tfrac{G}{\gamma\,(\ell_{\min})^{\varphi}}} +\sqrt{\tfrac{\varepsilon}{1-\beta_2}}\big)\,\sqrt{1-\beta_2^{k_0+1}}}{\sqrt{1-\beta_2^{k_0+1}}-\tfrac{2\sqrt{\varepsilon\,\gamma\,\ell^{(k)}}}{\sqrt{v_{\min}}}}$, which in turn implies $1-2\mu\,C_{L}(k_0)\ge0$.
\end{proof}
In \cref{prop:ladam_convergencerate}, the coefficient in front of $\sigma$ depends on the value of $\ell^{(k)}$ in both cases. Intuitively, as $\ell^{(k)}$ approaches zero, indicating that the objective function value approaches its optimal value, the coefficient associated with $\sigma$ converges to zero.  Increasing the value of $\gamma$ yields smaller coefficients in front of the $\sigma$-dependent terms in case (i), and larger coefficients in case (ii). 
Additionally, a larger value of $\gamma$ raises the chance that the AdamL optimizer exhibits the convergence behavior described in case (i). The result about case $(ii)$  also shows that a smaller value of $\ell_{\min}$ results in larger coefficients in front of $\sigma$ and $\sigma^2$.
Finally, a smaller value of $\ell_{\min}$ leads to a smaller value of $C_L$, which in turn results in a slower convergence rate of the optimizer. 
\section{Experimental results}
\label{sec:experiments}
This section is dedicated to the evaluation and comparison of the performances of the various optimizers analyzed in this work.  First, we consider two benchmark case studies, namely the Three-Hump Camel function and the Rosenbrock function, to show that AdamL makes good use of the information from the loss function and gets better results with respect to the other Adam's variants. More precisely, in the Three-Hump Camel function case, we observe that AdamL has a higher chance of converging to the global optimum than its competitors. In the Rosenbrock function case, all optimizers converge to the global minimum but AdamL necessitates fewer iteration steps. Since in this experiment, AdamL stays always in the Non-Adaptive mode, we also compare it with GD with momentum to showcase that they have similar performances. We remark that the Rosenbrock function satisfies the assumptions (cf.~\cref{eq:PLineq} and \cref{eq:Lineq}) of our convergence analysis, only when restricting $\bx$ to a compact set. For instance, it has empirically shown that it satisfies \cref{eq:Lineq} and \cref{eq:PLineq} with $L=2610$ and $\mu=0.2$ for $\bx\in [0,2]^2$~\cite[Section 6.2]{xia2023convergence}. Our numerical results show that the convergence rate of AdamL exhibits linearity with varying slopes as $x^{(k)}$ enters different regions of the Rosenbrock function.
 For both Rosenbrock and Three-Hump Camel functions, we employ the default parameters used in Adam, i.e., $\beta_1=0.9$, $\beta_2=0.999$, and $\eta=0.001$. For AdamL we set $\gamma=1$, and $\ell^{(k)}=f(\bx^{(k)})$.  Similar to the strategy in \cite[Sec.~3]{zhuang2020adabelief}, we have tuned the values of the parameters $\varepsilon$ and $\varphi$, by searching for $\varepsilon$ among $\{10^{-4}, 10^{-8}, 10^{-12}\}$, and $\varphi$ among $\{1, 2, 4\}$. The best results, which are the only ones reported in this document, have been obtained by employing $\varepsilon=10^{-4}$ for EAdam, AdaBelief, and AdamL and $\varphi=1$ in the AdamL optimizer.  The experiments on these benchmark functions are implemented using Matlab on CPUs. 

In addition to these standard benchmark functions, similar to prior works like \citep{luoadaptivebound,zhuang2020adabelief, yuan2020eadam}, we test the optimizers on machine learning case studies as training tasks and network types. These tasks include image classification with convolutional NNs (CNNs) on \textsf{CIFAR10} and \textsf{CIFAR100} datasets, language modeling with long short term memory networks (LSTM) on \textsf{Penn Treebank} corpus, and training Wasserstein-GANs (WGANs) on \textsf{CIFAR10} and \textsf{Anime Faces} datasets\footnote{https://github.com/learner-lu/anime-face-dataset}. For each case study, we conduct the simulation five times and present the average results. Furthermore, we compare the performances of AdamL using two strategies: one that employs Adam for estimating the minimum and maximum objective function values, and another that utilizes \cref{ag:strategyI} for training WGAN and LSTM. We show that these two strategies provide comparable convergence speed and accuracy. 
The implementations of EAdam\footnote{https://github.com/yuanwei2019/EAdam-optimizer} and AdaBelief\footnote{https://github.com/juntang-zhuang/Adabelief-Optimizer} are based on open-source codes, ensuring precise reproductions. All these experiments are implemented using PyTorch on the Tesla V100 PCIe 16GB GPUs. 

\subsection{Three-Hump Camel function}\label{sec:threehump}
The Three-Hump Camel function $f(\bx)=2\,x_1^2-1.05\,x_1^4+\frac{1}{6}\,x_1^6+x_1\,x_2+x_2^2$ is a simple multimodal function with three local minima. It is often evaluated within the square for which $x_1, x_2 \in [-5, 5]$; the global minimum is $0$ which is attained at $[0,0]^T$. 
\begin{figure}[htp]
    \centering
    \subfloat[]{\label{fig:3hump_1}\includegraphics[width=0.33\textwidth]{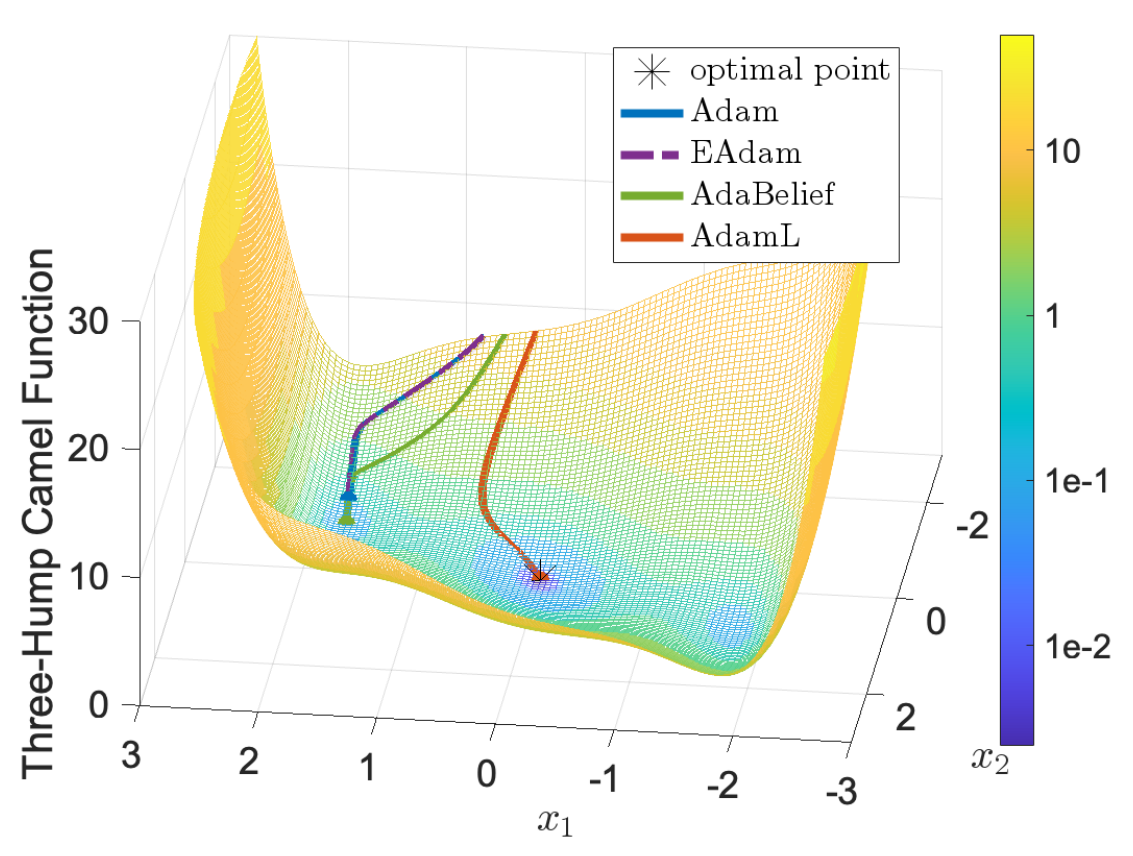}}
    \subfloat[]{\label{fig:3hump_2}\includegraphics[width=0.33\textwidth]{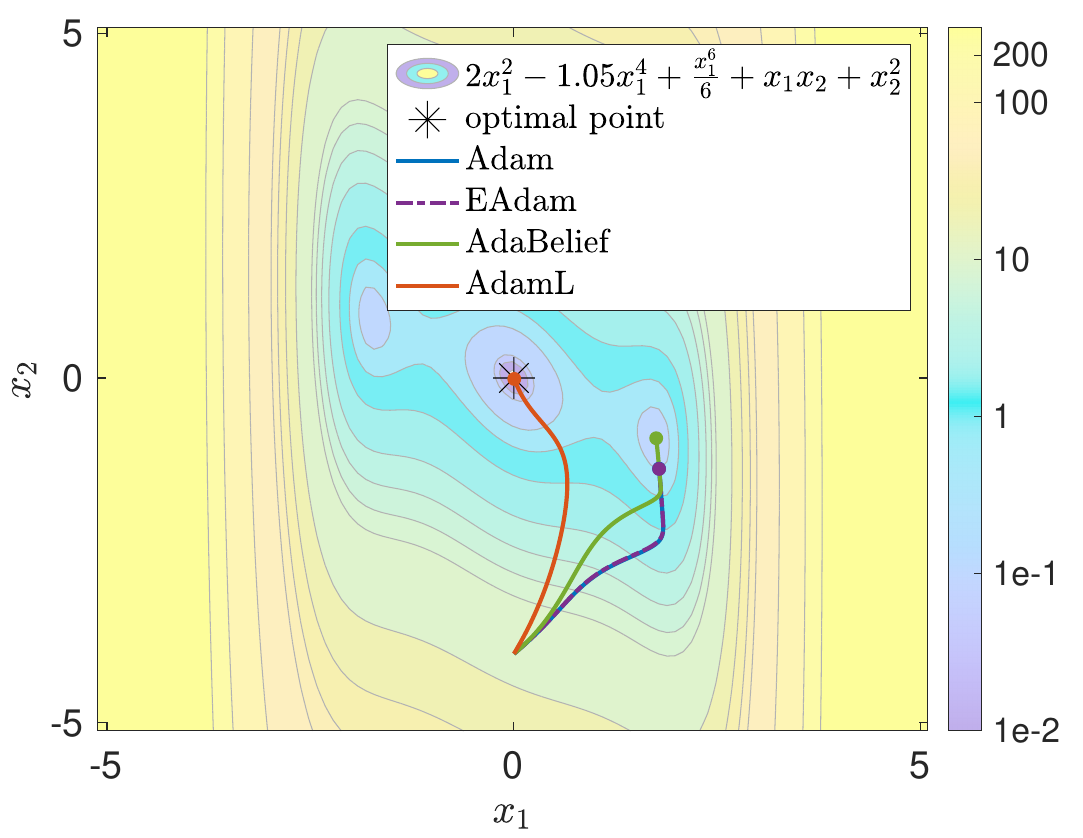}}
    \subfloat[]{\label{fig:3hump_3}\includegraphics[width=0.33\textwidth]{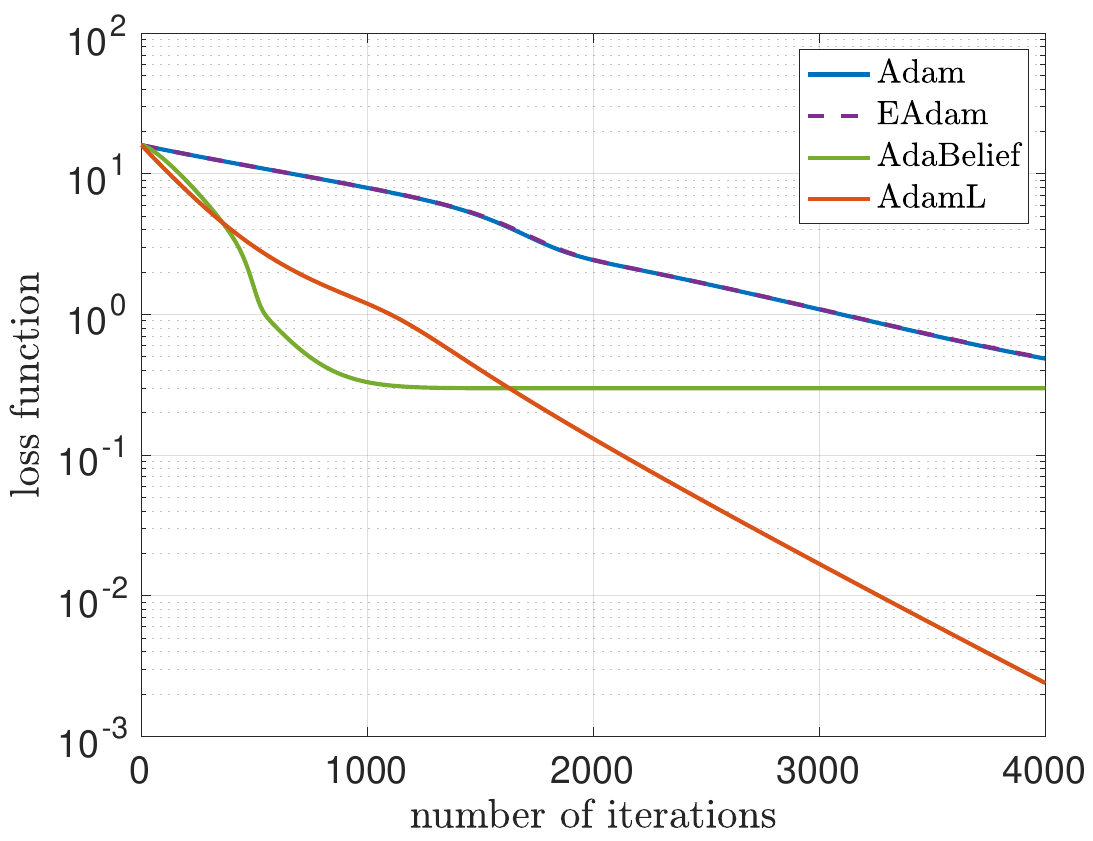}}
    \caption[Three-Hump Camel Function: trajectories of different optimizer in contour and mesh plots, and corresponding objective function values.]{Three-Hump Camel Function: trajectories of the various optimizers in the contour (a) and mesh (b) plot, and the corresponding objective function values (c), when the starting point is $\bx^{(0)}=[0,-4]^T$ and after $4000$ iteration steps.}
    \label{fig:3hump_x0_0-4}
\end{figure}
\cref{fig:3hump_x0_0-4} shows the trajectories of the various optimizers in the mesh plot, contour map, and their objective function values, with the starting point $\bx^{(0)}=[0,-4]^T$. Only AdamL converges to the global optimum; the other optimizers lead $\bx^{(k)}$ to a local optimum. It can be seen from \cref{fig:3hump_1} that the gradient and curvature near the local optima are small, which is very similar to the region of $x_7$, $x_8$, and $x_9$ in \cref{fig:curve2}. From \cref{fig:3hump_1,fig:3hump_2}, it can be seen that the gradients and curvatures around the global optimum are larger than those around the other local optima. As discussed in \cref{sec:adaml}, Adam and EAdam take large updating steps in the direction of small gradients, whereas AdaBelief takes large updating steps in the presence of small curvatures. From \cref{fig:3hump_2,fig:3hump_3}, it can be observed that the trajectories of Adam and EAdam are identical, which implies that the accumulated $\varepsilon$ in the second moment estimate of EAdam does not affect the convergence behavior in this study. By considering the loss function values, when the gradients are small and loss function values are large, the AdamL optimizer takes large updating steps in the directions that lead to a decrease in the loss function. This explains why, for this starting point, AdamL is the only method that converges to the global optimum.

\begin{figure}[htp]
    \centering
    \subfloat[]{\label{fig:3hump_-4-4}\includegraphics[width=0.33\textwidth]{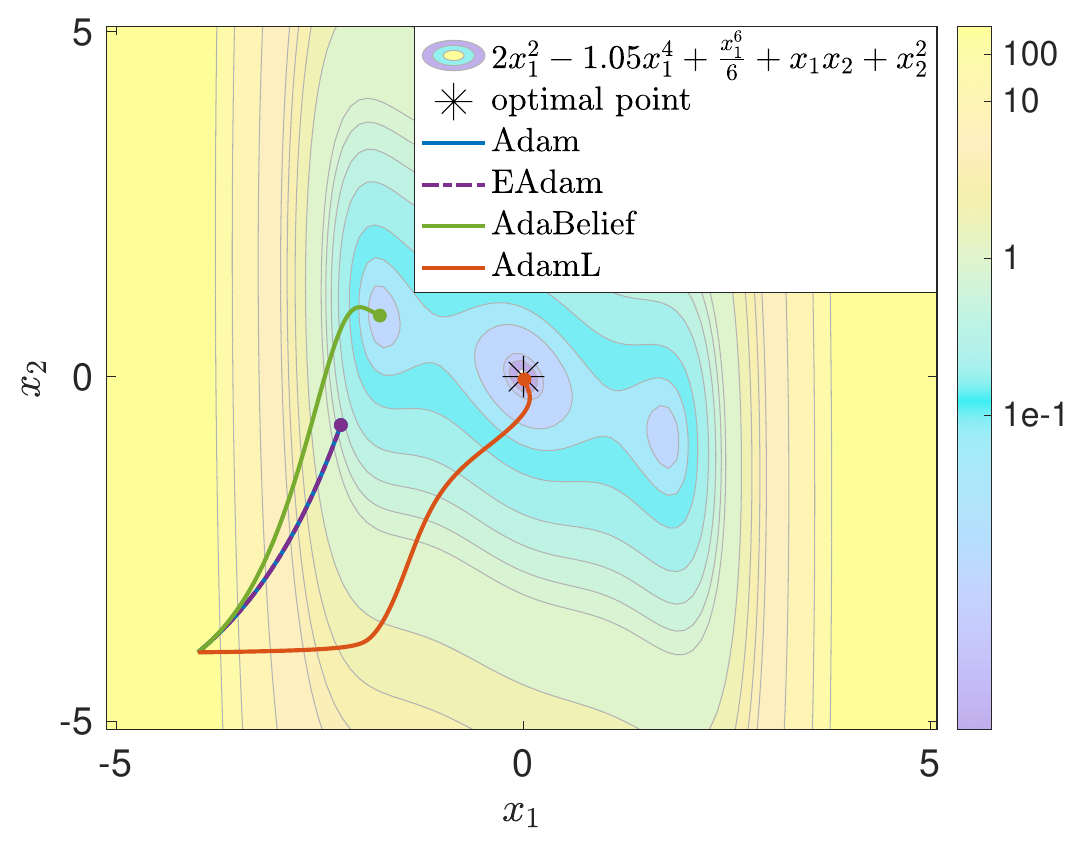}}
    \subfloat[]{\label{fig:3hump_4-4}\includegraphics[width=0.33\textwidth]{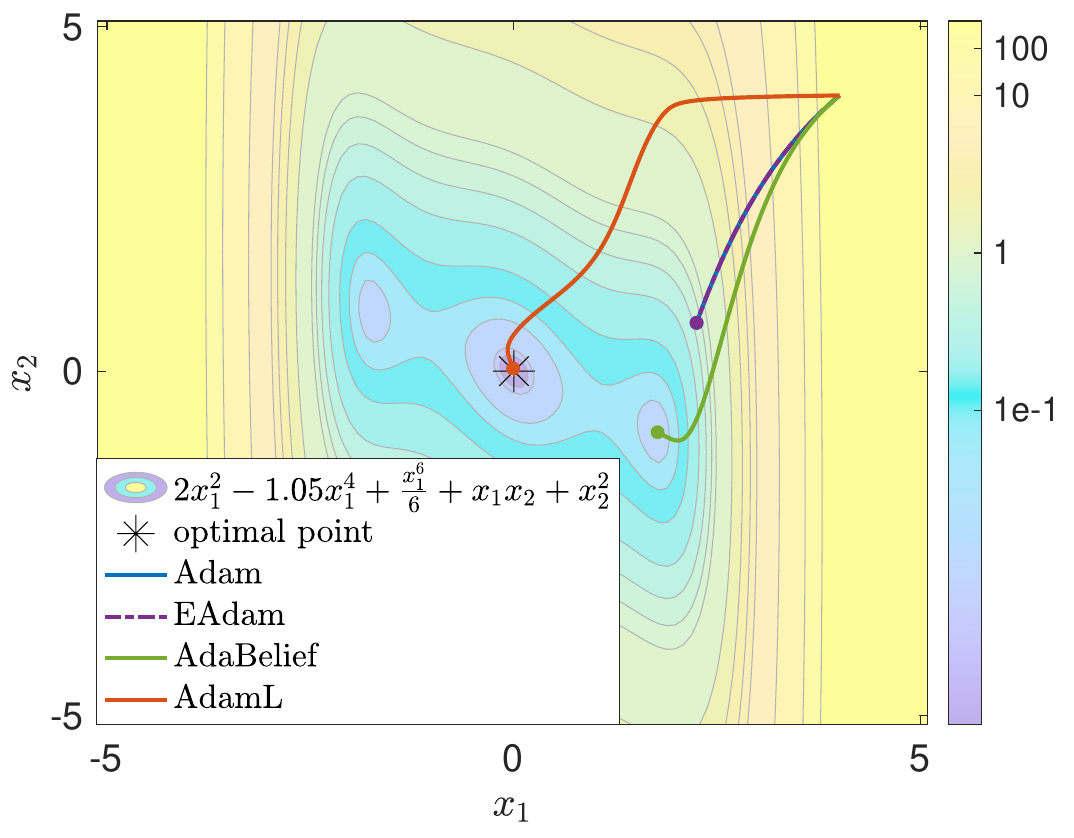}}
    \subfloat[]{\label{fig:3hump_40}\includegraphics[width=0.33\textwidth]{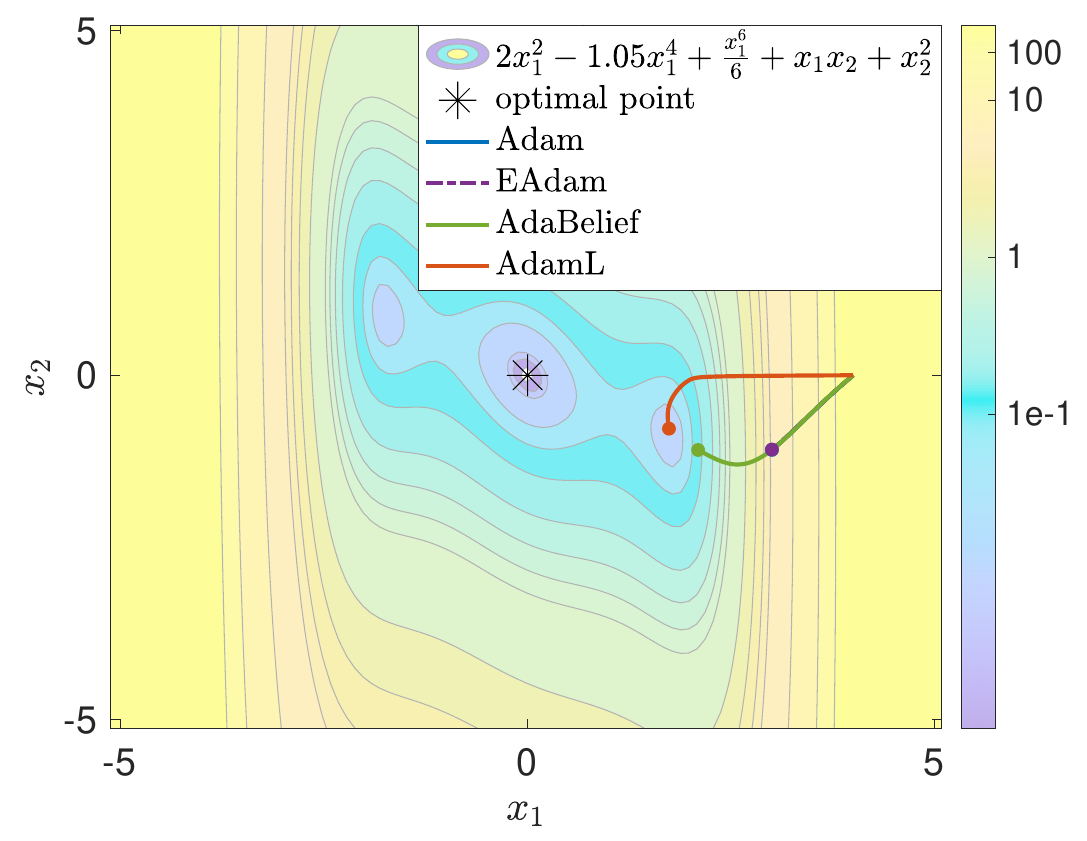}}
    \caption[Three-Hump Camel Function: trajectories of different optimizer with different starting points]{Three-Hump Camel Function: trajectories of different optimizer in the contour plot when the starting points are $\bx^{(0)}=[-4,-4]^T$ (a), $\bx^{(0)}=[4,4]^T$ (b) and $\bx^{(0)}=[4,0]^T$ (c), after $4000$ iteration steps.}
    \label{fig:3hump_x0different}
\end{figure}

In \cref{fig:3hump_x0different}, we repeat the experiment with different starting points, i.e., $\bx^{(0)}=[-4,-4]^T$ (\cref{fig:3hump_-4-4}), $\bx^{(0)}=[4,4]^T$ (\cref{fig:3hump_4-4}) and $\bx^{(0)}=[4,0]^T$ (\cref{fig:3hump_40}). In all cases, the convergence speed of the Adabelief optimizers is higher than that of Adam and EAdam, as the former takes into account the curvature of the function. We remark that AdamL performs better than its competitors in all tests by either finding the global optimum or obtaining the smallest values for the objective function. 

\subsection{Rosenbrock function}\label{sec:rosenbrock}
For the minimization of the Rosenbrock function, we choose the same settings as those applied in \cref{sec:threehump} for adaptive optimizers ($\eta=10^{-3}$). Additionally, we also run GD with momentum, to showcase that the AdamL consistently employs its Non-Adaptive Mode. In the use of GD with momentum, we apply two different settings. In the first setup, a constant $\eta=10^{-4}$ is used,  to maintain monotonicity. In the second scenario, we manually increase the value of $\eta$ from $\eta=10^{-4}$ to $\eta=10^{-3}$ when $\bx$ reaches the long and narrow flat valley of the Rosenbrock function. Specifically, $\eta$ is increased by a factor of 10 when $f<1$. Empirically, we observe that choosing $\eta=10^{-3}$ for the initial step of GD causes divergence.

\begin{figure}[htp]
\centering
\subfloat[]{\label{fig:rosenbga}\includegraphics[width=0.38\textwidth]{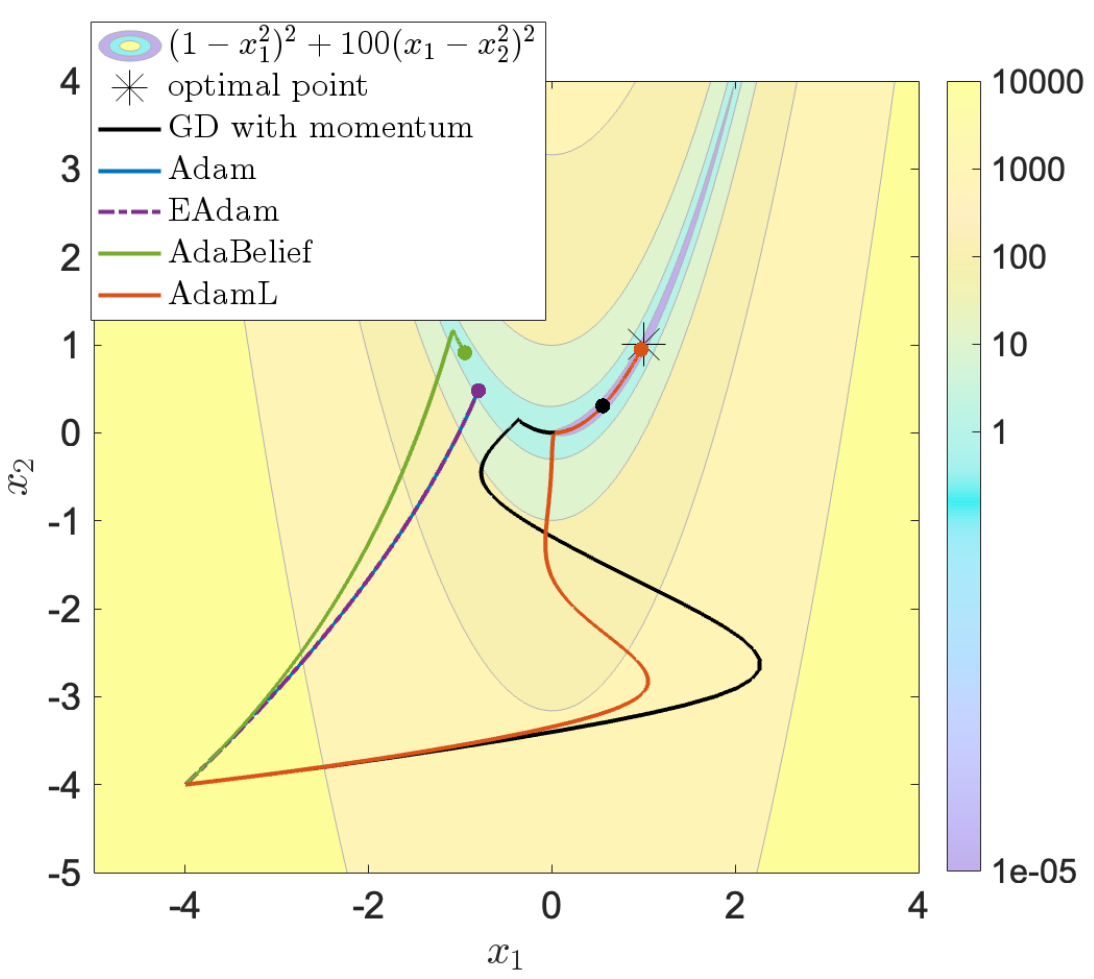}}\,
\subfloat[]{\label{fig:rosenbgb}\includegraphics[width=0.42\textwidth]{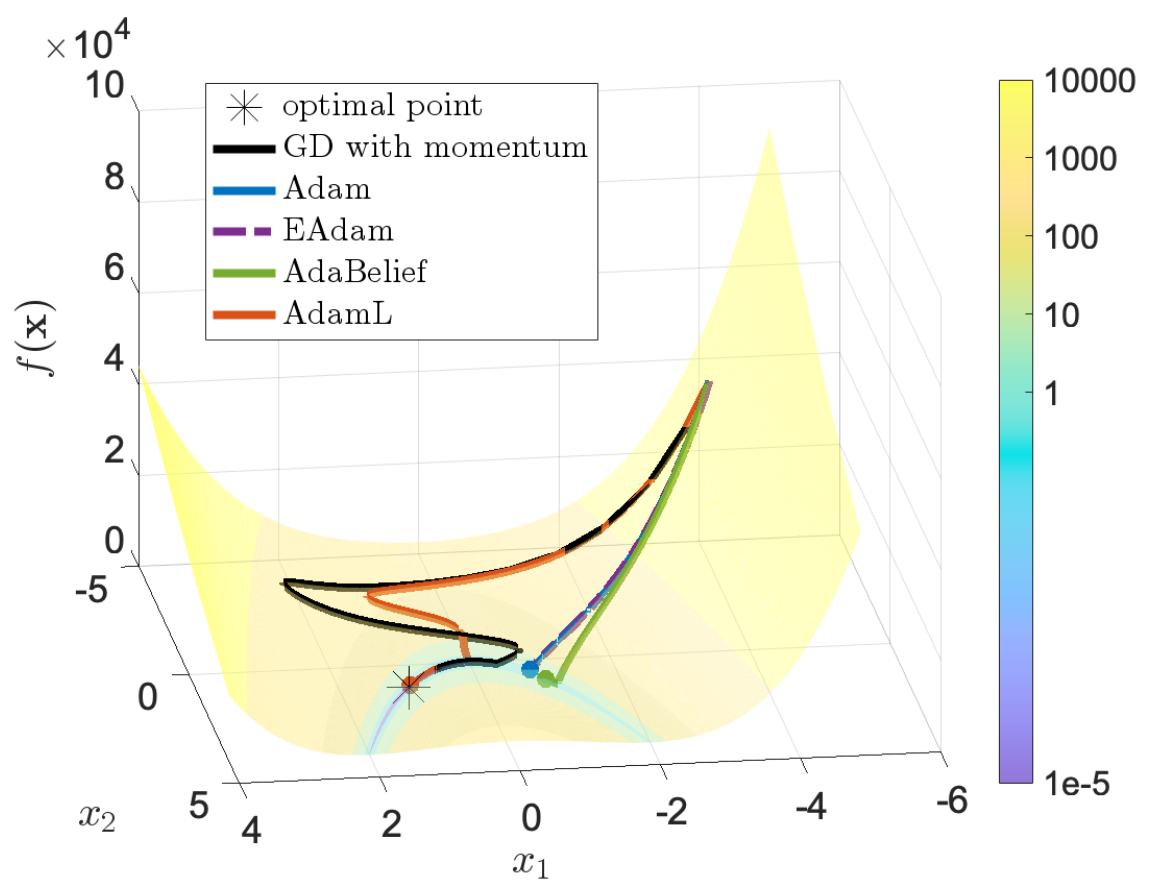}}\\
\subfloat[]{\label{fig:rosenbgc}\includegraphics[width=0.42\textwidth]{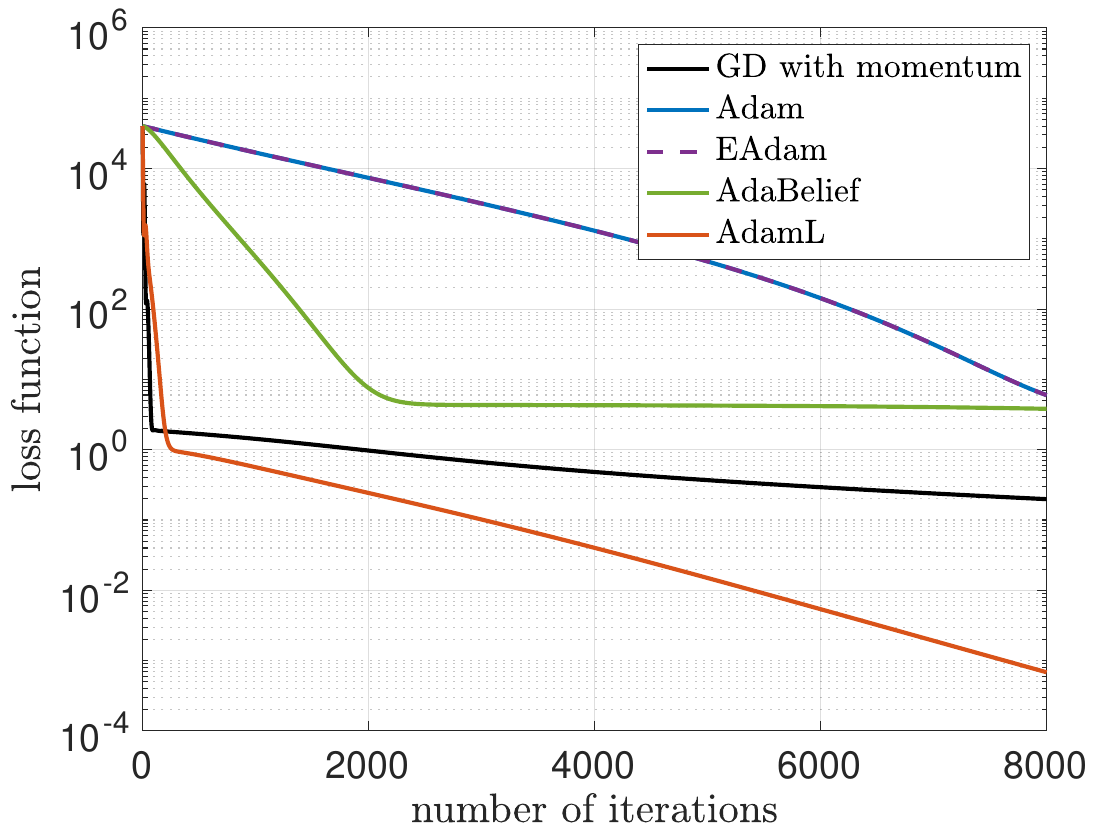}}\,
\subfloat[]{\label{fig:rosenb_f2}\includegraphics[width=0.42\textwidth]{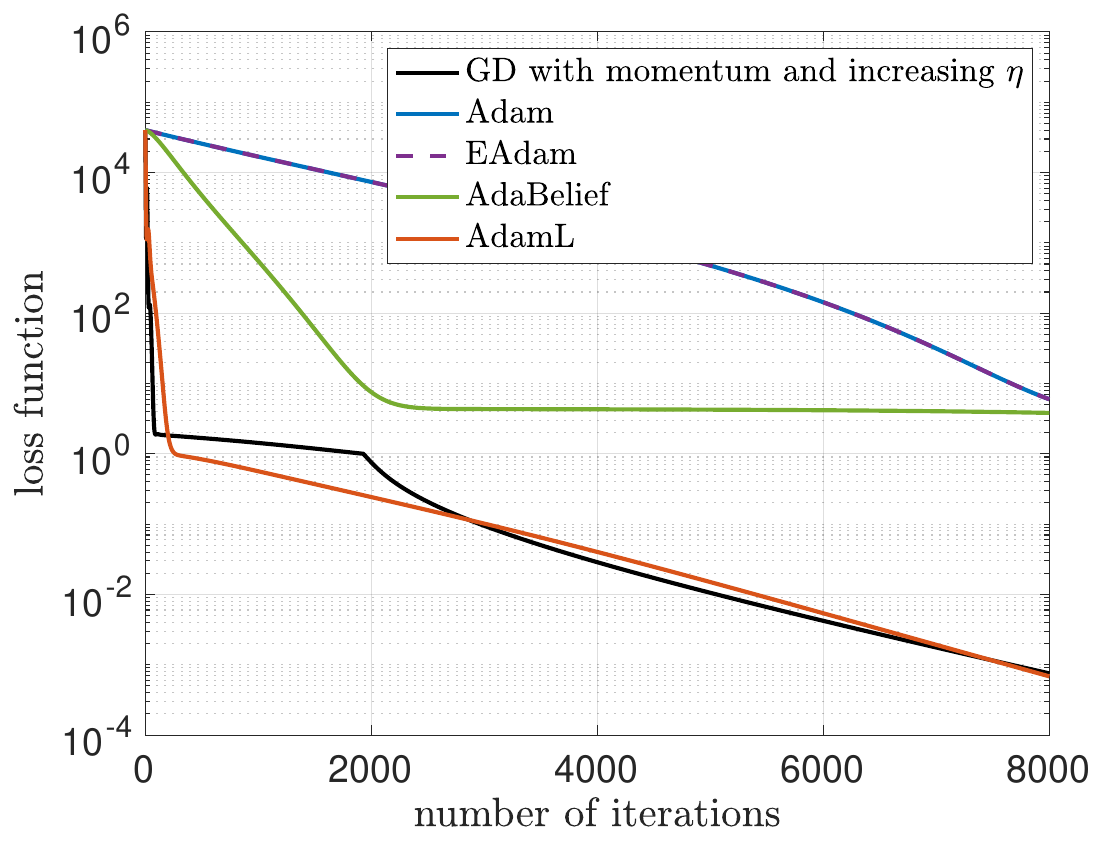}}
\caption[Rosenbrock function: trajectories of different optimizer in contour and mesh plots, and corresponding objective function values]{Rosenbrock function: trajectories of different optimizers in the contour (a) and mesh (b) plot, the corresponding objective function values when $\eta$ is constant (c), and increasing $\eta$ by a factor of 10 in GD with momentum (d), when $\bx^{(0)}=[-4,-4]^T$ and after $8000$ iteration steps.}\label{fig:rosenbrock}
\end{figure}
\cref{fig:rosenbga,fig:rosenbgb} show the trajectories of GD with momentum with constant $\eta=10^{-4}$ and various adaptive optimizers, including Adam, EAdam, AdaBelief, and AdamL, in both contour and mesh plots, respectively. Additionally, \cref{fig:rosenbgc} depicts the corresponding objective function values after 8000 iteration steps. In \cref{fig:rosenb_f2}, we report the results of GD with momentum with $\eta$ increasing from $\eta=10^{-4}$ to $\eta=10^{-3}$ when $f<1$.
 In \cref{fig:rosenbga}, we see that different optimizers cause $\bx^{(k)}$ to update in different directions. In particular, we note that the EAdam and AdaBelief optimizers update in a direction that is almost aligned with that of the Adam optimizer. This suggests that the Adaptive Mode is activated in EAdam and AdaBelief. Furthermore, the Adam and EAdam optimizers take larger steps in the directions characterized by small gradients, whereas Adabelief favors larger steps in directions where the curvatures are shallow. Therefore, they quickly reach the narrow and flat valley. From \cref{fig:rosenbga}, we see that GD with momentum and AdamL force $\bx^{(k)}$ to update in a similar direction, suggesting that the Non-Adaptive Mode is activated in AdamL. However, from \cref{fig:rosenbga,fig:rosenbgc}, it is evident that the convergence rates of AdamL and GD with momentum are different when $\bx^{(k)}$ reaches the narrow valley near the optimum. When manually increasing $\eta$ by a factor of 10 in GD with momentum, as shown in \cref{fig:rosenb_f2}, the convergence rates of AdamL and GD with momentum become comparable. This implies that AdamL always uses the Non-adaptive updating strategy. 
However, as it is discussed in \cref{sec:adaml} when the Non-Adaptive Mode is activated in AdamL, the inclusion of objective function values in the second moment estimate results in an increasing stepsize as the objective function value decreases. Consequently, AdamL obviates the need for manually increasing stepsize, unlike GD with momentum, and attains faster convergence in comparison to GD with a constant stepsize.

\subsection{CNNs for image classification}
In this section, we compare the performances of the various optimizers in image classification problems by training some common types of convolutional neural networks (CNNs), i.e., \textsf{VGG11}, \textsf{ResNet18}, and  \textsf{DenseNet121}, on the \textsf{CIFAR10} and \textsf{CIFAR100} datasets. All these CNNs use the softmax activation function in the output layer and optimize a multi-class cross-entropy loss function using different optimizers during backward propagation. For a detailed structure of these CNNs, see \citep{simonyan2015very,he2016deep,huang2017densely}. We apply the default setting of Adam, recommended in \citep{kingma2014adam}, for all the optimizers, i.e., $\beta_1=0.9$, $\beta_2=0.999$, $\eta=0.001$, and $\varepsilon=10^{-8}$. 

To appropriately choose the function $\ell^{(k)}$ and the value of $\varphi$ in the AdamL optimizer, we look at the range of the objective function values for these CNNs. Because of the use of the softmax activation function and the multi-class cross-entropy loss function, the loss function values are non-negative, with the optimal value close to $0$, i.e., $f^* \approx 0$. Additionally, the typical maximum values of the loss function, observed when optimizing these CNNs, are around $1$. In view of these remarks, we set $\ell^{(k)}=f(\bx^{(k)})$ for all the experiments contained in this subsection. Additionally, the loss function values vary depending on the network architectures and training datasets. When the network architecture is not sufficiently complex to effectively represent the features in the training dataset, the resulting minimum value of the loss function will often be significantly distant from zero. For example, when comparing \textsf{VGG11} to \textsf{ResNet18} and  \textsf{DenseNet121} on \textsf{CIFAR10} at the same iteration, \textsf{VGG11} may yield higher loss function values (cf.~\cite[Fig.~1 in Appendix E]{zhuang2020adabelief}).  In this case, a large value $\varphi$ is necessary to ensure that the minimum loss function approaches the optimal value as closely as possible. This choice facilitates a smooth transition between the Non-Adaptive and Adaptive modes within the AdamL optimizer. As discussed after \cref{eq:2ndmoment_adaml}, $\varphi$ should also be chosen such that $(\ell^{(k)})^{\varphi}$ reacts sensitively as $\ell^{(k)}$ approaches $0$. This requires choosing an appropriate value for $\varphi$: \emph{large enough for \textsf{VGG11}, but not too big for \textsf{ResNet18} and  \textsf{DenseNet121}}. For this reason, we define $\varphi$ as a function of the training error $e_{\mathrm{train}}$ (\%) in the exponent; more precisely we set $\varphi=\max \{4,4+\log_{10}(e_{\mathrm{train}})\}$.
With this choice, $\varphi$ is set to $4$ when the training error is $1\%$ or lower and a higher value of $\varphi$ when the training error is larger than $1\%$.

In \citep{luoadaptivebound,zhuang2020adabelief,yuan2020eadam}, to improve the training accuracy, the learning rates $\eta$ of the optimizers are manually reduced by a factor of $0.1$ when the training accuracy reaches a steady state. Instead of manually decreasing the value of $\eta$, the AdamL optimizer uses $\gamma$ to control the transition from SGD to the adaptive method. As it is stated in \cref{sec:hyperparametersinAdamL}, a larger value of $\gamma$ postpones the shift to the Adaptive Mode. We start with the default value $\gamma=1$ and monitor the curve of training accuracy until we observe that it becomes a steady state. If it reaches a steady state too early, we increase $\gamma$ by a factor of 10 and observe the effect on training accuracy. If the change is minimal, we consider larger increments, such as multiplying $\gamma$ by $50$ or $100$. If we notice a significant improvement in training accuracy, we choose to make smaller adjustments to $\gamma$, typically multiplied by a factor less than $10$. This way, we generate the results associated with AdamL in \cref{fig:class_cifar10,fig:class_cifar100}; the latter illustrates the convergence behavior of the method for different values of $\gamma$ when training CNNs on the \textsf{CIFAR10} and \textsf{CIFAR100} datasets. 
\begin{figure}[htp]
\centering
\subfloat[\textsf{VGG11} on \textsf{CIFAR10}]{\label{fig:vgg_cifar10_1}\includegraphics[width=0.45\textwidth]{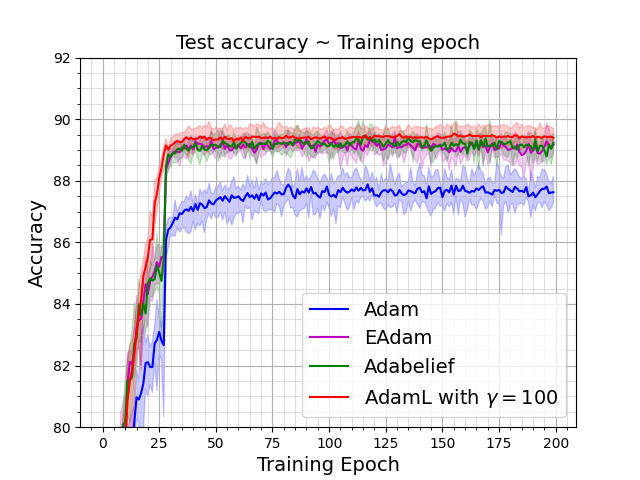}}\,
\subfloat[\textsf{VGG11} on \textsf{CIFAR10}]{\label{fig:vgg_cifar10_2}\includegraphics[width=0.45\textwidth]{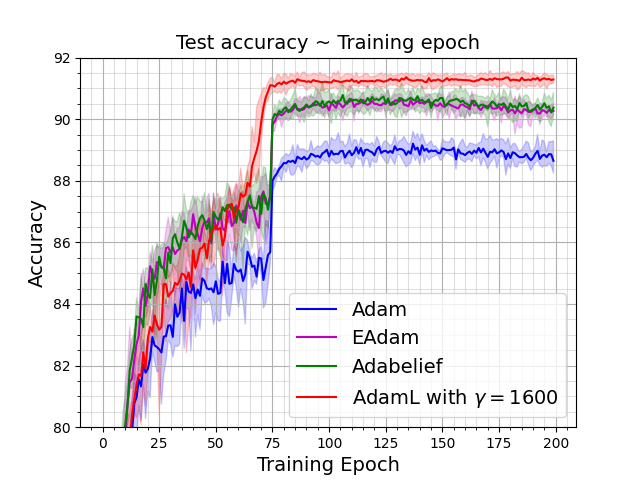}}\\
\subfloat[ \textsf{ResNet18} on \textsf{CIFAR10}]{\label{fig:resnet_cifar10_1}\includegraphics[width=0.45\textwidth]{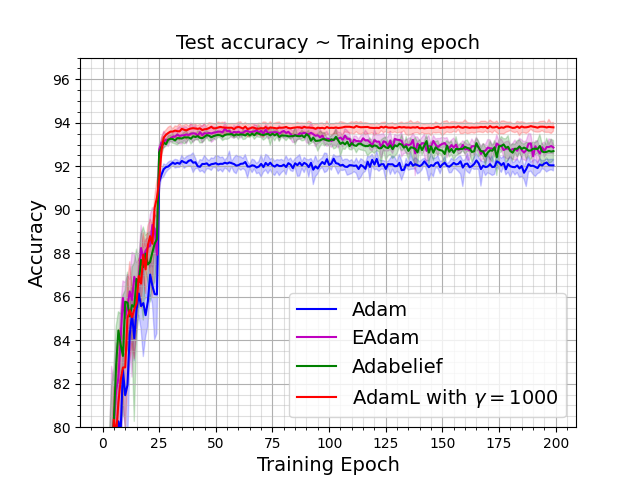}}\,
\subfloat[ \textsf{ResNet18} on \textsf{CIFAR10}]{\label{fig:resnet_cifar10_2}\includegraphics[width=0.45\textwidth]{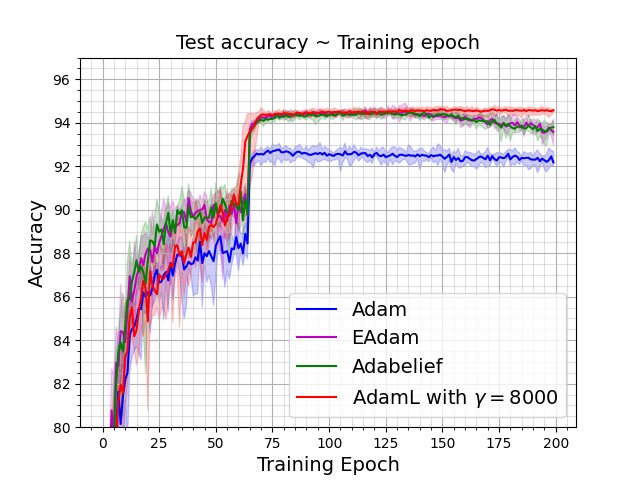}}\\
\subfloat[ \textsf{DenseNet121} on \textsf{CIFAR10}]{\label{fig:densenet_cifar10_1}\includegraphics[width=0.45\textwidth]{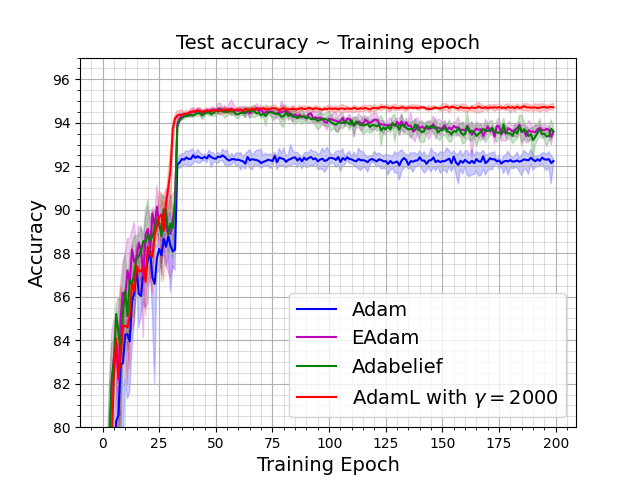}}\,
\subfloat[ \textsf{DenseNet121} on \textsf{CIFAR10}]{\label{fig:densenet_cifar10_2}\includegraphics[width=0.45\textwidth]{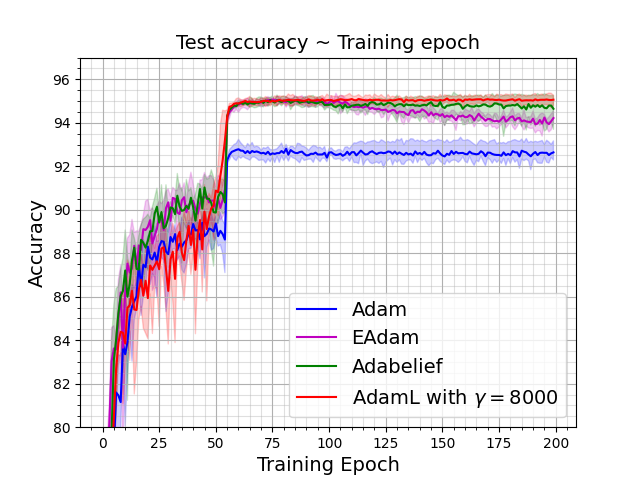}}
\caption[CIFAR10 Dataset: Testing accuracy of CNNs]{Mean values of testing accuracy of CNNs over 5 independent runs on \textsf{CIFAR10} dataset trained using Adam, Adabeleif, and AdamL with early (left column) and late (right column) decay epoch.}\label{fig:class_cifar10}
\end{figure}
\begin{figure}[htp]
\centering
\subfloat[\textsf{VGG11} on \textsf{CIFAR100}]{\label{fig:vgg_cifar100_1}\includegraphics[width=0.45\textwidth]{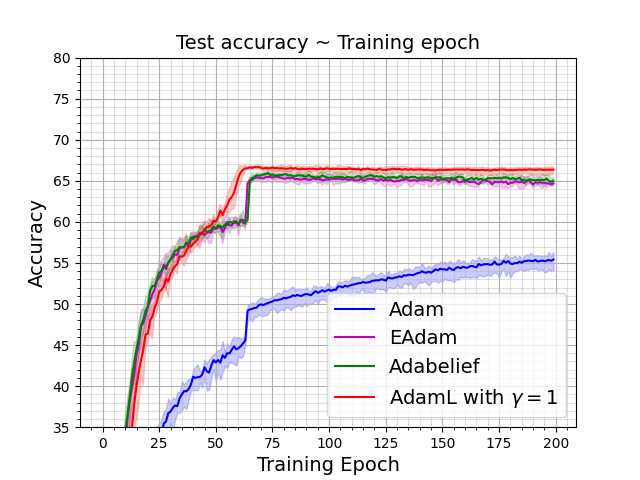}}\,
\subfloat[\textsf{VGG11} on \textsf{CIFAR100}]{\label{fig:vgg_cifar100_2}\includegraphics[width=0.45\textwidth]{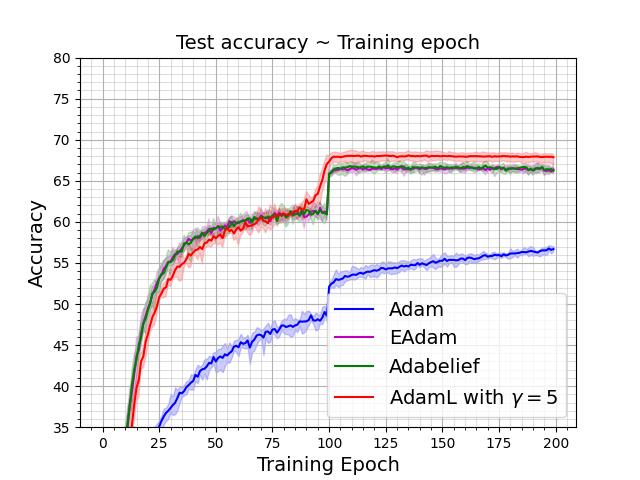}}\\
\subfloat[ \textsf{ResNet18} on \textsf{CIFAR100}]{\label{fig:resnet_cifar100_1}\includegraphics[width=0.45\textwidth]{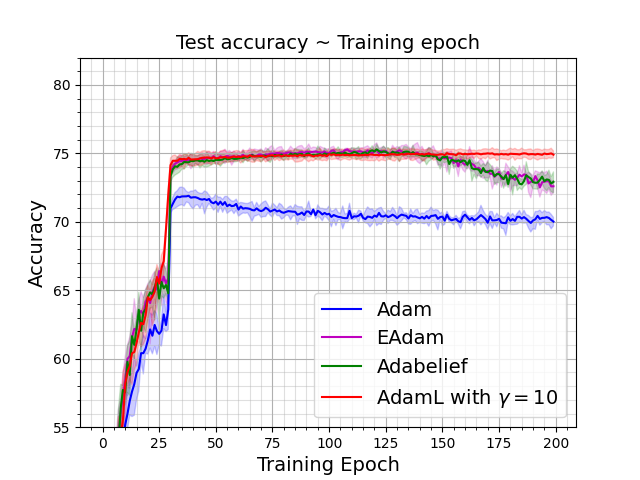}}\,
\subfloat[ \textsf{ResNet18} on \textsf{CIFAR100}]{\label{fig:resnet_cifar100_2}\includegraphics[width=0.45\textwidth]{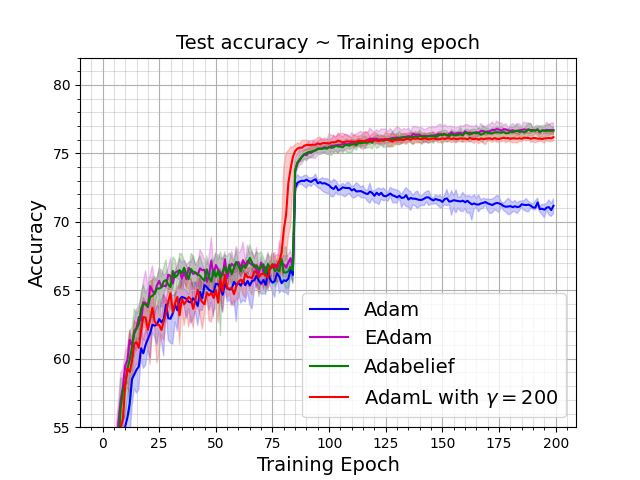}}\\
\subfloat[ \textsf{DenseNet121} on \textsf{CIFAR100}]{\label{fig:densenet_cifar100_1}\includegraphics[width=0.45\textwidth]{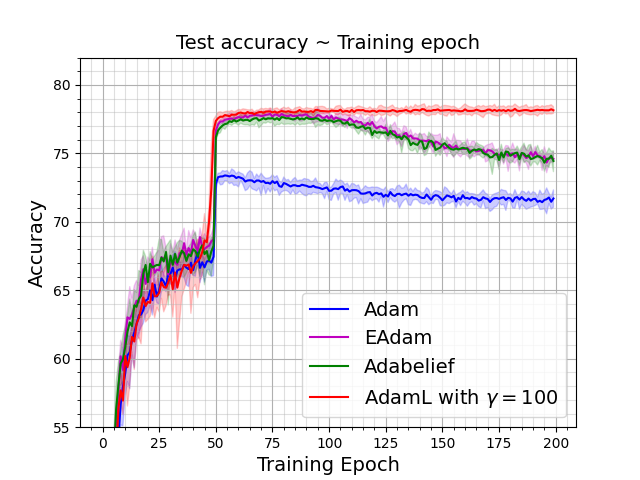}}\,
\subfloat[ \textsf{DenseNet121} on \textsf{CIFAR100}]{\label{fig:densenet_cifar100_2}\includegraphics[width=0.45\textwidth]{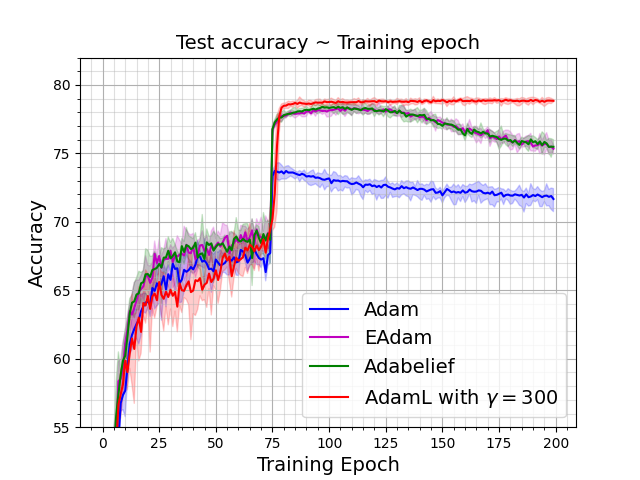}}
\caption[\textsf{CIFAR100} Dataset: Testing accuracy of CNNs]{Mean values of testing accuracy of CNNs over 5 independent runs on \textsf{CIFAR100} dataset trained using Adam, Adabelief, and AdamL with early (left column) and late (right column) decaying epoch.}\label{fig:class_cifar100}
\end{figure}
The performances of AdamL are compared with those obtained by running Adam, EAdam, and AdaBelief and manually reducing the learning rate in the same training epoch when AdamL switches from the Non-Adaptive to the Adaptive mode. \cref{fig:class_cifar10,fig:class_cifar100} show comparison of testing accuracies achieved by various optimizers for \textsf{VGG11},  \textsf{ResNet18}, and  \textsf{DenseNet121} when trained on the \textsf{CIFAR10} and \textsf{CIFAR100} datasets, respectively. It can be observed that the AdamL optimizer demonstrates higher testing accuracies and faster convergence compared to other optimizers when training \textsf{VGG11}. In the cases of  \textsf{ResNet18} and  \textsf{DenseNet121}, the highest testing accuracies obtained by EAdam, AdaBelief, and AdamL are very similar. When decreasing the learning rate at the early stage of training, both EAdam and AdaBelief show unstable convergence, as defined in \citep{ahn2022understanding}. Additionally, these two optimizers demonstrate similar testing accuracies across all numerical tests in this subsection. Increasing the value of $\gamma$ in the AdamL optimizer improves testing accuracies across all CNNs in \cref{fig:class_cifar10,fig:class_cifar100}, but it simultaneously leads to slower convergence. Furthermore, the Adam optimizer consistently produces the lowest testing accuracies among all the optimizers. By comparing \cref{fig:class_cifar10,fig:class_cifar100}, we can draw a similar conclusion for \textsf{CIFAR100} as we do for \textsf{CIFAR10}. 

It is worth noting that although tuning the value of $\gamma$ for different CNNs and datasets typically involves monitoring changes in training errors, \cref{fig:class_cifar10,fig:class_cifar100} demonstrate that it is not necessary to postpone tuning $\gamma$ until the later stages of training. In most cases, within the first 25 training epochs, it is possible to determine if the chosen $\gamma$ is appropriate.
\subsection{Experiments with generative adversarial network}
Adaptive optimizers such as Adam and RMSProp are the workhorse for the training of generative adversarial networks (GANs) because training GANs with SGD often fails.
Wasserstein GAN (WGAN) is modified based on GAN to improve the training stability by using the Wasserstein distance as the loss function \citep{wgan2017}. 
A thorough comparison between the AdaBelief optimizer and other optimization methods in training WGAN and its improved version with gradient penalty (WGAN-GP) has been carried out in \cite[Fig.6]{zhuang2020adabelief}. To ensure a fair comparison, we first replicate the same experiments using AdaBelief and, subsequently, evaluate the performances of the EAdam and AdamL optimizers based on the code published in the
GitHub repository\footnote{https://github.com/juntang-zhuang/Adabelief-Optimizer}. For a detailed description of the network structure, see \cite[Tab.~1 in Appendix F]{zhuang2020adabelief}. To observe the convergence behavior of each optimizer, the Fréchet Inception Distance (FID) \cite[(6)]{heusel2017gans} between the fake images and the real dataset is computed every 20 epochs by generating 64,000 fake images from random noise. Note that the FID score is a common metric to capture the similarity of generated images to real ones, and is considered to have consistency with increasing disturbances and human judgment \citep{heusel2017gans}. In both WGAN and WGAN-GP, the settings for Adabelief and EAdam are consistent with those in \cite[Fig.~6]{zhuang2020adabelief}, i.e., $\beta_1=0.5$, $\beta_2=0.999$, $\varepsilon=10^{-12}$, $\eta=2\cdot 10^{-4}$. We clip the weight of the discriminator to a range of $-0.01$ to $0.01$ in the WGAN experiment and we set the weight for the gradient penalty to $10$ in the case of WGAN-GP, to be consistent with the settings in \citep{zhuang2020adabelief}. 

As discussed in \cref{sec:hyperparametersinAdamL}, the generator aims to minimize the Wasserstein distance \cite[Eq.~(2)]{wgan2017}, while the discriminator/critic aims to maximize it. Additionally, the parameters of the discriminator are updated more frequently compared to that of the generator (cf.~\cite[Algorithm.~1]{wgan2017}), and this yields a more accurate approximation of the Wasserstein distance \cite[Sec.~3]{wgan2017}. Therefore, we only apply AdamL for updating the parameters of the discriminator. As for the generator, we set $\ell^{(k)}=1$ for all $k$, which is equivalent to the EAdam optimizer. 
In the context of training WGAN with AdamL, we present the results obtained based on two strategies for defining $\ell^{(k)}$. 
The first approach employs AdaBelief to estimate the output of the discriminator given fake generated data, and subsequently defining $\ell^{(k)}$ as $\tfrac{f(\bx_{f}^{(k)})-f_{\min}}{f_{\max}-f_{\min}}$. The second strategy uses \cref{ag:strategyI} for the automated adjustment of $\ell^{(k)}$. 

\textbf{Estimating $\ell^{(k)}$ using Adabelief.} In the training of WGAN, the discriminator tries to maximize its output given real training data, denoted as $f(\bx_{r}^{(k)})$, and minimizes its output given fake generated data, denoted as $f(\bx_{f}^{(k)})$, simultaneously. 
We use the output of the discriminator given fake generated data as our loss function values in the AdamL optimizer. Empirically, we find that when using the AdaBelief optimizer, the output of the discriminator given fake generated data typically falls within the range of $0.34\le f(\bx_{f}^{(k)})\le 0.6$. Therefore, by scaling the loss by using the scheme $\ell^{(k)}=\tfrac{f(\bx_{f}^{(k)})-0.32}{0.2}$, the updating stepsize will gradually decrease once the loss falls below $0.52$; this ensures $\ell^{(k)}\ge 0.1$. 
Numerically, we find that choosing $\gamma =1$ and reducing the updating stepsize by a factor of $100$ (by setting $\varphi=4$), is good enough for training WGANs in our experiments. 
Similarly, we set $\ell^{(k)}=\tfrac{f(\bx_{f}^{(k)})+100}{100}$, $\gamma=1$, and $\varphi=4$ in the experiment of WGAN-GP.

In the experiments presented in \cite[Fig.6]{zhuang2020adabelief}, the AdaBelief optimizer performs better than the other Adam's variants and  SGD, as it yields the lowest average FID score in both WAGN and WGAN-GP. Therefore, here we only compare with the results obtained by AdaBelief and EAdam (not considered in \cite{zhuang2020adabelief}).  We provide the averaged FID scores along with their maximum and minimum values from 5 independent simulations in \cref{fig:wganandwgangp} and summarize their detailed values and the $95\%$ confidence intervals in \cref{tab:fid_wgan}. We remark that in the training of WGAN, the AdamL optimizer yields the fastest convergence and lowest FID scores. In particular, the Adabelief optimizer yields an average FID score of approximately $82.23$ by the 100th training epoch. However, with the AdamL optimizer, an analogous value of $82.71$ can be attained in just 20 training epochs. After conducting 100 training epochs, the AdamL optimizer achieves an average FID score of $60.4$. Additionally, when comparing its performance to EAdam and AdaBelief, the AdamL optimizer exhibits the smallest deviation. Conversely, the EAdam optimizer performs the worst, yielding the highest FID score and the largest deviation.

In the training of WGAN-GP (cf.~\cref{fig:fidwgangp}), AdamL does not clearly outperform its competitors as observed in the case of WGAN. WGAN-GP introduces a penalty to the loss function, ensuring that the optimal result is associated with gradients having a norm close to 1. In particular, achieving the smallest loss function value doesn't necessarily coincide with a gradient norm approaching 1. This suggests that in WGAN-GP, using large step sizes may not be advantageous as the loss function values approach 0. However, AdamL is designed to employ a large step size as the loss function value approaches 0 and this explains the less good performance on this case study.

\begin{figure}[htp]
\centering
\subfloat[]{\label{fig:fidwgan}\includegraphics[width=0.48\textwidth]{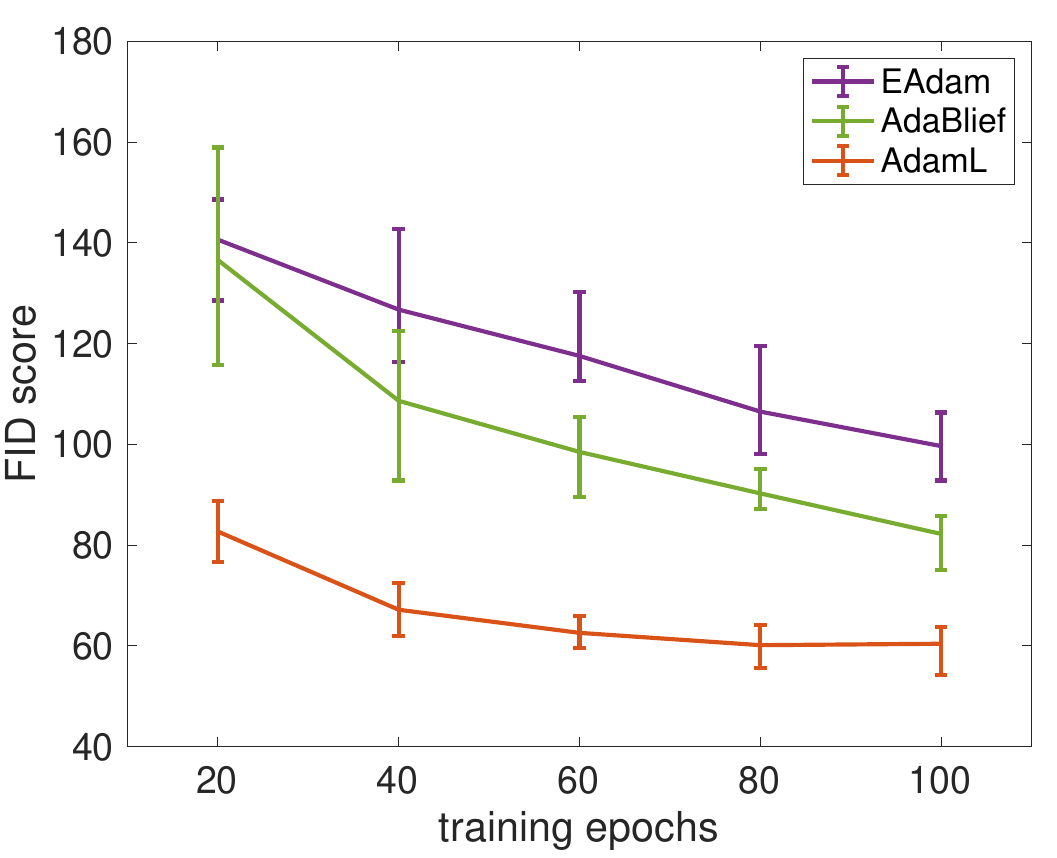}}\,\subfloat[]{\label{fig:fidwgangp}\includegraphics[width=0.48\textwidth]{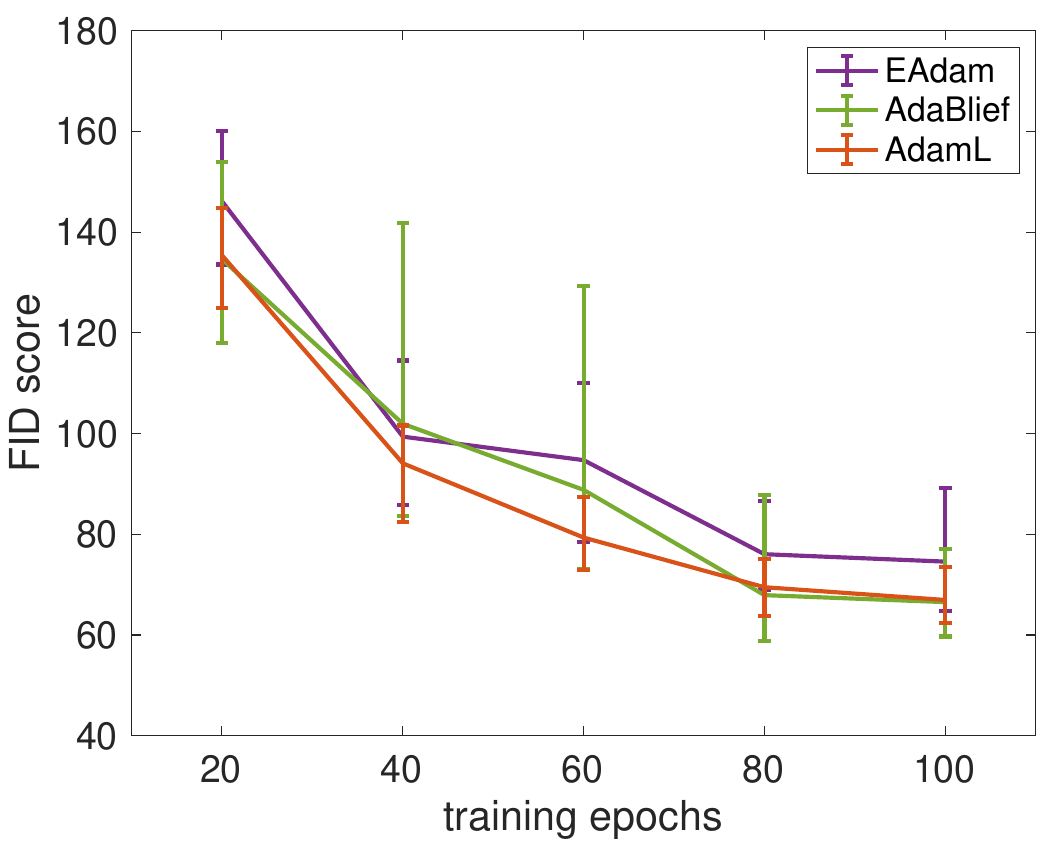}}
\caption[Mean of FID scores for WGAN and WGAN-GP]{Mean of FID scores over 5 independent runs with an error bar of the maximum and minimum FID values for WGAN (a) and WGAN-GP (b).}\label{fig:wganandwgangp}
\end{figure}
\begin{table}[htp]
\caption{\footnotesize Comparison of FID score (lower is better) of training WGAN using a vanilla CNN generator on \textsf{CIFAR10} between AdaBelief and AdamL; for each training epoch it shows the mean value and its $95$\% ($\pm\,2\sigma$) confidence interval over 5 independent runs.}\label{tab:fid_wgan}
\centering
\begin{adjustbox}{max width=\textwidth}
\begin{tabular}{lllllll}
\hline
&Optimizer & \phantom{MM}$20$th & \phantom{MM}$40$th & \phantom{MM}$60$th &\phantom{MM}$80$th &\phantom{MM}$100$th\\\hline \rule{0pt}{2.3ex}
\multirow{2}{3.5em}{\footnotesize WGAN} 
& EAdam  &  $140.6 \pm 20.1$ &$126.8 \pm 20.9$ & $117.6 \pm 14.4$  &$106.5 \pm 16.5$ &  $99.7 \pm 12.5$\\
\rule{0pt}{2.3ex} 
& AdaBelief   &   $136.6 \pm 32.5 $ & $108.7 \pm   26.2 $ &  $\phantom{1}98.5 \pm   13.4   $  & $\phantom{1}90.3 \pm\phantom{1}6.1    $  & $82.2\pm\phantom{1}8.3$\\
\rule{0pt}{2.3ex} 
& AdamL     &  $\phantom{1}82.7\pm 10.7$  &  $\phantom{1}67.2\pm \phantom{1}8.9 $   & $\phantom{1}62.6\pm\phantom{1} 6.0$ &$  \phantom{1}60.2 \pm \phantom{1} 7.0 $ &  $  60.4\pm\phantom{1}7.4$   \\ 
\rule{0pt}{2.3ex}  
\multirow{2}{3.8em}{{\footnotesize WGAN-GP}} 
& EAdam &$146.1\pm19.1$  & $\phantom{1}99.4 \pm20.2$ & $\phantom{1}94.7\pm24.7  $  & $\phantom{1}76.0\pm 12.7 $  & $74.6\pm  16.8$\\
\rule{0pt}{2.3ex} 
& AdaBelief                         &   $134.6 \pm 28.5 $ & $101.9 \pm   38.2 $ &  $\phantom{1}88.8 \pm    37.6   $  & $\phantom{1}67.9\pm  17.8    $  & $66.5\pm11.8$\\
\rule{0pt}{2.3ex}      
& AdamL     &  $135.4\pm 12.8$  &  $\phantom{1}94.1\pm 11.0 $   & $\phantom{1}79.3 \pm10.0$ &$ \phantom{1} 69.5 \pm \phantom{1}8.1 $ &  $  67.0 \pm\phantom{1} 6.5$    
\\ \hline
\end{tabular}
\end{adjustbox}
\end{table}

In \cref{fig:wganplot} of \cref{appendixA}, we compare the fake samples generated from the WGAN trained using the AdaBelief and AdamL optimizers after $40$ training epochs. Notably, the samples produced with AdamL exhibit significantly clearer images compared to those generated with AdaBelief, which are still slightly blurred.
For a more intuitive visualization, we conduct training of the WGAN on the Anime Faces dataset\footnote{https://github.com/learner-lu/anime-face-dataset}. We remark that the parameters of AdamL are chosen in accordance with those used for \textsf{CIFAR10}. In \cref{appendixA}, we illustrate the fake images generated by the WGAN, trained with different optimization algorithms, at the 20th, 80th, and 150th training epochs (cf.~\cref{fig:wgan_anime}).  It can be seen that at the 20th training epoch, the generated images obtained by Adam, EAdam, and AdaBelief demonstrate a deficiency in clarity and manifest distorted facial features. Conversely, the images generated with the AdamL optimizer exhibit better quality. With an increasing number of iterations, the quality of the generated images improves further. Remarkably, even at the 40th training epoch, the images produced by AdamL already show better quality than those generated by other optimizers at the 100th training epoch in terms of FID score.

\textbf{Automatically estimating $\ell^{(k)}$ using \cref{ag:strategyI}.} Let us replicate the experiment of training WGAN on the Anime Faces dataset using \cref{ag:strategyI} and keeping the same parameter configuration,  $\varphi=4$ and $\gamma=1$. \cref{fig:wgan_anime_compare} shows the mean FID scores over 5 independent runs, along with error bars indicating the maximum and minimum values. It can be observed that when using the AdamL optimizer, employing \cref{ag:strategyI} to automatically estimate $\ell^{(k)}$ yields FID scores similar to those obtained when $f$ is known (estimated using AdaBelief). Also, the FID scores obtained using Adam, EAdam, and AdaBelief show this behavior. Typically, the FID scores obtained with these optimizers are around $97$ with 150 training epochs, while AdamL achieves a similar FID score of approximately $99$ with only $60$ training epochs.
\begin{figure}[ht!]
\centering
\includegraphics[width=0.55\textwidth]{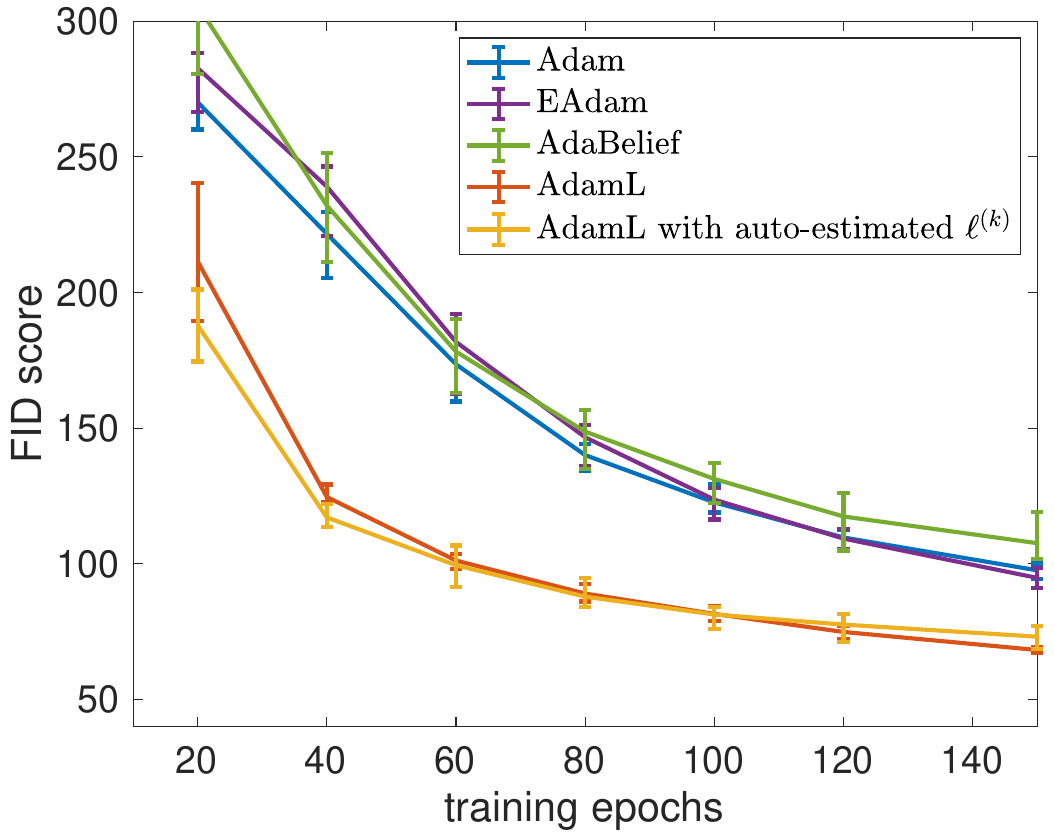}
\caption[AnimeFaces Dataset: Comparison of two strategies for estimating $\ell^{(k)}$ in AdamL for training WGAN]{AnimeFaces Dataset: Mean of FID scores over 5 independent runs with error bars of the maximum and minimum FID values for WGAN.}\label{fig:wgan_anime_compare}
\end{figure}
\begin{figure}[ht!]
\centering
\includegraphics[width=0.55\textwidth]{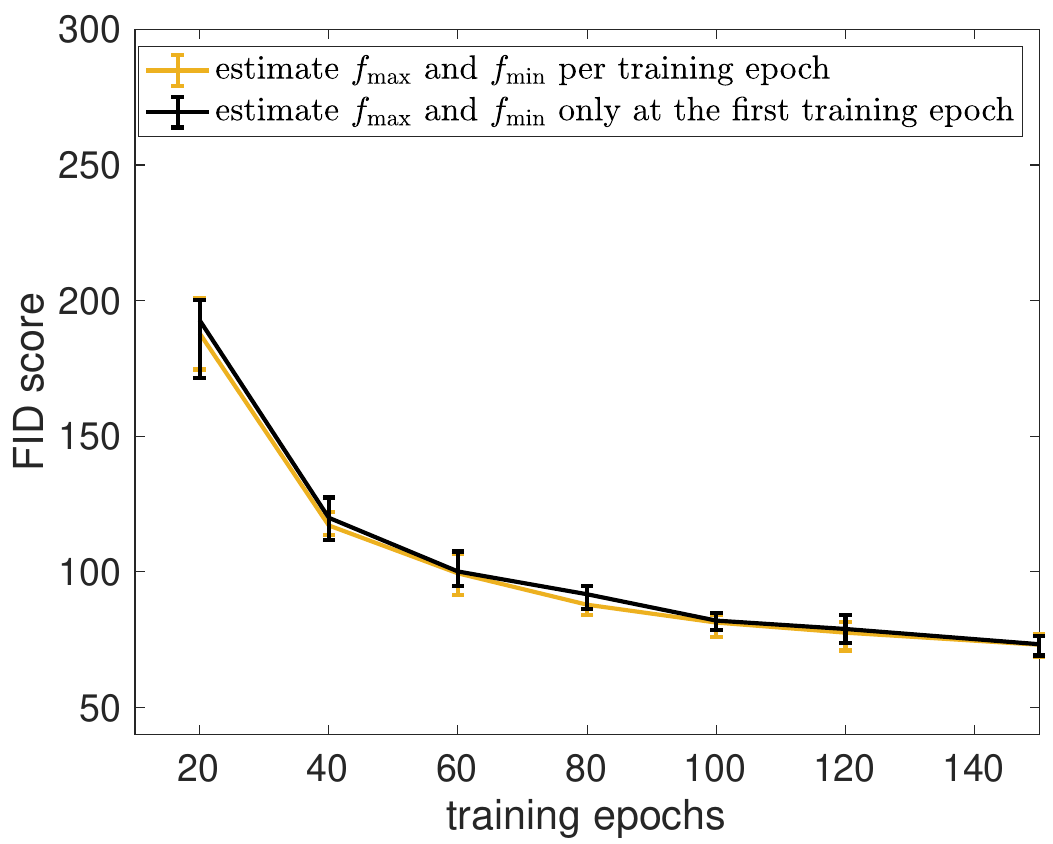}
\caption[AnimeFaces Dataset: Comparison of two strategies for estimating $\ell^{(k)}$ in AdamL for training WGAN]{AnimeFaces Dataset: Mean of FID scores over 5 independent runs with an error bar of the maximum and minimum FID values for WGAN.}\label{fig:wgan_anime_comparef}
\end{figure}

In \cref{sec:hyperparametersinAdamL}, we remark that when training a GAN, it is common to observe significant fluctuations in the discriminator's output during the first training epoch. Consequently, \cref{ag:strategyI} computes $\ell^{(k)}$ only based on the estimation of $f_{\min}$ and $f_{\max}$ in the first training epoch.
Here, we demonstrate that similar results can be achieved as in \cref{ag:strategyI} when training a WGAN by computing $\ell^{(k)}$ based on the estimation of $f_{\min}$ and $f_{\max}$ during each training epoch. We replicate the experiments presented in \cref{fig:wgan_anime_compare} using the AdamL optimizer,  with updates to $f_{\min}$ and $f_{\max}$ being made at the end of each training epoch. The results shown in \cref{fig:wgan_anime_comparef} are very close to those obtained by applying \cref{ag:strategyI}.

\subsection{LSTM for language modeling}
In this last numerical study, we consider the language modeling task that consists of training LSTM models with 1, 2, and 3 layers on the \textsf{Penn Treebank} dataset. These experiments are conducted based on the code available on the
GitHub repository\footnote{https://github.com/juntang-zhuang/Adabelief-Optimizer}. In line with the evaluations of EAdam and AdaBelief \citep{yuan2020eadam,zhuang2020adabelief}, the perplexity is employed to measure the performance of all algorithms under comparison. Note that perplexity is a commonly used metric for assessing the performance of a language model, with lower values indicating better performance \citep{CHEN1999359}.

For all the experiments of LSTM, the settings for AdaBelief, EAdam, and AdamL are aligned with those in \citep{zhuang2020adabelief}, i.e., $\beta_1=0.9$, $\beta_2=0.999$, $\varepsilon=10^{-12}$ and $\eta=0.01$. Note that in \citep{zhuang2020adabelief}, to achieve the best result of AdaBelief, a learning rate of $0.001$ was chosen for 1-layer LSTM, and a learning rate of $0.01$ was used for 2 and 3-layer LSTM. In particular with $\eta= 0.001$, the lowest test perplexity of 1-layer LSTM obtained by AdaBelief is approximately $84.2$ \cite[Fig.~5]{zhuang2020adabelief}, which is still higher than the result of AdamL with $\eta= 0.01$, i.e., $81.7$ in \cref{tab:lstm}. 

\textbf{Estimating $\ell^{(k)}$ using Adam.} Observing the results obtained by Adam, we found that the objective function values of $1$, $2$, and $3$-layer LSTM change very mildly; more precisely, in the first 200 training epochs, it ranges between $9$ to $3$. To ensure that the objective function has some part larger than 1 and some part smaller than 1, we normalize it by dividing it by 6, i.e., $\ell^{(k)}=\ell(f(\bx^{(k)}))=\frac{1}{6}\,f(\bx^{(k)}$. Additionally, we set $\varphi=4$ to improve the sensitivity of $(\ell^{(k)})^{\varphi}$ to the changes in $\ell^{(k)}$, causing the value of $(\ell^{(k)})^{\varphi}$ to approach $0$ as $\ell^{(k)}$ decreases. We empirically find that setting $\gamma=0.5$ yields the best performance for AdamL in the case of a 1-layer LSTM, while for 2 and 3-layer LSTMs, the optimal value appears to be $\gamma=1$.

We manually decrease the learning rate by a factor of $0.1$ at different training epochs for all the optimizers. The corresponding average perplexities over 5 independent simulations are presented in \cref{tab:lstm}. 
Note that in the experiment of LSTM, we lose AdamL's benefit of not manually adjusting the learning rate to improve performance. In \cref{tab:lstm}, the optimizers that produce the highest perplexities are highlighted in red and the ones that yield the lowest ones are marked in bold. As can be observed, the EAdam optimizer slightly outperforms the AdaBelief optimizer in the training of LSTM models, while the AdamL optimizer always results in the lowest perplexities for all the experiments. 

\begin{table}[ht!]
{\footnotesize
\caption{Mean value of test perplexity and its maximum deviation from its mean value over 5 independent simulations (lower is better) of 1, 2, and 3-layer LSTM on \textsf{Penn Treebank}.}\label{tab:lstm}
\begin{center}
\begin{tabular}{cccccc}
\hline
&learning rate decay epoch & Adam& EAdam & AdaBelief & AdamL \\\hline \rule{0pt}{2.3ex}
\multirow{2}{3.5em}{\footnotesize 1-layer} &\phantom{1}$75$th and $150$th & $89.4$ &  $90.0$ &\textcolor{red}{$91.3$} & $\mathbf{82.1}$   \\
\rule{0pt}{2.3ex} 
&$100$th and $145$th & $91.3$ &  $91.6$ &\textcolor{red}{$93.3$} & $\mathbf{81.7}$
\\ \rule{0pt}{2.3ex}
\multirow{2}{3.5em}{\footnotesize 2-layer} &$\phantom{1}75$th and $150$th & \textcolor{red}{$67.1$} &  $66.5$ &$66.9$ & $\mathbf{66.0}$   \\
\rule{0pt}{2.3ex} 
&$100$th and $145$th & \textcolor{red}{$67.0$} &  $66.4$ &$66.7$ & $\mathbf{65.4}$
\\ \rule{0pt}{2.3ex}
\multirow{2}{3.5em}{\footnotesize 3-layer} &$\phantom{1}75$th and $150$th & \textcolor{red}{$64.6$} &  $61.9$ &$62.1$ & $\mathbf{61.1}$   \\
\rule{0pt}{2.3ex} 
&$100$th and $145$th & \textcolor{red}{$64.4$} &  $61.0$ &$61.2$ & $\mathbf{60.5}$
\\ \hline
\end{tabular}
\end{center}
}
\end{table}

\begin{figure}[ht!]
\centering
\subfloat[]{\label{fig:lstm100_train75}\includegraphics[width=0.45\textwidth]{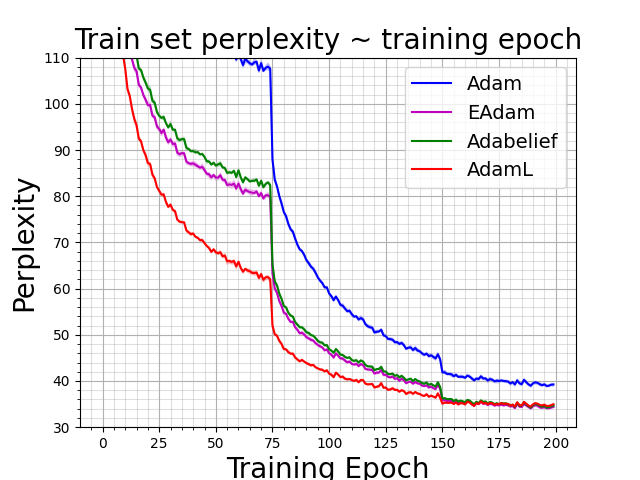}}\,
\subfloat[]{\label{fig:lstm100_train}\includegraphics[width=0.45\textwidth]{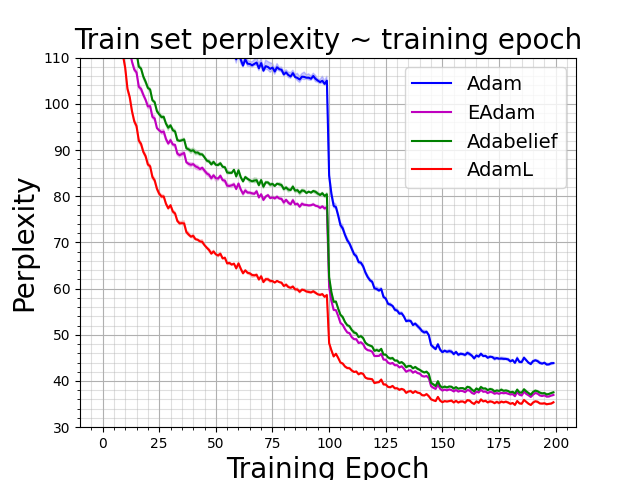}}\,
\subfloat[]{\label{fig:lstm100_test75}\includegraphics[width=0.45\textwidth]{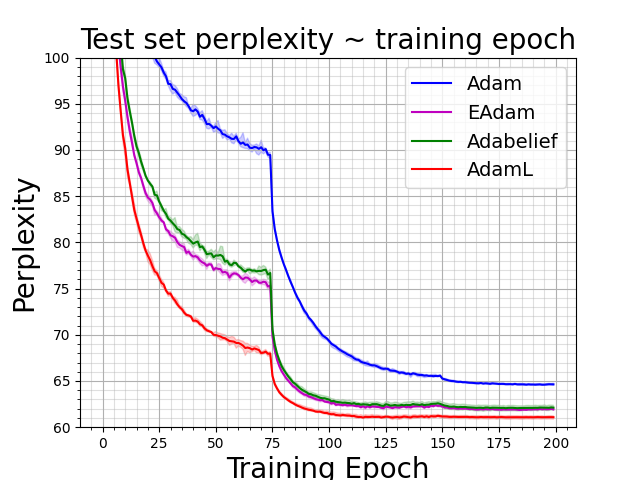}}
\subfloat[]{\label{fig:lstm100_test}\includegraphics[width=0.45\textwidth]{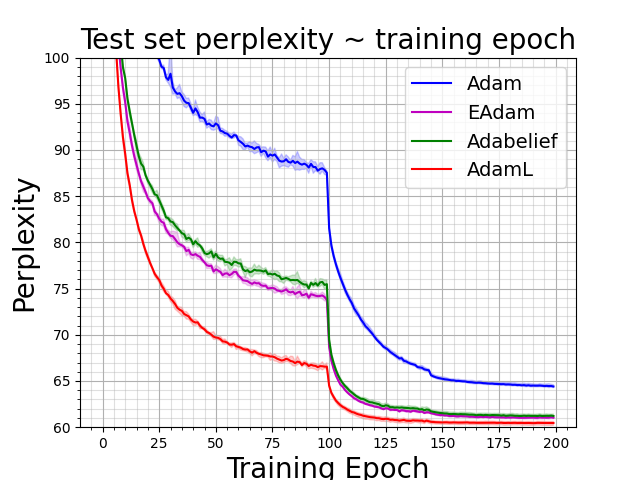}}
\caption[Mean values of train and testing perplexities of 3-layer LSTM on \textsf{Penn Treebank}]{Mean values of training (top row) and testing (bottom row) perplexities and their maximum deviation from the mean value over 5 independent simulations (lower is better) of 3-layer LSTM on \textsf{Penn Treebank}.}\label{fig:lstm100}
\end{figure}

\cref{fig:lstm100} shows the mean values of train and test perplexities, along with their maximum deviations, from 5 independent simulations of the 3-layer LSTM. Notably, the improvement in convergence achieved by the second learning rate decrease in the AdamL optimizer appears to be minimal. After the first learning rate reduction, the testing perplexity obtained by AdamL is already lower than that achieved by other optimizers even when they undergo two learning rate decreases. Overall, the AdamL optimizer achieves both the fastest convergence and good accuracy compared to the other adaptive methods.

\textbf{Automatically estimating $\ell^{(k)}$ using \cref{ag:strategyI}.} To study the effectiveness of \cref{ag:strategyI}, we repeat the same experiments for training a 2-layer and 3-layer LSTMs on the \textsf{Penn Treebank} dataset using the estimation strategy in \cref{ag:strategyI}. \cref{fig:lstm_auto} shows the mean values of training and testing perplexities over 5 independent runs and their corresponding maximum deviations. We see that before triggering the reduction of the initial learning rate, the convergence is slower in comparison to the strategy that relies on the knowledge of $\ell^{(k)}$ obtained with an initial training phase with the Adam optimizer. On the other hand, after the second reduction in the learning rate, the training and testing perplexities obtained with these two strategies become quite similar.

\begin{figure}[ht!]
\centering
\subfloat[]{\label{fig:lstm_layer2}\includegraphics[width=0.48\textwidth]{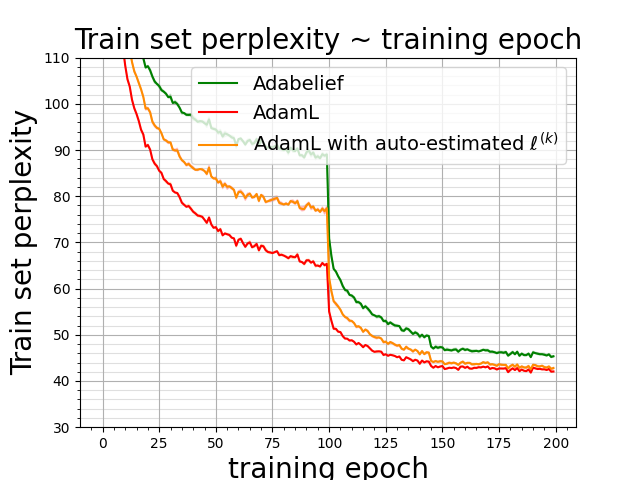}}\,
\subfloat[]{\label{fig:lstm_layer3}\includegraphics[width=0.48\textwidth]{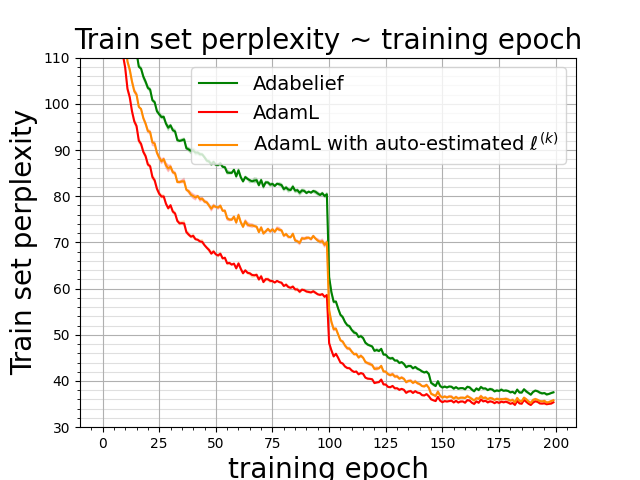}}\,
\subfloat[]{\label{fig:lstm_layer2_test}\includegraphics[width=0.48\textwidth]{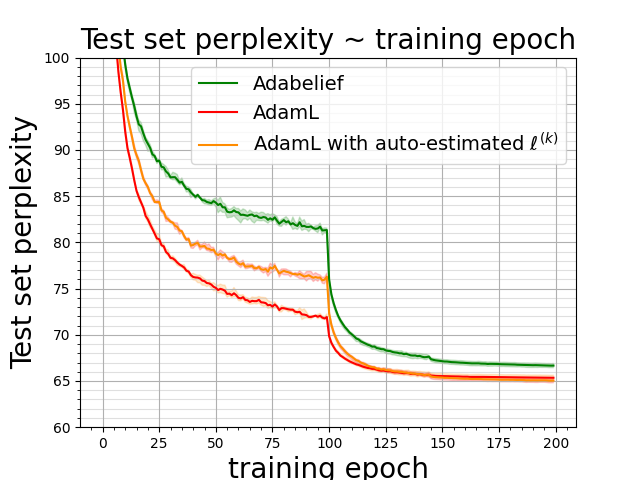}}
\subfloat[]{\label{fig:lstm_layer3_test}\includegraphics[width=0.48\textwidth]{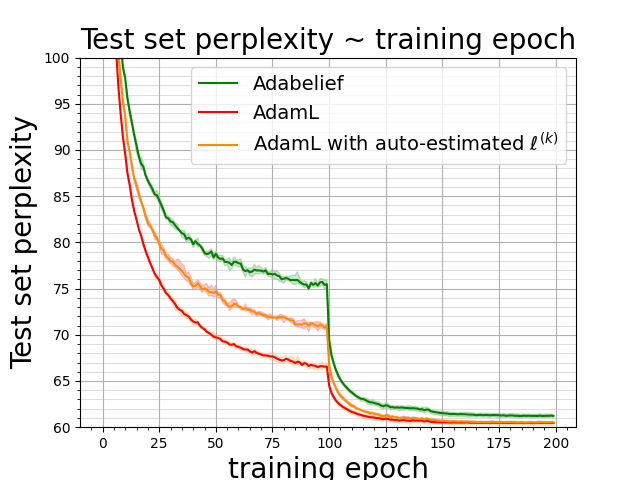}}
\caption[Mean values of train and testing perplexities of 3-layer LSTM on \textsf{Penn Treebank}]{Mean values of training (top row) and testing (bottom row) perplexities and their maximum deviation from the mean value over 5 independent simulations (lower is better) of 2-layer ((a) and (c)) and 3-layer ((b) and (d)) LSTM on \textsf{Penn Treebank}.}\label{fig:lstm_auto}
\end{figure}

\section{Conclusion}\label{sec:conclusions}
We have proposed a new variant of Adam that takes into account the loss function, namely \emph{AdamL}. We have provided sufficient conditions for the monotonicity and linear convergence of AdamL and other Adam's variants. 
Extensive numerical tests on both benchmark examples and neural network case studies have been presented to compare AdamL with some leading competitors.

Overall, we found that AdamL performs better than the other Adam's variants in various training tasks. For instance, it exhibits either faster or more stable convergence when applied to the training of convolutional neural networks (CNNs) for image classification and for training generative adversarial networks (WGANs) using vanilla CNNs. 
In the context of training CNNs for image classification, AdamL stands out from other variants of Adam and eliminates the need for manually adjusting the learning rate in the later stages of the training.
When training WGANs, we have introduced two primary approaches for implementing AdamL.
The first approach involves pre-training the WGANs using Adam or another optimizer, and then estimating $\ell^{(k)}$ based on the maximum and minimum cost function values observed during the pre-training process. Alternatively, \cref{ag:strategyI} presents a method for automatically estimating $\ell^{(k)}$, that yields comparable results with the first approach. Notably, for both approaches, the convergence speed of AdamL is more than twice as fast as that achieved by leading competitors such as Adam, AdaBelief, and EAdam.
Finally, AdamL also demonstrates improved performance in training LSTMs for language modeling tasks. 
Despite the requirement for additional hyperparameters when compared to Adam, EAdam, and AdaBelief, we have shown that the choice of these hyperparameters is relatively consistent. For instance, values like $\gamma=1$ and $\varphi=4$ are employed in the training of both WGANs and LSTMs, and the scaling function can be automatically estimated using \cref{ag:strategyI}.
The potential of AdamL in other training scenarios remains a subject for future investigation.
\section*{Acknowledgements}
We thank Olga Mula and Michiel E.~Hochstenbach for their valuable discussions that significantly enhanced the presentation of this paper. This research was funded by the EU ECSEL Joint Undertaking under grant agreement no.~826452 (project Arrowhead Tools).
\footnotesize\bibliography{ref}
\newpage
\appendix
\section{Numerical experiments on WGANs}\label{appendixA}
\begin{figure}[ht!]
\centering
\subfloat[]{\label{fig:wgana}\includegraphics[width=0.46\textwidth]{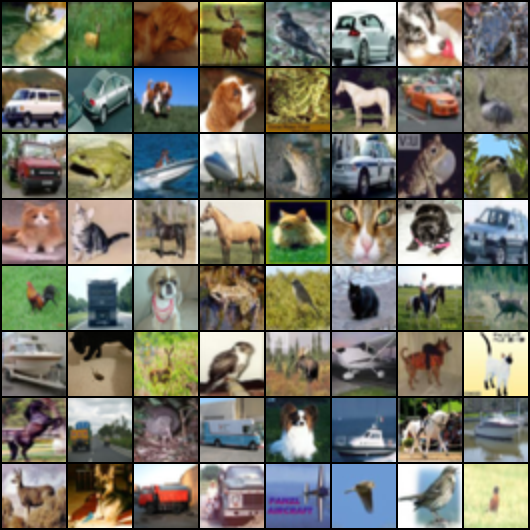}}\quad
\subfloat[]{\label{fig:wgand}\includegraphics[width=0.46\textwidth]{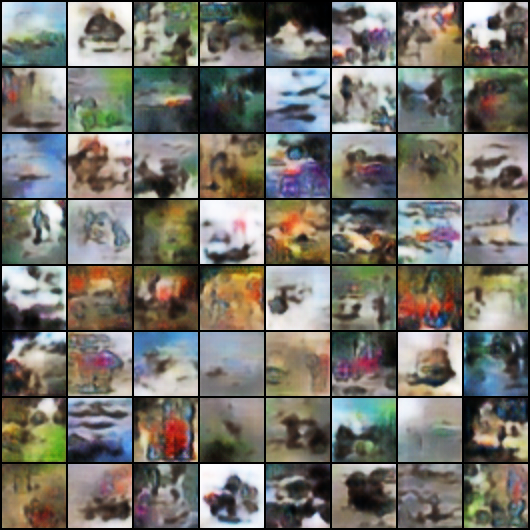}}\\
\subfloat[]{\label{fig:wganc}\includegraphics[width=0.46\textwidth]{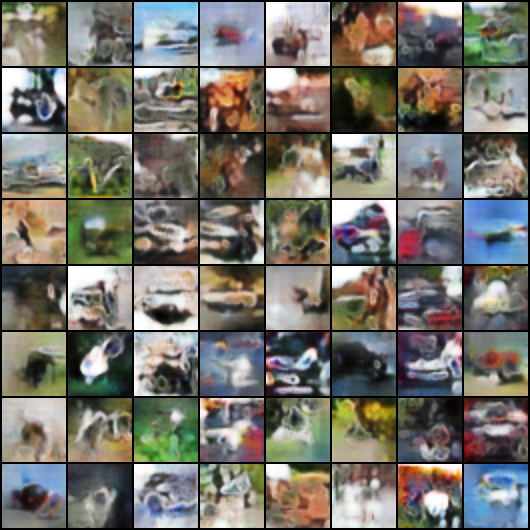}}\quad
\subfloat[]{\label{fig:wganb}\includegraphics[width=0.46\textwidth]{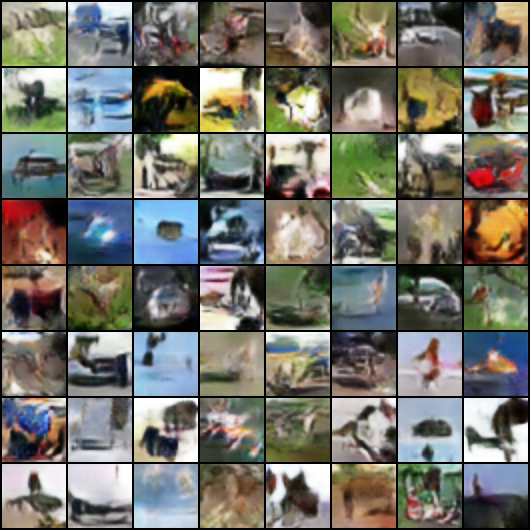}}
\caption[CIFAR10 Dataset: Generated images of WGAN trained using different optimizers.]{Real images (a) and samples from WGAN trained using EAdam (b), AdaBelief (c) and AdamL (d) at 40th training epoch.}\label{fig:wganplot}
\end{figure}
\begin{figure}[htp]
\centering
\subfloat[Real Images]{\label{fig:wgan_anime_real}\includegraphics[width=0.9\textwidth]{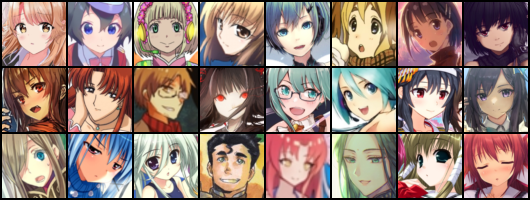}}\\
\subfloat[Adam at $20$th]{\label{fig:wgan_anime_adam_20}\includegraphics[width=0.24\textwidth]{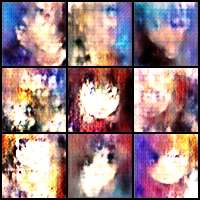}}\
\subfloat[EAdam at $20$th]{\label{fig:wgan_anime_eadam_20}\includegraphics[width=0.24\textwidth]{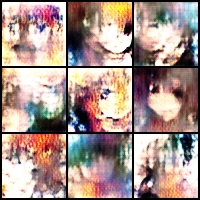}}\
\subfloat[AdaBelief at $20$th]{\label{fig:wgan_anime_adabelief_20}\includegraphics[width=0.24\textwidth]{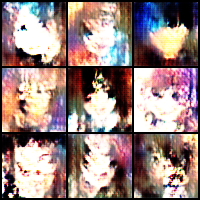}}\
\subfloat[AdamL at 20th]{\label{fig:wgan_anime_adaml_20}\includegraphics[width=0.24\textwidth]{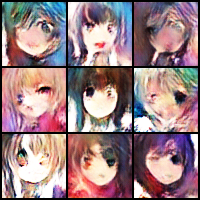}}\\
\subfloat[Adam at 80th]{\label{fig:wgan_anime_adam_80}\includegraphics[width=0.24\textwidth]{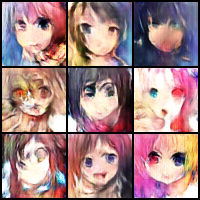}}\
\subfloat[EAdam at 80th]{\label{fig:wgan_anime_eadam_80}\includegraphics[width=0.24\textwidth]{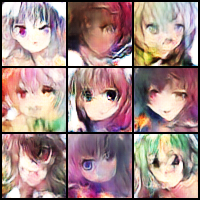}}\
\subfloat[AdaBelief at 80th]{\label{fig:wgan_anime_adabelief_80}\includegraphics[width=0.24\textwidth]{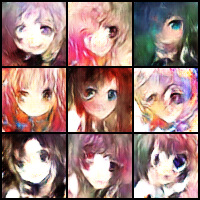}}\
\subfloat[AdamL at 80th]{\label{fig:wgan_anime_adaml_80}\includegraphics[width=0.24\textwidth]{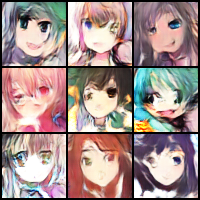}}\\
\subfloat[Adam at 150th]{\label{fig:wgan_anime_adam_150}\includegraphics[width=0.24\textwidth]{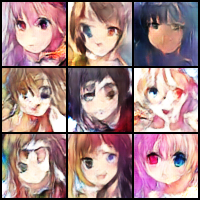}}\
\subfloat[EAdam at 150th]{\label{fig:wgan_anime_eadam_150}\includegraphics[width=0.24\textwidth]{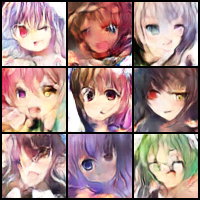}}\
\subfloat[AdaBelief at 150th]{\label{fig:wgan_anime_adabelief_150}\includegraphics[width=0.24\textwidth]{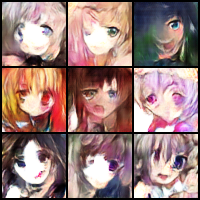}}\
\subfloat[AdamL at 150th]{\label{fig:wgan_anime_adaml_150}\includegraphics[width=0.24\textwidth]{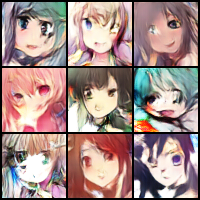}}
\caption[AnimeFaces Dataset: Generated images of WGAN trained using different optimizers.]{AnimeFaces Dataset: Real images (a) and Generated samples from WGAN trained using EAdam, AdaBelief, and AdamL at different training epochs.}\label{fig:wgan_anime}
\end{figure}
\end{document}